%% file: paper.tex
\documentclass[english]{article}

\usepackage{geometry}
\usepackage[T1]{fontenc}
\usepackage[latin9]{inputenc}
\usepackage{bm}
\usepackage{amsmath,mathtools}
\usepackage{amssymb}
\usepackage[unicode=true,
 bookmarks=false,
 breaklinks=false,pdfborder={0 0 1},colorlinks=false]
 {hyperref}
\hypersetup{colorlinks,citecolor=blue,filecolor=blue,linkcolor=blue,urlcolor=blue}

\usepackage{tikz}
  \usetikzlibrary{positioning,fit,backgrounds, matrix,calc}
\usepackage{xcolor}

\makeatletter
 
\usepackage{amsthm}
\usepackage{cite}  
\usepackage{comment}
\usepackage{authblk}
\usepackage[sort,compress]{natbib}
\usepackage{booktabs}
\usepackage{mathabx}
\usepackage{graphicx}
\usepackage{array}

\usepackage{bbding}
\usepackage{colortbl}
\usepackage{makecell}

\usepackage[linesnumbered,ruled,vlined]{algorithm2e}

\SetCommentSty{mycommfont}

\usepackage{float}
\usepackage{multirow}
\usepackage{dsfont}
\usepackage{tcolorbox}
\usepackage{color}
\definecolor{yxc}{RGB}{255,0,0}
\definecolor{ytw}{RGB}{255,69,0}
\definecolor{gen}{RGB}{0,0,200}
\definecolor{zhh}{RGB}{200,200,0}

\allowdisplaybreaks

\input{notations}

\newtheorem{definition}{Definition}

\theoremstyle{plain} \newtheorem{lemma}{\textbf{Lemma}}\newtheorem{proposition}{\textbf{Proposition}}\newtheorem{theorem}{\textbf{Theorem}}

\theoremstyle{assumption}\newtheorem{assumption}{\textbf{Assumption}}
\theoremstyle{remark}\newtheorem{remark}{\textbf{Remark}}
\theoremstyle{Corollary}\newtheorem{corollary}{\textbf{Corollary}}

\usepackage{enumitem}

\usepackage[nosort,capitalise]{cleveref}
\crefname{enumi}{}{}
\crefname{equation}{Eqn.}{Eqns.}

\crefname{theorem}{Theorem}{Theorems}
\crefname{lemma}{Lemma}{Lemmas}
\crefname{corollary}{Corollary}{Corollaries}
\crefname{proposition}{Proposition}{Propositions}
\crefname{definition}{Definition}{Definitions}
\crefname{remark}{Remark}{Remarks}
\crefname{assumption}{Assumption}{Assumptions}
\crefname{conjecture}{Conjecture}{Conjectures}

\Crefname{theorem}{Theorem}{Theorems}
\Crefname{lemma}{Lemma}{Lemmas}
\Crefname{corollary}{Corollary}{Corollaries}
\Crefname{proposition}{Proposition}{Propositions}
\Crefname{definition}{Definition}{Definitions}
\Crefname{remark}{Remark}{Remarks}
\Crefname{assumption}{Assumption}{Assumptions}
\Crefname{conjecture}{Conjecture}{Conjectures}

\allowdisplaybreaks

\title{Efficient Sampling with Discrete Diffusion Models: \\ Sharp and Adaptive Guarantees}

\author{Daniil Dmitriev\thanks{Equal contribution, alphabetical order.} \thanks{Department of Statistics and Data Science, the Wharton School, University of Pennsylvania; email: \texttt{\{daniild,zhihanh,ytwei\}@wharton.upenn.edu}
},  ~~~Zhihan Huang$^{*\dagger}$, ~~~Yuting Wei$^\dagger$
}

\date{\today}

\makeatother

\begin{document}

\maketitle 

\vspace{-0.5cm}
\begin{abstract}
Diffusion models over discrete spaces have recently shown striking empirical success, yet their theoretical foundations remain incomplete. In this paper, we study the sampling efficiency of score-based discrete diffusion models under a continuous-time Markov chain (CTMC) formulation, with a focus on $\tau$-leaping-based samplers. We establish sharp 
  convergence guarantees for attaining $\varepsilon$ accuracy in Kullback-Leibler (KL) divergence for both uniform and masking noising processes. 
For uniform discrete diffusion, we show that the $\tau$-leaping algorithm achieves an iteration complexity of order $\wt O(d/\varepsilon)$, with $d$ the ambient dimension of the target distribution, eliminating linear dependence on the vocabulary size $S$ and improving existing bounds by a factor of $d$; moreover, we establish a matching algorithmic lower bound showing that linear dependence on the ambient dimension is unavoidable in general. 
For masking discrete diffusion, we introduce a modified $\tau$-leaping sampler whose convergence rate is governed by an intrinsic information-theoretic quantity, termed the \emph{effective total correlation}, which is bounded by $d \log S$ but can be sublinear or even constant for structured data. As a consequence, the sampler provably adapts to low-dimensional structure without prior knowledge or algorithmic modification, yielding sublinear convergence rates for various practical examples (such as hidden Markov models, image data, and random graphs). Our analysis requires no boundedness or smoothness assumptions on the score estimator beyond control of the score entropy loss.
\end{abstract}

\setcounter{tocdepth}{2}
\tableofcontents

\input{introduction}

\input{setting}

\input{results}

\input{discussion}

\section*{Acknowledgements}
This work is supported in part by the NSF grants CCF-2106778, CCF-2418156 and 
CAREER award DMS-2143215.

\bibliographystyle{apalike}
\bibliography{reference-diffusion}

\appendix

\input{appendix.tex}

\end{document}

%% file: notations.tex
\makeatletter

\def\LatinUpper{A,B,C,D,E,F,G,H,I,J,K,L,M,N,O,P,Q,R,S,T,U,V,W,X,Y,Z}

\newcommand{\genCal}[1]{\expandafter\newcommand\csname c#1\endcsname{{\mathcal #1}}}
\@for\q:=\LatinUpper\do{%
	\expandafter\genCal\q
}

\newcommand{\genBb}[1]{\expandafter\newcommand\csname b#1\endcsname{{\mathbb #1}}}
\@for\q:=\LatinUpper\do{%
	\expandafter\genBb\q
}

\makeatother

\DeclareMathOperator{\ham}{d_H}
\DeclarePairedDelimiter{\abs}{\lvert}{\rvert}%

\newcommand{\back}[1]{\overset{\leftarrow}{#1}}

\newcommand{\mi}{\mathrm{I}}
\newcommand{\ent}{\mathcal{H}}
\newcommand{\lse}{\cL_{\mathrm{SE}}}
\newcommand{\breg}{D}
\newcommand{\wt}{\widetilde}
\newcommand{\wh}{\widehat}
\newcommand{\defn}{\coloneqq}

\newcommand{\mask}{\mathrm{MASK}}

\renewcommand{\P}{\mathbb{P}}

\newcommand{\tr}{\mathrm{Tr}}

\renewcommand{\d}{\mathrm{d}}

\newcommand{\KL}{\mathsf{KL}}

\newcommand{\escore}{\varepsilon_{\mathrm{score}}}
\newcommand{\E}{\mathbb{E}}

\newcommand{\indep}{\perp \!\!\! \perp}

\newcommand{\Unif}{\text{Unif}}

%% file: introduction.tex
\section{Introduction}

Diffusion models have recently emerged as state-of-the-art approaches for high-fidelity image generation and video synthesis (\cite{ho2020denoising,dhariwal2021diffusion,song2019generative,ho2022video}), and have already led to significant scientific advances in various domains, including climate modeling, protein structure prediction, and materials science (\cite{watson2023novo,zeni2025generative,li2024generative}). 
At their core, diffusion models are built upon two stochastic processes: a forward process that gradually corrupts the data distribution into pure noise, and a reverse process that generates samples by learning the logarithmic gradient of the perturbed marginals, commonly referred to as the score function.

Despite their broad empirical success, diffusion models have been predominantly developed for continuous data. 
Their extension to discrete domains, such as natural language, graph-structured data, and categorical labels, has long remained challenging,~although already discussed in~\cite{sohl2015deep}.
This perspective began to shift following the seminal work of \cite{austin2021structured}, which revealed the promise of diffusion-based approaches in discrete settings. 
Analogous to the continuous case, discrete diffusion models rely on a pair of noisy forward and reverse processes, with sampling achieved by learning appropriate ratios of distributions. Among recent developments (\cite{campbell2022continuous,sahoo2024simple,shi2024simplified,ou2024your,bach2025sampling}), score-entropy discrete diffusion (SEDD) has demonstrated striking performance in text generation (\cite{lou2023discrete}), challenging the long-standing dominance of autoregressive language models. In contrast to autoregressive approaches, diffusion-based language models are not constrained to a fixed generation order (such as left-to-right), and they naturally lend themselves to more flexible forms of controlled generation, including conditional and structured text synthesis.

The promise of discrete diffusion models has spurred growing interest in their theoretical foundations. A particularly influential line of work formulates discrete diffusion through the lens of continuous-time Markov chains (CTMCs) \citep{campbell2022continuous}, in which the forward dynamics is governed by a carefully designed rate matrix, and the backward dynamics is approximated via a learned score function. 
Among the proposed constructions, two choices have emerged as especially prominent: the uniform rate matrix, which induces a uniform stationary distribution for the forward process, and the absorbing rate matrix, which yields a degenerate stationary distribution with an absorbing state. In practice, the performance of the resulting samplers depends sensitively on the choice of the rate matrix~(\cite{lou2023discrete,von2025scaling}).
Correspondingly, two parallel lines of work have sought to understand the sampling efficiency of discrete diffusion models --- specifically, the number of steps required to produce sufficiently accurate samples --- under these respective constructions. Representative results include \cite{chen2024convergence,ren2024discrete,zhang2024convergence,pham2025discrete,liang2025discrete} for uniform diffusion and \cite{park2024jump,liang2025absorb,conforti2025non,libreaking,chen2025optimal} for masking diffusion (also referred to as absorbing diffusion).

Existing theoretical analyses for score-based discrete diffusions suggest that convergence rates typically scale at least linearly with both the vocabulary size $S$ and the ambient dimension $d$. Such scaling can quickly become prohibitive in applications; for instance, in GPT-2-based tasks, the vocabulary size is $S = 50{,}257$ and the dimension is $d = 10^2 \sim 10^3$ \citep{lou2023discrete}. These considerations naturally motivate a fundamental question: 
\vspace{-0.2cm}
\begin{center}
\emph{
How efficient are discrete diffusion models? When is sublinear convergence possible?}
\end{center}

\subsection{Sampling efficiency and adaptivity}

To put our discussion in context, there has been substantial progress in understanding the sample efficiency of continuous diffusion models. Seminal work by \citet{chen2022sampling} characterizes the iteration complexity of the DDPM sampler under Lipschitz (or smoothness) assumptions on the score functions across all steps. Subsequent studies significantly relax these conditions and establish convergence guarantees for broader classes of continuous distributions \citep{benton2024denoising,li2023towards,chen2023improved}. Nevertheless, it is now well understood that for general distributions, a linear dependence on the ambient dimension $d$ is unavoidable. By contrast, when the target distribution exhibits additional structure --- such as Gaussian mixture models or support on low-dimensional manifolds --- a growing body of work shows that popular samplers can adaptively exploit intrinsic low-dimensional geometry, achieving improved efficiency without explicit dimension reduction (see, e.g., \cite{li2024adapting,li2025dimension,huang2024denoising,liang2025low}).

The landscape shifts considerably as we move to discrete diffusion models. Under the CTMC formulation, algorithms such as Gillespie's method and uniformization allow for exact simulation of the reverse process, free of discretization error \citep{gillespie1976general,van1992uniformization,chen2024convergence}. However, these methods suffer from high computational costs in high-dimensional settings.
Moreover, their convergence guarantees are inherently stochastic, as they depend on a random number of transitions. An alternative and widely adopted approach,  particularly in diffusion-based language models, is provided by $\tau$-leaping and its variants, including truncated $\tau$-leaping \citep{gillespie2001approximate,campbell2022continuous}. Originally developed in chemical kinetics, $\tau$-leaping replaces sequential state transitions with parallel updates across coordinates, offering substantial computational gains in large systems.
Yet, our theoretical understanding of $\tau$-leaping remains incomplete. Current state-of-the-art results exhibit at least linear dependence on the vocabulary size $S$, linear dependence on $d$ for the absorbing case, and quadratic dependence on $d$ for the uniform case (\cite{liang2025absorb,liang2025discrete,conforti2025non}); see Table~\ref{tab:comparison} for more details. It remains an open question whether these dependencies are fundamental information-theoretic barriers or merely analytical artifacts. Furthermore, as in the continuous setting, an ideal sampling algorithm should automatically adjust to the intrinsic difficulty of the target distribution. For example, one would expect substantially faster convergence for Dirac delta measures or uniform target distributions, without prior knowledge of the structure or modifications to the algorithm. Existing analyses of 
$\tau$-leaping do not illuminate whether such adaptivity is possible.
More specifically, we aim to address the following question:

\begin{center}
\emph{
Can score-based samplers automatically adapt to structured target distributions?
}
\end{center}

\subsection{Our contributions}

The contributions of this work are centered on establishing sharp convergence guarantees for discrete diffusion models, bridging the gap between empirical success and theoretical understanding. Specifically, our contributions are mainly threefold:
\begin{itemize}
    \item 

 \textbf{Optimal rates for uniform diffusion}: We establish that for the uniform diffusion process, the $\tau$-leaping sampler requires only $\widetilde{O}(d/\varepsilon)$ discretization steps to achieve an $\varepsilon$-error in KL divergence. This result significantly sharpens the previously best-known bound of $\widetilde{O}(d^2 S/\varepsilon)$ \citep{liang2025discrete}, effectively removing a factor of $d$ and the dependence on the vocabulary size $S$.

\item 
\textbf{Fundamental lower bounds}: We demonstrate that the linear dependence on the dimension $d$ is essentially unimprovable for the $\tau$-leaping algorithm. Specifically, we show that under uniform diffusion, an $o(d)$ complexity bound is unattainable unless the target distribution is already proximal to the uniform measure. This result characterizes a fundamental price of sampling for informative distributions.

\item 
\textbf{Adaptivity for masking diffusion}: For the masking diffusion process, we introduce a refined $\tau$-leaping sampler that has a complexity governed by $\widetilde{O}(\mathcal{D}/\varepsilon)$, where $\mathcal{D}$ is the effective total correlation, an information-theoretic measure of the target distribution's intrinsic complexity. Notably, while $\mathcal{D}$ is always bounded by the classical total correlation and the dual total correlation (and thus by $d \log S$), it can be sublinear or even $O(1)$ for highly structured data, allowing our sampler to adapt automatically to low-dimensional target distributions.
\end{itemize}

In contrast to prior work, our upper bounds 
do not require the boundedness of the score estimator or any auxiliary regularity assumptions beyond control of the score entropy loss. The key technical ingredients include a Girsanov change-of-measure argument, combined with establishing the martingale properties of the sampling dynamics. The latter effectively separates the approximation error from the discretization error, allowing each to be analyzed independently.
For the lower bound, we leverage a log-Sobolev inequality together with a strong data-processing inequality along the uniform noising process. 
To demonstrate the scope of our adaptivity results for masking discrete diffusion, we present several examples whose analysis requires control of information-theoretic quantities, which may be of independent interest.

\subsection{Notation}

For a positive integer \(n\), we define \([n] \defn \{1, \ldots, n\}\) and let \(I_n \in \bR^{n \times n}\) denote the identity matrix. 
Let \(d > 0\) denote the number of dimensions, \(S > 0\) denote the vocabulary size and \(T > 0\) denote the time horizon.
Let \(\mask\) denote a special value outside of \([S]\). Let \(\cX \defn \cV^d\) denote the domain, where, depending on the context, \(\cV \defn [S]\) or \(\cV \defn [S] \cup \{\mask\}\).
We denote the set of all distributions on \(\cX\) by \(\cP(\cX)\). Let \(\ent\), \(\KL\), and \(\mi\) denote \emph{entropy}, \emph{Kullback-Leibler (KL) divergence}, and \emph{mutual information}, respectively. Let \(\delta_x\) denote the Dirac measure at point \(x\). We adopt the standard asymptotic notation \(O(\cdot), {\Omega}(\cdot),\Theta(\cdot)\), \(\lesssim\), and \(\ll\). Additionally, $\widetilde{O}(\cdot), \widetilde{\Omega}(\cdot)$, and $\widetilde{\Theta}(\cdot)$ are defined analogously, except that the logarithmic dependency on $d, S$, $1/\varepsilon$, and \(1 / \delta\) is hidden. For a vector \(x = (x^1, x^2, \ldots, x^d) \in \cX\), \(i \in [d]\), and \(c \in \cV\), we define vectors \(x^{-i} \defn (x^1, \ldots, x^{i - 1}, x^{i + 1}, \ldots, x^d)\), and \(x \oplus_i c\), \(x \odot_ic \in \cX \) as follows:
\begin{itemize}
    \item  \((x \oplus_i c)^j = x^j\) for all \(j \neq i\), and \((x \oplus_i c)^i = (x^i + c) \bmod \abs{\cV}\) \footnote{In this case, we assume that \(\cV\) has additive structure. We only apply this notation when \(\cV = [S]\). We use the convention that \(0 \bmod S = S\).},
    \item \((x \odot_i c)^j = x^j\) for all \(j \neq i\), and \((x \odot_i c)^i = c\),
\end{itemize}
For \(x, y \in \cX\), we denote the Hamming distance by \(
    \ham(x, y) \defn \abs{\{i \in [d]: x^i \neq y^i\}} \), and for \(x \in ([S] \cup \{\mask\})^d\), we denote \(m(x) \defn \abs{\{i \in [d], \text{ such that } x^i = \mask\}}\).

\begin{table}[t]
    \centering
    
    \renewcommand\arraystretch{1.5} 
    \fontsize{7}{10}\selectfont
    \begin{tabular}
    {>{\centering\arraybackslash}m{2.6cm}|>
    {\centering\arraybackslash}m{1.5cm}|>
    {\centering\arraybackslash}m{2cm}|>
    {\centering\arraybackslash}m{1.8cm}|>
    {\centering\arraybackslash}m{1.5cm}|>
    {\centering\arraybackslash}m{1.8cm}|>
    {\centering\arraybackslash}m{1.7cm}}
    \hline
    \textbf{Paper} & 
    \textbf{Noising process} & 
    \textbf{Additional Assump.\  on \(\wh s_t\)}&
    \textbf{No Early Stopping}&
    \textbf{Sampler} &
    \textbf{Iteration Complexity}&
    \textbf{Adaptation}\\
    \hline
    \rule{0pt}{-3pt}\makecell{\vspace{0pt}\citet{ren2024discrete}}& Uniform & Bounded & \XSolidBrush&
    $\tau$-leaping & $d^2 S^2/\varepsilon$ 
    &{\centering\XSolidBrush}\\
    \hline
    \rule{0pt}{-3pt}\citet{liang2025discrete}& Uniform & Bounded & \XSolidBrush &
    $\tau$-leaping & $d^2 S/\varepsilon$ 
    &{\centering\XSolidBrush}\\
    \hline
    \rowcolor{gray!20}
    \rule{0pt}{5pt}\makecell{\vspace{-1pt}Our work,  \\Theorems~\ref{thm:uniform}\&\ref{thm:uniform_lower}}& Uniform & No requirement & \Checkmark &
    $\tau$-leaping & $d/\varepsilon$\(\ ^\ast\)
    &{\centering\XSolidBrush}\\
    \hline\hline
    \rule{0pt}{-5pt}\citet{liang2025absorb} & Masking& Bounded & \XSolidBrush &
    $\tau$-leaping & $d S/\varepsilon$ 
    &{\centering\XSolidBrush}\\
    \hline
    \citet{conforti2025non}& Masking & Small \(L_2\) error &
    \XSolidBrush
    &DMPM&
    $dS/\varepsilon$
    & {\centering\XSolidBrush}\\
    \hline
    \rowcolor{gray!20}
    \makecell{Our work, \\Theorem~\ref{thm:masking_main}}& Masking & 
    No requirement
    &\Checkmark&
    Algorithm~\ref{alg:modified_ttl}
    & $\cD/\varepsilon$
    &\Checkmark\\
    \hline
    \end{tabular}
    \caption{
    Comparison with prior work. 
    Logarithmic factors in the iteration complexity are omitted. \cite{ren2024discrete} and \cite{liang2025absorb} describe bounds without early stopping under more stringent assumptions on the target distribution, the score function, or the score estimator. The assumption on the approximation error in~\cite{conforti2025non} requires both the score entropy loss \(\cL_{\mathrm{SE}}\) and the $L_2$ error \(\|\wh s_t - s_t\|^2\) to be small at discretization points.
    The quantity \(\cD\) (defined in~\Cref{def:etc}), is upper bounded by \(d \log S\) and captures the intrinsic low-dimensional structure of the target distribution. 
    The entry marked with \(\,\ast\) indicates a sharp rate, with the matching lower bound established in \Cref{thm:uniform_lower}.
    \label{tab:comparison}
    }
\end{table}

%% file: setting.tex
\section{Preliminaries of discrete diffusion models}

\subsection{A continuous-time Markov chain formulation}

Our goal is to model $d$-dimensional discrete data $X_0 = (X_0^1, X_0^2, \ldots, X_0^d) \in [S]^d$. 
Let $q_{\mathrm{data}} = q_0$ denote the probability mass function (p.m.f.) of $X_0$ from which we aim to sample, and let $q_0^i$ be the marginal p.m.f. of the $i$-th coordinate.
Analogous to continuous diffusion models, their discrete counterparts consist of a forward and a reverse process over the discrete space. 

\paragraph{The forward process.} We define a forward noising process that progressively transforms the data distribution $q_0$ to a distribution $q_T$ that is close to an easy-to-sample distribution. This process is modeled using a continuous-time Markov chain (CTMC).
\begin{definition}[Continuous-time Markov chain]
    A \emph{CTMC} with an initial distribution \(q_0\) and rate matrices \((Q_t)_{t \in [0, T]}\) is a right-continuous stochastic process \((x_t)_{t \in [0, T]}\) such that 
    \begin{itemize}
        \item \((x_t)_{t \in [0, T]}\) satisfies the \emph{Markov property}: for any \(0 \leq s < t \leq T\), the conditional distribution of \(x_t\) given the history \(\{x_u, u \leq s\}\) depends only on \(x_s\),
        \item for any \(0 \leq t < T\), the transition probabilities satisfy, as \(\Delta t \to 0^+\):
        \begin{equation}
        \label{eq:qt_conditionals}
        \Pr(x_{t+\Delta t} = y \mid x_t = x) = \bI\{x = y\} + Q_t(x, y) \Delta t + o(\Delta t).
        \end{equation}
    \end{itemize}
     Here, the rate matrices satisfy $Q_t(x, y) \geq 0$ for all $x \neq y \in \mathcal{X}$ and $Q_t(x, x) = -\sum_{y \neq x} Q_t(x, y)$.
\end{definition}
We refer to~\cite{feller1940integro,feinberg2014solutions} for a rigorous treatment of CTMCs.
For a given $q_0$, the marginals \((q_t)_{t \in [0, T]}\) satisfying~\Cref{eq:qt_conditionals} are the solutions to the \emph{Kolmogorov forward equation}: $$\frac{\d q_t} {\d t} = Q_t^\top q_t, \quad \text{ for } 0 \leq t \leq T.$$
\paragraph{The reverse process.} For such a CTMC, there exists a time-reversed process with an initial distribution \(q_T\), rate matrices \((\back Q_t)_{t \in [0, T]}\), and marginals \((\back q_t)_{t \in[0, T]}\), such that $q_t \equiv \back q_{T - t}$, for $t \in [0, T]$.
The forward and reverse rate matrices are explicitly related~\citep{campbell2022continuous} by
\begin{equation}
\label{eq:Qback}
    \back Q_{t}(x, y) = Q_{T - t}(y, x) \frac{q_{T - t}(y)}{q_{T - t}(x)}, \quad \text{ for } x\neq y \in \cX \text{ and } 0 \leq t \leq T.
\end{equation}
In this paper, we focus on rate matrices that satisfy three conditions:
\begin{enumerate}
    \item they are time-homogeneous, $Q_t \equiv Q$,
    \item \(Q(x, y) = 0\) whenever \(\ham(x, y) \geq 2\),
    \item if \(\ham(x, y) = 1\) and \(x^i \neq y^i\), then \(Q(x, y) = Q^\mathrm{tok}(x^i, y^i)\), for some fixed matrix \(Q^\mathrm{tok}\).
\end{enumerate}
In particular, we consider two important instances of CTMCs that are widely adopted in practice, namely the \emph{uniform noising process} and the \emph{masking} (or absorbing) \emph{noising process}, which are defined through the choice of \(Q^\mathrm{tok}\).
\begin{itemize}
     \item \textbf{uniform noising process}: A CTMC is a \emph{uniform noising process} if for \(a \neq b \in [S]\)
    \begin{align}
    \label{def:uniform}
        Q^{\mathrm{tok}}(a, b) = 1 / S.
    \end{align}
    This CTMC converges to the uniform distribution on the domain $\cX \defn [S]^d$ in the limit. 

    \item \textbf{masking noising process}:
    A CTMC on the domain \(\cX \defn ([S] \cup \{\mask\})^d\) is a \emph{masking noising process} if for \(a \neq b \in [S] \cup \{\mask\}\)
    \begin{align}\label{def:masking}
        Q^{\mathrm{tok}}(a, b) = \bI\{a \neq \mask \text{ and } b = \mask\}.
    \end{align}
    The corresponding CTMC converges to the Dirac measure \((\delta_{\mask})^{\otimes d}\) as \(t \to \infty\). 
     Note that we constrain the initial distribution \(q_0\) to be supported on non-masked data, i.e., on \([S]^d\).
\end{itemize}

\subsection{Score estimation}

Recall that the reverse process is a CTMC with rate matrices satisfying the relation~\eqref{eq:Qback}, which is similar to the reverse process in the continuous case. The density ratio here generalizes the typical score function $\nabla_x \log q_t(x)$ in the continuous case and is often referred to as the (concrete) score function for discrete diffusion models \citep{meng2022concrete}. Formally, we define \emph{the score function} \(s_t(y, x)\) as $$s_t(y, x) = \frac{q_t(y)}{q_t(x)}, \quad \text{for } x \neq y \in \cX.$$

\paragraph{Score entropy loss. }For both the uniform and masking noising processes,
the marginals \((q_t)\), and consequently the score function, are intractable in general. In practice, one therefore resorts to an approximation \(\wh s_t(y, x)\) of the true score function \(s_t(y, x)\), which is learned from data sampled from the target distribution \(q_0\). 
To evaluate the quality of the estimated score, a widely used loss function is the \emph{score entropy loss}, originally introduced in \cite{lou2023discrete}, which has since become the de facto standard for training score-based discrete diffusion models. This loss provides a principled objective for matching the approximate score $\widehat{s}_t$ to the true score induced by the forward diffusion process.
Specifically, for \(t \geq 0\) and functions \(\wh s, s: \cX \times \cX \to \bR_+\), the score entropy loss \(\cL_{\mathrm{SE}}\) is defined as follows:
    \begin{align*}
        \cL_{\mathrm{SE}}(t, \wh s, s) &\defn
        \bE_{x \sim q_t} \Big[\sum_{y \neq x}Q_t(y, x) s(y, x) \breg(\wh s(y, x), s(y, x))\Big] \geq 0. 
    \end{align*}
    Here, for 
    \(a,b \geq 0\), the Bregman divergence for \(\phi(a) = - \log a\) is given by $$\breg(a, b)\defn \frac{a}{b} - 1 - \log \frac{a}{b} \geq 0.$$ 

In practice, to implement any sampling algorithm, one has to discretize the continuous dynamics and obtain score estimates at discrete time steps.
Suppose that the score estimates $\wh s_{T - t}$ are obtained at discrete time points \(0 \leq t_0 < t_1 < \ldots < t_N \leq T\).
We make the following standard assumption regarding the score estimation errors. 
\begin{assumption}[Approximation error]
\label{asm:escore}
    Let \(N > 0\) and \(0 \leq t_0 < t_1 < \ldots < t_N \leq T\). We assume that
    \begin{equation}\label{eqn:escore}
        \sum_{k=0}^{N - 1} (t_{k+1} - t_k) \lse(T - t_k, \wh s_{T - t_k}, s_{T - t_k}) \leq \escore.
    \end{equation}
\end{assumption}
This assumption is concerned with the aggregated estimation errors over all $N$ steps.
Several works have constructed estimates that satisfy this assumption; examples include~\cite{lou2023discrete,ou2024your,benton2024denoising}.

\subsection{Score-based sampling algorithms}

Armed with the score estimates \((\wh s_{T - t})_{t \in \{t_0,\ldots, t_N\}}\), 
we aim to construct a generative model $\widehat{q}_0$ that approximates the data distribution $q_0$. A natural approach proposed in \citet{campbell2022continuous} is to define a surrogate CTMC that starts from an easy-to-sample distribution $p_0 \approx q_T$ and approximates the backward dynamics in~\eqref{eq:Qback}. Concretely, we define the time-inhomogeneous rate matrix
\begin{equation}
\label{eq:Qhat}
    \wh Q_t(x, y) = Q_{T - t}(y, x) \wh s_{T - t}(y, x).
\end{equation}
In practice, score estimates are only available on a fixed discretization $\tau = (t_0,\ldots,t_N)$, and extending these estimates to the full interval $[0,T]$ introduces 
\emph{discretization error}.

\paragraph{$\tau$-leaping algorithm.} As mentioned above, a widely adopted sampler is the \(\tau\)-leaping algorithm \citep{campbell2022continuous}, which approximates~\Cref{eq:Qhat}
with multiple possible transitions within each discretization interval. 
Formally,  for \(k \in \{0, \ldots, N - 1\}\) and \(t \in [t_k, t_{k+1})\), given \(x_{t_k}\) and \(\wh s_{T - t_k}\), {\(\tau\)-leaping} obtains \(x_{t_{k+1}}\) as a random vector whose coordinates are sampled independently via \(d\) one-dimensional CTMCs. For each \(i \in [d]\), the initial distribution is \(\delta_{x_{t_k}^i}\) and the rate matrices are given by
\footnote{The algorithm admits an equivalent Poisson formulation, in which $dS$ Poisson random variables corresponding to coordinate-value transitions are sampled and applied in parallel.}:
    \begin{equation}
    \label{eq:tau-leaping}
        \wh Q_t^i(a, b) = \wh Q_{t_k}(x_{t_k}, x_{t_k} \oplus_i (b - a)), \qquad \text{ for } a \neq b \in \cV.
    \end{equation}
The formulation in~\Cref{eq:tau-leaping} requires either an additive structure on the state space or the restriction that each coordinate undergoes at most one transition between discretization points. Existing analyses for uniform and masking diffusions~\citep{campbell2022continuous,liang2025discrete} adopt the latter assumption (exactly or with high probability). In Section~\ref{sec:uniform}, we explore the necessity of this requirement for the uniform noising process.
\cite{lou2023discrete} generalizes \(\tau\)-leaping by introducing a class of samplers termed \emph{\(\tau\)-leaping strategies}, which allow arbitrary transformations \(x_{t_{k+1}}^i = f^{i}_k(\wh s_{T - t_k}, x_{t_k})\). Both the Euler method and Tweedie \(\tau\)-leaping fall into this class. However, they remain challenging for direct theoretical analysis due to the absence of a CTMC structure.

\begin{figure}[t]
  \centering
  \begin{minipage}[c]{1.0\linewidth}
      \centering
      \begin{tikzpicture}[
      scale=1.1,transform shape,
        baseline=(current bounding box.center),
        font=\small,
        item/.style={align=left},
        title/.style={font=\bfseries},
      ]

      \def\a{3.2}          
      \def\b{1.45}         
      \def\db{3.0}         
      \def\TopPad{0.05}    
      \def\BotPad{1.05}    
      \def\xMargin{0.9}    
      \def\yStep{0.85}     
      \def\LW{0.95pt}      

      \coordinate (L) at (0,0);
      \coordinate (R) at (\db cm,0);
      \coordinate (M) at ($(L)!0.5!(R)$);

      \begin{scope}
          \clip (L) ellipse (\a cm and \b cm);
          \fill[gray!20] (R) ellipse (\a cm and \b cm);
      \end{scope}

      \draw[line width=\LW] (L) ellipse (\a cm and \b cm);
      \draw[line width=\LW] (R) ellipse (\a cm and \b cm);

      \node[title, anchor=south] at ($(L)+(-0.3cm,\b cm+\TopPad cm-0.05cm)$) {parallel updates};
      \node[title, anchor=south] at ($(R)+(0.4cm,\b cm+\TopPad cm)$) {CTMC-based};

      \node[item, anchor=west] at ($(L)+(-\a cm+\xMargin cm+0.2cm, 0.85cm)$) {Euler method};
      \node[item, anchor=west] at ($(L)+(-\a cm+\xMargin cm -0.6cm, 0.35cm-\yStep cm)$) {Tweedie $\tau$-leaping};

      \node[item, anchor=west] at ($(M)+(-1.12cm, 0.70cm)$) {$\tau$-leaping alg.};
      \node[item, anchor=west] at ($(M)+(-1.42cm, 0.85cm-\yStep cm)$) {Truncated $\tau$-leaping};
      \node[item, anchor=west] at ($(M)+(-1.05cm, 0.95cm-2*\yStep cm)$) {\Cref{alg:modified_ttl}};

      \node[item, color=black!53, anchor=north] at ($(M)+(0,-\b cm)$) {\textbf{$\tau$-bridging}};

      \node[item, anchor=west] at ($(R)+(\a cm-2.05cm-\xMargin cm, 0.6cm)$) {Gillespie's alg.};

      \node[item, anchor=west] at ($(R)+(\a cm-1.75cm-\xMargin cm, 0.8cm-\yStep cm)$) {Uniformization};

      \node[item, anchor=west] at ($(R)+(\a cm-2.15cm-\xMargin cm, 0.0cm-\yStep cm)$) {DMPM};

      \end{tikzpicture}
      \caption{Overview of score-based samplers. The left part comprises score-based samplers that allow parallel updates, defined as \(\tau\)-leaping strategies in~\cite{lou2023discrete}. The right part consists of samplers that can be defined through the CTMC framework. At the intersection are \(\tau\)-bridging strategies, defined in~\Cref{eq:tau-bridging}.
      }
      \label{fig:tau-strategies}
  \end{minipage}
      
\end{figure}
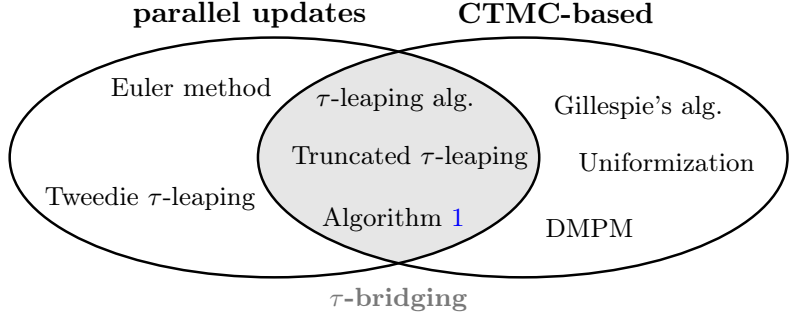

\paragraph{This paper: $\tau$-bridging strategies.} 
We introduce a structured class of samplers that generalizes the $\tau$-leaping algorithm. We name this class of algorithms \emph{$\tau$-bridging strategies}, which retain the parallel updating structure while remaining analytically tractable.
A {$\tau$-bridging strategy} generates $x_{t_{k+1}}$ from $x_{t_k}$ by evolving $d$ independent one-dimensional CTMCs on $[t_k,t_{k+1})$. For each coordinate $i\in[d]$, the chain is initialized at $\delta_{x_{t_k}^i}$ and has the rate matrix given by
\begin{equation}
\label{eq:tau-bridging}
    \widehat Q_t^i = G_t^i(\widehat s_{T-t_k}, x_{t_k}),
\end{equation}
for some mapping $G_t^i : \mathbb R_+^{\mathcal X\times\mathcal X}\times\mathcal X \to \mathbb R^{\mathcal V\times\mathcal V}$.
Compared to general $\tau$-leaping strategies, $\tau$-bridging strategies restrict updates to CTMC-based transitions. This restriction preserves parallel coordinate updates while facilitating theoretical analysis. Figure~\ref{fig:tau-strategies} summarizes the relationships between these classes of sampling algorithms.

\noindent

A representative instance of a \(\tau\)-bridging sampler is the \emph{truncated \(\tau\)-leaping} sampler of~\cite{liang2025discrete}.
For \(k \in [N]\) and \(t \in [t_k, t_{k+1})\), the corresponding rate matrices take the form:
\begin{align}
\label{eq:Qhat_ttl}
G_t^i(\wh s_{T - t_k}, x_{t_k})(a, b) = Q_{T - t_k}(x_{t_k} \odot_i b, x_{t_k})\wh s_{T - t_k}(x_{t_k} \odot_i b, x_{t_k}) \bI\{x_{t_k}^i = a\} \quad \text{for } a \neq b \in \cV. 
\end{align}
The indicator $\bI\{x_{t_k}^i = a\}$ enforces the constraint that at most one transition occurs per coordinate \(i \in [d]\)  within each discretization interval $[t_k, t_{k+1})$. In~\Cref{sec:masking_results}, 
we show that an instance of this scheme achieves sublinear complexity for the masking noising process under mild distributional assumptions. To the best of our knowledge, this is the first result establishing such a guarantee.

%% file: results.tex
\section{Main results}\label{sec:main_results}
In this section, we characterize the sampling efficiency of the $\tau$-bridging strategies for both the uniform and masking noising processes.
We develop sharp convergence guarantees and highlight cases where adaptivity is automatically achieved. We provide proof sketches for all results in this section, with full proofs deferred to the appendix.

\subsection{Uniform discrete diffusion}
\label{sec:uniform}

\subsubsection{A sharp convergence characterization}

We begin with the uniform discrete diffusion models, whose forward dynamics are given by the uniform noising process. We establish explicit sampling guarantees for the $\tau$-leaping algorithm, measured in KL divergence. The proof is given in~Appendix~\ref{sec:proof_uniform}.
\begin{theorem}\label{thm:uniform}
    Let $q_{\mathrm{data}} = q_0$ be the data distribution on $ \cX \defn [S]^d$. For \(0 = t_0 < t_1 < \ldots < t_N = T\), let \(\Delta \defn \max_k \{t_{k+1} - t_k\} = O(1)\). Set $p_0 = \mathrm{Unif}(\cX)$. 
    Under Assumption~\ref{asm:escore}, the $\tau$-leaping algorithm initialized at $p_0$ generates a sample from $p_{\mathrm{output}} = p_{T}$ such that
    \begin{align}\label{eq:uniform_kl}
        \KL(q_{\mathrm{data}} \,\|\, p_{\mathrm{output}}) \lesssim \escore + e^{-T} d \log(S) + \Delta d\log(S/\Delta).
    \end{align}
\end{theorem}

As expected, the KL divergence bound in Theorem~\ref{thm:uniform} decomposes into three terms. 
The first term \(\escore\) quantifies the quality of score estimation and captures the accumulation of estimation errors over the $N$ discretization steps. 
The second term corresponds to the initialization error, arising from initializing the sampler with the uniform distribution $p_0$ instead of the true terminal distribution $q_T$; this term decays exponentially with the diffusion horizon $T$. 
Finally, the third term accounts for the discretization error incurred by approximating the continuous-time reverse process with a discrete-time $\tau$-leaping scheme.

To further interpret Theorem~\ref{thm:uniform} and place it in context with existing results, we highlight several of its salient features.
First, the discretization error scales linearly with the dimension $d$ and only logarithmically with the vocabulary size $S$. 
This matches the result obtained for the random walk model~\citep{conforti2025non} and reveals that the discretization error is insensitive to the distribution scale, as has been shown for continuous diffusion models (e.g., \citet{huang2024denoising}).
Second, the theorem permits a flexible choice of step size schedules and does not require early stopping. In contrast to prior analyses that rely on carefully selected step sizes and introduce an early stopping time $\delta$ (where the algorithm outputs $p_{T-\delta}$ in place of $p_T$), the bound in Theorem~\ref{thm:uniform} depends only on the maximum step size. 
Moreover, the same bound applies uniformly to early stopping variants: the right-hand side of~\eqref{eq:uniform_kl} remains unchanged for any $\delta \ll 1$.

The only requirement we have on score estimation is Assumption~\ref{asm:escore}, with no additional boundedness or regularity conditions (typically assumed in the existing literature). As a result, the theorem applies to a broad class of score estimation procedures commonly used in practice.
We provide a sketch of its proof to illustrate the main proof ideas. 
\begin{proof}[Proof sketch of~\Cref{thm:uniform}]
    In view of the data-processing inequality and the chain rule for KL divergence, we upper bound the KL divergence between \(q_0\) and \(p_T\) by the KL divergence between the paths \(q_{T-t_0,\ldots, T-t_N}\) and \(p_{t_0, \ldots, t_N}\), which can be decomposed as
    \begin{align*}
        \KL(q_0 \,\|\, p_T) &\leq \KL(q_{T - t_0, \ldots, T - t_N} \,\|\, p_{t_0, \ldots, t_N})\\
        &= \KL(q_T \,\|\, p_0) + \sum_{k=0}^{N - 1} \bE_{x_{t_k} \sim \back q_{t_k}} \left[\KL\left(\back q_{t_{k+1}\mid t_k}(\cdot|x_{t_k}) \,\|\, p_{t_{k+1}\mid t_k}(\cdot|x_{t_k})\right)\right].
    \end{align*}
    The first term is the initialization error, which can be upper bounded by the log-Sobolev inequality in Lemma~\ref{lem:uniform_log_sobolev}.
    For the second term, we apply Girsanov's change-of-measure theorem for continuous-time Markov chains to obtain the following upper bound:
    \begin{align*}
        \frac{1}{S} \sum_{k=0}^{N-1} \int_{t_k}^{t_{k+1}} \bE_{x_{t_k}, x_t \sim \back q_{t_k, t}} \sum_{i \in [d]}\sum_{c \in [S]} s_{T - t}(x_t  \oplus_i c, x_t) \breg \big(\wh s_{T - t_k}(x_{t_k} \oplus_i c, x_{t_k}), s_{T - t}(x_t \oplus_i c, x_t)\big)\d t.
    \end{align*}
    The details can be found around Eqn.~\eqref{eq:uniform_kl_main}.

    To further control the right-hand side, we apply the law of cosines for Bregman divergence and derive that (with \(\ell \defn t_k\))
    \begin{align*}
        &\sum_{i \in [d]}\sum_{c \in [S]} s_{T - t}(x_t  \oplus_i c, x_t) \breg \big(\wh s_{T - t_k}(x_{t_k} \oplus_i c, x_{t_k}), s_{T - t}(x_t \oplus_i c, x_t)\big)\\
        &= \underbrace{\sum_{y_\ell: \ham(y_\ell, x_\ell) = 1} s_{T - \ell}(y_\ell, x_\ell)\breg \big(\wh s_{T - \ell}(y_\ell , x_\ell), s_{T - \ell}(y_\ell, x_{\ell})\big)}_{\text{Controlled by Assumption~\ref{asm:escore}}} \\
        &\quad + \underbrace{\sum_{i \in [d]}\sum_{c \in [S]}  \big(s_{T - \ell}(x_\ell \oplus_i c, x_\ell) - s_{T - t}(x_t \oplus_i c, x_t)\big) \log \wh s_{T - \ell}(x_\ell \oplus_i c, x_\ell)}_{\text{Expectation controlled by Lemma~\ref{lem:uniform_martingale}}}\\
        &\quad + \sum_{y_t: \ham(y_t, x_t) = 1} \big(-\log s_t(y_t, x_t)\big) -  \sum_{y_\ell: \ham(y_\ell, x_\ell) = 1}\big(-\log s_{T - \ell}(y_{\ell}, x_{\ell})\big).
    \end{align*}
    The first term can be controlled by Assumption~\ref{asm:escore} after taking the expectation over \(x_{\ell} \sim \back q_{\ell}\) and integrating over time.
    The second term can be shown to be zero with the help of~\Cref{lem:uniform_martingale} after taking the expectation over \(x_t \sim \back q_{t \mid \ell}(\cdot \mid x_\ell)\).
    Thus, the problem boils down to upper bounding the third term above, whose properties are characterized in~\Cref{lem:uniform_log_score_bound}. 
    After taking the expectation and integrating over time, we can upper bound the third term by \(\Delta d \log(S/\Delta)\).
    Combining the bounds for all three terms completes the proof.
\end{proof}

\medskip

Next, we specialize Theorem~\ref{thm:uniform} to the concrete choice of a discretization schedule to derive the iteration complexity required to obtain an $\varepsilon$-accurate sampler in KL divergence.
For a simple step size schedule, it turns out that $d/\varepsilon$ steps (up to logarithmic factors) suffice for convergence, significantly improving on the state-of-the-art complexity of $d^2 S/\varepsilon$ from~\cite{liang2025discrete}. Refer to Appendix~\ref{sec:proof-uniform-cor} for the proof.
\begin{corollary}\label{cor:uniform}
    For the setting in Theorem~\ref{thm:uniform} and $\varepsilon > 0$, the output of the $\tau$-leaping algorithm with constant step size schedule $t_{k+1} - t_k = T/N$ for $k  \in [N - 1]$, achieves 
    \begin{align*}
        \KL(q_{\mathrm{data}} \,\|\, p_{\mathrm{output}}) \lesssim \escore + \varepsilon,
    \end{align*}
    provided that the time horizon $T = \log(d\log(S)/\varepsilon)$ and the iteration number
    \begin{align}
         N = \widetilde{O}\left(\frac{d}{\varepsilon}\right).
    \end{align}
\end{corollary}

\begin{remark}[Step size schedule]
    In Corollary~\ref{cor:uniform}, we adopt the constant step size schedule for simplicity. This choice is optimal in the sense that it minimizes the worst-case upper bound for a fixed number of steps $N$, and it is also empirically effective~\citep{campbell2022continuous}.
    However, other step size schedules commonly used in practice and theory achieve the same iteration complexity, including the exponential-then-constant schedule (defined as in Corollary~\ref{cor:masking} and used in \cite{liang2025discrete}) and the log-linear schedule~\citep{lou2023discrete}. 
    In these works, early stopping is introduced to maintain numerical stability in score estimation during training and also to ensure a small discretization error. Our result shows that, under Assumption~\ref{asm:escore}, early stopping is not necessary for a small discretization error.  
\end{remark}

\subsubsection{A matching lower bound for $\tau$-leaping}

While Theorem~\ref{thm:uniform} establishes an upper bound for the \(\tau\)-leaping algorithm scaling nearly linearly with the dimension $d$ and logarithmically with the vocabulary size $S$, the fundamental question remains: is this dependence an intrinsic limit or merely a technical artifact? We show that the former is indeed the case 
by establishing a matching lower bound.

We note that for target distributions sufficiently close to the uniform distribution, sampling can be achieved with very few steps, as the forward CTMC converges efficiently to its limit. 
To avoid these pathological instances,
we restrict our focus to the class of  distributions that remain sufficiently well-separated from the uniform distribution.
Specifically, for any $\gamma \in [0, 1]$, define the subset \(\cP^{\gamma}(\cX) \subseteq \cP(\cX)\) as
\begin{align*}
    \cP^{\gamma}(\cX) = \big\{q_0 \in \cP(\cX): \ent(q_1) \leq (1-\gamma) \cdot \ent(\Unif(\cX)) = (1-\gamma)d\log(S)\big\},
\end{align*}
where $q_1$ is the marginal distribution at $t = 1$ of the uniform noising process initialized at $q_0$, $\Unif(\cX)$ is the uniform distribution on $\cX$, and $\ent(\cdot)$ denotes the entropy function of a distribution. 
Intuitively, for $\gamma \in (0, 1)$, the class \(\cP^{\gamma}(\cX)\) imposes a structural constraint on the convergence of the forward process; it describes distributions that do not mix rapidly.
In this sense, for $\gamma = O(1)$, \(\cP^{\gamma}(\cX)\) contains distributions that remain informative enough at time $t = 1$ in the forward process. This covers most distributions of practical interest, since they carry nontrivial information characterized by relatively low entropy.

The following lower bound shows that, when sampling from a distribution in \(\cP^\gamma(\cX)\) with the $\tau$-leaping algorithm, the iteration complexity bound in Corollary~\ref{cor:uniform} cannot be improved up to logarithmic factors. 
The proof is given in~Appendix~\ref{sec:uniform_lower_proof}.

\begin{theorem}\label{thm:uniform_lower}
    For any target distribution \(q_0 \in \cP^{\gamma}(\cX)\) and early stopping time $0 \leq \delta \ll 1$, denote the path measure of the backward process by $Q \overset{d}{=} \{\back q_t\}_{t \in [0,T-\delta]}$ and the sampling process by $P \overset{d}{=} \{p_t\}_{t\in[0,T-\delta]}$. Let $\gamma = \Omega(1)$. Then, for any step size schedule \(0 = t_0 < t_1 < \ldots < t_N = T-\delta\) with \(\max_k \{t_{k+1} - t_k\} \le \frac{1}{2}\), it takes the $\tau$-leaping algorithm at least 
    \begin{align}
        N = {\Omega}(d\log(S))
    \end{align}
    iterations to achieve
    \begin{align*}
        \KL(Q \,\Vert\, P) \le \escore + O(1).
    \end{align*}
\end{theorem}

\noindent We make several remarks concerning the nature and implications of our lower bound.

\Cref{thm:uniform_lower} reveals that for informative target distributions in $\mathcal P^\gamma(\mathcal X)$, ensuring that the KL divergence between the sampling process and the reverse process is small requires the number of steps to scale at least linearly with the dimension $d$, which cannot be avoided for general distributions.
In addition, the lower bound is uniform over both early stopping schedules ($0<\delta \ll 1$) and non-early stopping schemes ($\delta=0$).

This lower bound is algorithm-dependent: it relies on structural properties of the $\tau$-leaping algorithm and therefore differs fundamentally from information-theoretic or minimax lower bounds.
In principle, alternative sampling schemes may circumvent the linear dependence on $d$. Indeed, in Section~\ref{sec:masking_results}, we show that a modified $\tau$-leaping procedure achieves sublinear dependence on $d$ for structured target distributions under the masking noising process. Whether analogous improvements are possible for uniform discrete diffusion through modified algorithms remains an open question.

    When the target distribution has high entropy,
    the lower bound need not apply. Indeed, when $q_{\mathrm{data}}$ satisfies $\KL(q_{\mathrm{data}}\Vert \mathrm{Unif}(\cX))=o(d)$, one can show that $\ent(q_1)=\Theta(d\log S)$, and that a sample from a distribution with the KL error at most $\escore+\varepsilon$ can be obtained using $N=o(d)$ steps. A precise formulation of this claim is given in Appendix~\ref{sec:uniform-adaptive}.

We remark that the quantity controlled in Theorem~\ref{thm:uniform_lower} is the KL divergence between two path measures, rather than the divergence between the terminal output distributions, which may appear weaker than the upper bound in Corollary~\ref{cor:uniform}. However, to the best of our knowledge, all existing upper-bound analyses for the KL divergence, including ours, proceed by first bounding the KL divergence between path measures and then invoking the data-processing inequality. Consequently, the lower bound applies to all current analysis techniques. In this sense, Theorem~\ref{thm:uniform_lower} establishes the optimality of the iteration complexity in Corollary~\ref{cor:uniform} within the scope of the existing analysis techniques.

Finally, we provide the proof sketch of \Cref{thm:uniform_lower} to illustrate the main proof techniques. 

\begin{proof}[Proof sketch of~\Cref{thm:uniform_lower}]
    The proof is based on a refined analysis of the decay of KL divergence along the forward process for any distribution $q_0 \in \cP^{\gamma}(\cX)$. While we state our result for $\gamma = \Omega(1)$, the proof works for every $\gamma \in (0,1).$
    It can be shown that the KL divergence along the forward process is a differentiable function of time $t$, and we denote its negative rate of change as $\varphi(t)$, i.e.,
    \begin{align*}
        \varphi(t) = -\frac{\d}{\d t} \KL(q_t \| p_0) = \sum_{x, y: \ham(x, y) = 1} q_t(x) s_{t}(y, x) \log \left(\frac{s_t(y, x)}{s_t(x, y)}\right),
    \end{align*}
    where $p_0 = \mathrm{Unif}(\cX)$ is the limit distribution of the forward noising process.
    First, we show that the condition $\KL(Q \| P) \leq \escore + O(1)$ with the definition of $\cP^\gamma(\cX)$ implies that $T > 1$ and the following bound
    \begin{align}\label{eq:lower_phi_condition}
        \sum_{k=1}^{N-1} \int_{t_k}^{t_{k+1}} \big(\varphi(T - t) - \varphi(T - t_k)\big) \d t = O(1).
    \end{align}
    Furthermore, we can show that $\varphi(t)$ is a non-increasing and differentiable function of $t$.
    Thus,~\Cref{eq:lower_phi_condition} and the Newton-Leibniz formula lead to a stronger condition:
    \begin{align}
        \sum_{k=1}^{N-1} \inf_{t_k \le t \le t_{k+1}} (-\varphi'(T-t)) \cdot \frac{1}{2} (t_{k+1} - t_k)^2
        &\leq \sum_{k=1}^{N-1} \int_{t_k}^{t_{k+1}} \int_{T - t}^{T - t_k} -\varphi'(u) \d u \d t\notag\\
        &= \sum_{k=1}^{N-1} \int_{t_k}^{t_{k+1}} \big(\varphi(T - t) - \varphi(T - t_k)\big) = O(1).\label{eq:lower_phi'_condition}
    \end{align}
    Next, we view the forward process as an $S$-ary symmetric channel \citep{makur2018comparison} and apply the strong data-processing inequality to
    prove that for any $q_0 \in \cP^{\gamma}(\cX)$, the function $-\varphi'(t)$ has a lower bound scaling with $\gamma d \log(S) $ for all $t \in (0, 1)$.
Since $\max_k\{t_{k+1} - t_k\}  \leq \frac{1}{2}$, we can choose a suitable $M$, such that $1 < M < N$ and $T - t_M \in [\frac{1}{2}, 1]$.
    Combining this with Eqn.~\eqref{eq:lower_phi'_condition}, we obtain
    \begin{align*}
        \sum_{k=M}^{N-1} (t_{k+1} - t_k)^2 \lesssim \frac{1}{\gamma d \log(S)},
    \end{align*} 
    which implies that $N = \Omega(\gamma d \log(S))$ by the Cauchy-Schwarz inequality.
\end{proof}

\subsection{Masking discrete diffusion}
\label{sec:masking_results}

We now turn our attention to the masking noising process. 
Our main result in this setting is an upper bound that intrinsically depends on the structural properties of the target distribution \(q_{\mathrm{data}}\), 
rather than scaling with the ambient dimension $d$.
This aligns with the intuition that for highly structured distributions --- such as a sparse mixture of Dirac measures --- a sensible sampler should converge at a sublinear scale, or perhaps even logarithmically with $d$.

\subsubsection{Preliminaries}\label{sec:pre-info}
We begin by recalling two fundamental quantities in information theory: the total correlation and the dual total correlation.
For a distribution \(q\) over \([S]^d\) and \(x \sim q\),  {the total correlation} \(\cC(q)\) and {the dual total correlation} \(\cB(q)\) are defined as 
\begin{equation}
\label{def:dtc}
\cC(q) \defn \sum_{i=1}^d \ent(x^i) - \ent(x) \text{ and }
    \cB(q) \defn \ent(x) - \sum_{i=1}^d \ent(x^i \mid x^{-i}).
\end{equation}
We now introduce a time-dependent quantity associated with the masking noising process.
Consider a masking noising process defined by~\Cref{def:masking} with marginals \((q_t)_{t \geq 0}\). For \(x \in ([S] \cup \{\mask\})^d\) and \(i \neq j \in [d]\), let \(x^{-(i, j)}\) denote the collection of all unmasked elements of \(x\), excluding the \(i\)-th and the \(j\)-th coordinates. 
We define the \emph{effective total correlation} of the target distribution as 
    \begin{align}
        \label{def:etc}
        \cD(q_0) \defn \int_0^{\infty} \min(1, t) \cI(t) \d t
        \quad \text{with} \quad 
        \cI(t) \defn \sum_{i \neq j \in [d]} \mi(x_t^i; x_t^j \mid x_t^{-(i, j)}) \geq 0,
    \end{align}
    where $\mi(A ; B \mid C)$ denotes the conditional mutual information, and \(x_t \sim q_t\).  
\Cref{lem:int_dtc_tc} shows that the total correlation and the dual total correlation can be expressed through \(\cI(t)\) by
\begin{equation*}
    \cB(q_0) = \int_0^\infty \cI(t) \d t \quad \text{and} \quad \cC(q_0) = \int_0^\infty (e^t - 1) \cI(t) \d t.
\end{equation*}
Consequently, \(\cD(q_0) \leq \min(\cB(q_0), \cC(q_0))\).
The statement and the proof of this result are given in~Appendix~\ref{sec:proof-lemma-int}.

 Note that both \(\cB(q_0), \cC(q_0)\), and hence \(\cD(q_0)\) are upper bounded by \(d \log(S)\). Moreover, there exist distributions \(q_0\) with \(\cB(q_0) = O(1)\) while \(\cC(q_0) = \Omega(d \log(S))\), and vice versa. We refer to~\cite{austin2018multi} for a detailed study of the total correlation and the dual total correlation. Importantly, there are also natural distributions for which both \(\cB(q_0)\) and \(\cC(q_0)\) are of order \(d \), while \(\cD(q_0)\) remains small. See Proposition~\ref{prop:structure-with-noise} for an example of such a distribution.

\subsubsection{An adaptive characterization }
\label{sec:adaptive-masking}
\begin{algorithm}[t]
\SetAlgoLined
\DontPrintSemicolon 
\caption{Modified truncated \(\tau\)-leaping}
\label{alg:modified_ttl}

\SetKwInput{KwInput}{Input}
\SetKwInput{KwOutput}{Output}

\KwInput{\\
Initial distribution: \(p_0\), \\
Discretization steps: \(0 = t_0 < t_1 < \ldots < t_N = T\), \\
Score estimate function: \(\wh s_{T - t}\) for \(t \in \{t_0, \ldots, t_{N - 1}\}\).}
\KwOutput{Sample \(\wh x \in [S]^d\).}

Sample \(x_0\) from \(p_0\)

\For{\(k = 0, \ldots, N - 1\)}{
    \For{\(i \in m(x_{t_k}):= \{i\text{, such that } x_{t_k}^i = \mask\}\) }{
        \(\wh Q_k^i(a) \gets \wh s_{T - t_k}(x_{t_k} \odot_i a, x_{t_k})\),\quad for 
        \(a \in [S]\)\;
    
        \(\wh Q_k^i (\mathrm{MASK}) \gets - \sum_{a \in [S]} \wh Q_k^i(a)\)

        \If{\(k < N - 1\)}{
        \(\Delta_k \gets (e^{T - t_k} - 1) \log \left(\frac{e^T - e^{t_k}}{e^T - e^{t_{k+1}}}\right)\)  

        \(\cP_k \gets \exp(\wh Q_k^i(\mathrm{MASK}) \Delta_k)\)
        }
        \Else{ 
        \(\cP_k \gets 0\)
        }

        \(x_{t_{k+1}}^i \gets \begin{cases}
            \mathrm{MASK}, &\quad \text{with probability } \cP_k, \\
            a, &\quad \text{with probability } \frac{\wh Q^i_k(a)}{\sum_{b \in [S]} \wh Q_k^i(b)}(1 - \cP_k), \text{ for } a \in [S].
        \end{cases}\) \label{step:bla}
    }
}
\Return{$x_{t_N}$}
\end{algorithm}

Equipped with the above preliminaries, we present our main result on the masking noising process. The proof is given in~Appendix~\ref{sec:proof-masking}.
\begin{theorem}
\label{thm:masking_main}
Let \(q_{\mathrm{data}} = q_0\) be the target distribution on \([S]^d\). For \(0 = t_0 < t_1 < \ldots < t_N = T\), let \(h_k \defn t_{k+1} - t_k\) be the step sizes and assume that \(\Delta \defn \max_k h_k = O(1)\). Let
    \begin{align*}
        p_0 \defn \left(\big(1 - e^{-T}\big) \delta_{\mask} + S^{-1}e^{-T}\sum_{k=1}^S \delta_k\right)^{\otimes d}.
    \end{align*}
    Under~\Cref{asm:escore},~\Cref{alg:modified_ttl} initialized at \(p_0\) 
    produces a sample from \(p_{\mathrm{output}} = p_T\) such that
    \begin{equation}
    \label{eq:kl_ub_masking}
        \KL(q_{\mathrm{data}} \,\|\, p_{\mathrm{output}}) \lesssim \escore + e^{-T} d \log(S) + \sum_{k=0}^{N - 1} h_k \int_{T - t_{k+1}}^{T - t_k} \cI(t) \d t.
    \end{equation}
\end{theorem}

\noindent
A few remarks on the consequences and implications of Theorem~\ref{thm:masking_main} are in order. 

As in Theorem~\ref{thm:uniform}, the last term in the upper bound
corresponds to the discretization error measured using the integrated mutual information defined in~\Cref{def:etc}. 
While the first two terms are generic, the third term governs the dependence on the dimension $d$ and reflects the information-theoretic properties of the target distribution.  
For structured distributions, our algorithm implicitly adapts to the underlying structure of the target distribution without requiring any prior knowledge of that structure or any modification to the algorithm itself.
 
In~Appendix~\ref{app:mask_ttl}, we analyze the performance of truncated $\tau$-leaping as an alternative to~\Cref{alg:modified_ttl}, which has an additional ${d}/{N^2}$ term in the upper bound~\Cref{eq:kl_ub_masking}, ignoring lower-order contributions. Although for structured target distributions the resulting iteration complexity already scales as $\sqrt d$ rather than $d$ (as in the existing literature), it does not fully adapt to the geometry of the target distribution.
To provide some intuition, the standard (or truncated) \(\tau\)-leaping algorithm informally satisfies for \(t \in [t_k, t_{k+1})\) (see~\Cref{eq:Qhat_ttl})
\begin{equation}
    G_t^i(s_{T - t_k}, x_{t_k}) \approx G_{t_k}^i(s_{T - t_k}, x_{t_k}), \quad \text{and thus} \quad \wh Q_t \approx \back Q_{t_k},
\end{equation}
where we recall the mapping \(G_t^i\) from~\Cref{eq:tau-bridging}. 
That is, even when the score estimation is exact, $\widehat s_{T-t_k}\equiv s_{T-t_k}$, the $\tau$-leaping algorithm introduces a mismatch between the surrogate and true rate matrices as \(s_{T - t_k} \nequiv s_{T - t}\).
Algorithm~\ref{alg:modified_ttl} corrects this discrepancy by enforcing
\begin{equation}
    G_t^i(s_{T - t_k}, x_{t_k}) \approx G_t^i(s_{T - t}, x_{t_k}), \quad \text{and thus} \quad \wh Q_t \approx \back Q_t,
\end{equation} 
through the rescaling of the score estimate function: \(\wh s_{T - t} = \frac{e^{T - t_k} - 1}{e^{T - t} - 1} \wh s_{T - t_k}\). As it is a linear transformation of the score estimate function, we can simulate its dynamics only at discrete points \(T - t_0, \ldots, T- t_N\) (see~\Cref{alg:modified_ttl} and~\Cref{lem:alg_is_ctmc}).
This leads to a sharper upper bound in Theorem~\ref{thm:masking_main} relative to the analogous bound for truncated $\tau$-leaping (Theorem~\ref{thm:masking_ttl}; see also Remark~\ref{rem:masking_ttl}).
Empirically, the benefit of rescaling the score function in masking discrete diffusion models has also been observed in prior work; see, for example,~\cite{lou2023discrete,ou2024your}.

Notably, our results are closely connected to an intriguing parallel line of work on the masking diffusion models~(\cite{libreaking,chen2025optimal}), which focuses on the design of unmasking schedules without adopting a CTMC perspective. In particular, \cite{chen2025optimal} derives optimal unmasking schedules and discusses two representative instances in which the number of steps scales linearly with  \(\cB(q_{\mathrm{data}})\) and \(\cC(q_{\mathrm{data}})\), respectively. Their algorithms require an a priori estimate of \(\cB(q_{\mathrm{data}})\) and \(\cC(q_{\mathrm{data}})\) or a doubling search procedure to calibrate the unmasking schedule and rely on a different sampling mechanism. The fact that our score-based sampler automatically exploits similar information-theoretic quantities without additional hyperparameters underscores both the fundamental nature of these quantities and the robustness of the CTMC framework. 

Below we provide a proof sketch of~\Cref{thm:masking_main} with the details deferred to Appendix~\ref{sec:proof-masking}.

\begin{proof}[Proof sketch of~\Cref{thm:masking_main}]
First,~\Cref{lem:alg_is_ctmc} shows that~\Cref{alg:modified_ttl} outputs a sample from a CTMC with initial distribution \(p_0\) and rate matrices
\begin{equation}
    \wh Q_t(x, y) \coloneqq \begin{cases}
    \wh s_{T - t_k}(x_{t_k} \odot_{i} y^{i}, x_{t_k} ) \frac{e^{T - t_k} - 1}{e^{T - t} - 1}\bI\{x^{i} = \mathrm{MASK}\}, &\quad \text{if } \ham(x, y) = 1, x^{i} \neq y^{i}, \text{ and } x_{t_k}^{i} = \mask, \\
    -\sum_{z \neq x} \wh Q_t(x, z), &\quad \text{if } y = x, \\
    0, &\quad \text{otherwise.}
    \end{cases}
\end{equation}
This corresponds to a \(\tau\)-bridging strategy (\Cref{eq:tau-bridging}) with the following function \(G^i_t(\wh s_{T - t_k}, x_{t_k})\):
\begin{equation*}
    G^i_t(\wh s_{T - t_k}, x_{t_k})(a, b) = \frac{e^{T - t_k} - 1}{e^{T - t} - 1}\wh Q_{T - t_k}(x_{t_k}, x_{t_k} \odot_i b) \bI\{x^i_{t_k} = a\} \quad \text{for } a \neq b \in \cV.
\end{equation*}
    By the data-processing inequality, we upper bound the KL divergence between \(q_0\) and \(p_T\) by the KL divergence between the paths \(q_{T-t_0,\ldots, T-t_N}\) and \(p_{t_0, \ldots, t_N}\). 
    Next, we apply the Markovian property of the paths along with Girsanov's change-of-measure theorem to upper bound \(\KL(q_0 \| p_T)\) by
    \begin{align*}
\KL(q_T \| p_0) + \sum_{k=0}^{N - 1} \int_{t_k}^{t_{k+1}} \bE_{x_{t_k}, x_t \sim \back q_{t_k, t}} \Bigg[\sum_{y_t: Q(y_t, x_t) > 0}&s_{T - t}(y_t, x_t)\\
&\times \breg \left(\frac{e^{T - t_k} - 1}{e^{T - t} - 1}\wh s_{T - t_k}(y_t, x_t), s_{T - t}(y_t, x_t)\right)\d t\Bigg].
    \end{align*}
    The first term is the initialization error and is controlled by choosing the time horizon \(T = \Omega(\log d + \log \log (\varepsilon^{-1}S))\).
    For the second term, we apply the law of cosines for Bregman divergence and obtain (with \(\ell \defn t_k\) and \(y_\ell \defn x_{\ell} \odot_i c\), where \(y_t = x_t \odot_i c\)): 
    \begin{align*}
        &s_{T - t}(y_t, x_t)\breg \left(\frac{e^{T - \ell} - 1}{e^{T - t} - 1}\wh s_{T - \ell}(y_t, x_t), s_{T - t}(y_t, x_t)\right)\\
        &= \underbrace{\frac{e^{T - \ell} - 1}{e^{T - t} - 1} s_{T  - \ell}(y_{\ell}, x_\ell)\breg \left(\wh s_{T - \ell}(y_\ell, x_t), s_{T - \ell}(y_t, x_t)\right)}_{\text{Controlled by~\Cref{asm:escore}}} 
         + \underbrace{\left(s_{T - t}(y_\ell, x_\ell) - s_{T - t}(y_t, x_t)\right) \log \frac{\wh s_{T - \ell}(y_\ell, x_\ell)}{s_{T - \ell}(y_\ell, x_\ell)}}_{\text{Expectation controlled by~\Cref{lem:mask_mart}}}\\
        &\qquad \qquad \qquad \qquad + s_{T - t}(y_t, x_t)\breg \left( s_{T - t}(y_\ell, x_\ell), s_{T - t}(y_t, x_t)\right). \\
    \end{align*}
    Similar to the proof of the uniform discrete diffusion model, the first term can be controlled by Assumption~\ref{asm:escore} after taking the expectation over \(x_{t_k} \sim \back q_{t_k}\) and integrating over time,
    and the second term can be proved to be zero by the martingale property from~\Cref{lem:mask_mart}.
    Finally, using Dynkin's formula, we relate the third term to the effective total correlation 
    \(\cD(q_0)\).
\end{proof}

Next, we derive iteration complexity guarantees for our algorithm under specific choices of step size schedules. The proof is given in~Appendix~\ref{sec:proof-masking-cor}.
\begin{corollary}
\label{cor:masking}
    Consider the setting in~\Cref{thm:masking_main}. Let \(T = \log (d \log(S))\). For a fixed \(\varepsilon > 0\), the distribution \(p_{\mathrm{output}}\) satisfies $\KL(q_{\mathrm{data}} \,\|\, p_{\mathrm{output}}) \lesssim \escore + \varepsilon$,
    \begin{itemize}
        \item under the constant step size schedule, \(t_{k} - t_{k-1} = T / N\) for all \(k \in [N]\), provided 
        \begin{align*}
            N = \wt O\left(\frac{\cB(q_{\mathrm{data}})}{ \varepsilon}\right);
        \end{align*}
        \item under the exponential-then-constant step size schedule, when \(t_{k+1} - t_k \leq \kappa \min(1, T - t_{k+1})\) for \(k \in \{0, \ldots, N-2\}\), \(T - t_{N - 1} = \varepsilon / (d \log(S))\), and $\kappa = N^{-1}(T + \log (\varepsilon^{-1} d \log(S)))$, provided 
        \begin{align*}
            N = \wt O\left(\frac{\cD(q_{\mathrm{data}})}{\varepsilon}\right) \leq \wt O\left(\frac{\min \{\cB(q_{\mathrm{data}}), \cC(q_{\mathrm{data}})\}}{\varepsilon}\right).
        \end{align*}
    \end{itemize}
\end{corollary}

In words, Corollary~\ref{cor:masking} shows that the sampling complexity of Algorithm~\ref{alg:modified_ttl} required to obtain an $\varepsilon$-accurate distribution is governed by intrinsic complexity measures of the target distribution. 
Under the constant step size schedule, the iteration complexity is controlled by the dual total correlation of the target distribution, whereas under the exponential-then-constant schedule, the effective total correlation becomes the relevant quantity.
For illustration, let us consider the following two simple examples. 
\begin{itemize}
    \item  Consider first the uniform distribution on $[S]^d$. In this case, both complexity measures scale independently of the ambient dimension $d$, which means
    \begin{align}
        N = \widetilde{O}\Big(\frac{1}{\varepsilon}\Big),
    \end{align}
    reflecting the fact that it is exceptionally easy to sample from uniform distributions. While intuitive in hindsight, this phenomenon has not been previously formalized in the literature.

    \item As a second example, consider a mixture of two Dirac measures, $\frac{1}{2}\delta_{k_1} + \frac{1}{2}\delta_{k_2}$. A direct calculation shows that the dual total correlation remains independent of $d$, which means
     \begin{align}
        N = \widetilde{O} \Big(\frac{1}{\varepsilon}\Big),
    \end{align}
    indicating that such distributions are also handled automatically by our algorithm.
\end{itemize}

To further illustrate the implications of \Cref{thm:masking_main}, we consider some representative distributions for which one or more of the quantities \(\cB(q_{\mathrm{data}}), \cC(q_{\mathrm{data}})\), or $\cD(q_{\mathrm{data}})$ are small.
Since the iteration complexity scales linearly with these quantities, our result shows that discrete diffusion models can provably achieve efficient sampling.
Appendix~\ref{sec:examples} develops these examples in detail and provides rigorous proofs of the stated claims. 
\begin{itemize}
    \item 
    
\noindent \textbf{Hidden Markov models.} Here, the observed variables correspond to words or tokens in a sentence, while the hidden states encode latent semantic topics. 
Under the natural assumption that topics evolve slowly, we show that  \(\cB(q_\mathrm{data})\) grows sublinearly with the sequence length.

    \item 
\noindent \textbf{Low-dimensional structures.} Motivated by image generation, when the discrete data arise from the quantization of a continuous distribution with intrinsic dimension \(k\), the dual total correlation \(\cB(q_\mathrm{data})\) scales linearly with \(k\) rather than with the ambient dimension \(d\).
    
    \item 
    \noindent\textbf{Random graph models.} Such models define distributions over \(d = {n \choose 2}\) binary variables corresponding to the edges of a graph with \(n\) vertices. Besides Erd\H{o}s-R\'{e}nyi random graphs, which have independent edges and are therefore easy to sample, we consider both sparse random regular graphs and stochastic block models. In these cases, \(\cB(q_\mathrm{data})\) grows at most linearly (up to logarithmic factors) with \(n\), rather than quadratically. 
    
    \item 
    \noindent \textbf{Latent parity model.} Finally, we present an example in which both the total correlation \(\cC(q_{\mathrm{data}})\) and the dual total correlation \(\cB(q_{\mathrm{data}})\) are of order \(d\), while the effective total correlation \(\cD(q_{\mathrm{data}})\) remains of constant order. Such distribution is motivated by applications such as error-correcting codes and DNA sequences, where substantial noise may be present, yet the underlying signal is highly structured.
\end{itemize}
\vspace{1em}

%% file: discussion.tex
\section{Discussion}
In this work, we establish novel theoretical results for both uniform and masking discrete diffusions. For uniform diffusion models, we show that the \(\tau\)-leaping algorithm requires \(\widetilde{O}\left(d / \varepsilon \right)\) iterations to achieve $\varepsilon$ accuracy in KL divergence, improving on the prior bound \(\widetilde{O}\left(d^2 S/\varepsilon\right)\). We further establish the first algorithmic lower bound for the \(\tau\)-leaping sampler, which shows that our upper bound is unimprovable for a large class of distributions. 
For the masking discrete diffusion, we derive an upper bound that captures the intrinsic complexity of the data distribution and can scale logarithmically with the ambient dimension. Importantly, our results for both models only require a small score estimation error and, in contrast to prior work, do not rely on early stopping or the boundedness assumptions of the score estimator.

The improved bound for the masking noising process is achieved via a modification of the $\tau$-leaping algorithm. This modification falls within a structured subclass of $\tau$-leaping strategies that (i) allow for parallel coordinate updates, and thus sublinear rates, and (ii) preserve CTMC dynamics, which facilitates theoretical analysis. We hope that this perspective motivates the development of adaptive samplers for uniform discrete diffusion as well in the future.

  Several other open questions remain. Understanding which noising mechanisms --- masking, uniform, or others 
--- are best suited to different classes of target distributions is an important direction for future work. Moreover, the problem of learning accurate score functions in discrete diffusion models remains largely unexplored and warrants further investigation.

%% file: appendix.tex
\input{appendix/app_adaptivity}

\input{appendix/app_technical}

\input{appendix/app_main_proofs}

\input{appendix/app_main_lemmas}

\input{appendix/app_aux_lemmas}

%% file: appendix/app_adaptivity.tex
\section{Examples of low intrinsic dimensions}
\label{sec:examples}

\subsection{Details and formal results}
\label{sec:examples-full}

In this section, we revisit the examples outlined in Section~\ref{sec:adaptive-masking} and develop them in full detail. We formalize the statements in this section, and provide rigorous proofs in Appendix~\ref{sec:proofs-prop}.

\paragraph{Hidden Markov models.}
A hidden Markov model (HMM) consists of a latent Markov chain whose states are observed only indirectly through noisy measurements. Such models are widely used in natural language processing and pattern recognition \citep{mor2021systematic, mark2024application}.
In language modeling, for instance, the hidden states $z^i$ may encode the semantic topic or grammatical structure of the $i$-th token or word, while the observed variables $x^i$ represent the realized words or tokens. 

Formally, let $\{z^i\}_{i \in [d]}$ be a discrete-state Markov chain supported on $\cZ$, and let $\{x^i\}_{i \in [d]}$ be observations generated according to
$$x^i = f_i(z^i, \varepsilon^i),$$ where $\{\varepsilon^i\}_{i \in [d]}$ are i.i.d. noise variables independent of ${z^i}_{i \in [d]}$.
When $z^i$ represents the semantic topic of the $i$-th paragraph in a document, it is natural to assume that topic transitions happen only infrequently; that is, $z^i = z^{i-1}$ with high probability for $i>1$. 
Under this model, we establish the following proposition, whose proof is deferred to Section~\ref{sec:pf-hmm}.

\begin{proposition}
\label{prop:hmm_btc}
Consider the HMM described above. Suppose the transition probability of $\{z^i\}_{i \in [d]}$ satisfies $\Pr(z^i \neq z^{i-1}) \leq p$ for all $i \in \{2, \ldots, d\}$.
Assume that $1/d \lesssim p \ll 1$. Then
\begin{align}
   \mathcal B(q_{\mathrm{data}}) \leq  pd \log\left(\frac{|\mathcal{Z}|}{p}\right). 
\end{align}
\end{proposition}
To develop some intuition, consider generating a document with a constant number of paragraphs, where the transition probability scales $p = \Theta(1/d)$. Suppose further that the latent space $\mathcal{Z} \in [S]^k$ for some $k \ll d$ and $S$ denotes the vocabulary size. Then, the above bound yields $$\mathcal B(q_{\mathrm{data}}) \lesssim k \log(Sd),$$ which is substantially smaller than the ambient dimension $d \log(S)$. As such, with Theorem~\ref{thm:masking_main}, the sampling complexity scales with the intrinsic topic dimension $k$ rather than the document length $d$.

\paragraph{Low-dimensional Structures.} In image generation and other structured data settings, it is commonly assumed that the data lie on or near a low-dimensional manifold embedded in a high-dimensional ambient space, which often refers to as the manifold hypothesis \citep{gorban2018blessing, pope2021intrinsic}.
For example, natural images may be viewed as points on a manifold parameterized by a small number of underlying factors, such as lighting conditions, pose, and object identity.

In discrete settings, the notion of a manifold is not mathematically well defined. To capture low-dimensional structure, we instead model the data as arising from a continuous mapping from a latent representation into a high-dimensional observation space. 
For some latent continuous random variable $z$ supported on $\cZ \subset \mathbb{R}^k$, consider a decoding procedure $f: [0,1]^k \to \mathbb{R}^d$ as $$x^{\mathrm{con}} = f(z) + \varepsilon_{\mathrm{noise}},$$ 
for additive perturbations $\varepsilon_{\mathrm{noise}}.$
Thus, data lies close to a manifold $\{f(z): z \in \cZ\}$.
The final discrete observation is obtained via a quantization operator $\mathcal{Q}_S$, i.e., $x = \mathcal{Q}_{S}(x^{\mathrm{con}}) \sim q_{\mathrm{data}}$.

To align the model with standard image processing pipelines, we work with the uniform lattice quantization function $\mathcal{Q}_S: \mathbb{R}^d \to [S]^d$ defined coordinate-wise as $[\mathcal{Q}_S(x)]^i = \mathrm{clip}(\lfloor x^i \rfloor, 0, S)$ for $i \in [d]$, where $\mathrm{clip}(x, a, b)\defn \min\{\max\{x, a\}, b\}$ is the clip function and $\lfloor \cdot \rfloor$ is the floor function.
To ensure regularity of both the manifold and the induced data distribution, we focus on the case where $\cZ$ is a compact set and $f$ is a Lipschitz function. 
The noise $\varepsilon_\mathrm{noise}$ is taken to be Gaussian for simplicity of analysis; the arguments extend readily to more general smooth noise distributions. 

\begin{proposition}\label{prop:manifold}
    Let $\cZ \subset \mathbb{R}^k$ be compact with diameter $\mathsf D$, and let 
$f : \cZ \to \mathbb{R}^d$ be $\mathsf L$-Lipschitz.
Assume the noise satisfies  $\varepsilon_{\mathrm{noise}} \sim \mathcal{N}(0, \sigma^2 I_d)$ independenly generated for each obersvation.
   Then the resulting distribution satisfies 
    \begin{align}
        \mathcal{B}(q_{\mathrm{data}}) \leq k \log\left(2 + \frac{\mathsf{2D}\mathsf{L}}{\sigma}\right).
    \end{align}
\end{proposition}
In image generation, the ``ideal image'' $x^{\mathrm{con}}$ may be interpreted as the vector of continuous pixel intensities prior to quantization, while the observed image $x$ is obtained by applying pixel-wise quantization to $x^{\mathrm{con}}$.
When $k \ll d$, the above bound yields $$\mathcal{B}(q_{\mathrm{data}}) = \widetilde{O}(k) = o(d),$$ and hence we can efficiently sample such images despite the high dimensionality of the observation space.

\paragraph{Random graph models.}
Discrete diffusion models have also found applications in scientific domains such as molecular generation and protein design, where data are naturally represented as random graphs with fixed vertex sets and random edges \citep{ingraham2019generative,xu2022geodiff}.
To make this concrete, we consider two widely studied random graph models on $n$ vertices, which can be viewed as a discrete distribution over adjacency matrices of dimension $n^2.$

\begin{itemize}
    \item \textbf{Regular graphs}: A $k$-regular graph is a graph in which each vertex has degree exactly \(k\).
    Suppose we want to sample a random graph $\mathcal{G}$ from some distribution supported on the set of $k$-regular graphs with $n$ vertices. 
    
    \begin{proposition}\label{prop:regular_graph}
        For sparse regular graph model, i.e., $k \leq n/\log (n)$, we have
    \begin{align}
        \mathcal{B}(\mathcal{G}) \lesssim kn \log\left(\frac{n}{k}\right) = o(n^2).
    \end{align}
    \end{proposition}
    
    \item \textbf{Stochastic block models}: A stochastic block model (SBM) is a generative model for random graphs that captures community structure within networks. 
    In an SBM, $n$ vertices of the graph are partitioned into $r$ distinct communities or blocks, represented by latent variables $\{z^i\}_{i \in [n]}$ taking values in $[r]$.
    Conditioned on the latent labels, edges are generated independently. For two vertices $i, j \in [n]$, an edge is created with probability
    $$p \bI\{z^i = z^j\} + q\bI\{z^i \neq z^j\},$$ where $p, q \in [0, 1]$ govern the within- and between-community connection probabilities, respectively.
    \begin{proposition}\label{prop:sbm}
    Let $\cG$ be a random graph drawn from the above $r$-block SBM. 
    Then 
    \begin{align*}
        \mathcal{B}(\mathcal{G}) \leq n \log(r) = o(n^2).
    \end{align*}
    \end{proposition}
\end{itemize}
For both random graph models, as the number of vertices $n$ grows large, the complexity satisfies $\mathcal{B}(\mathcal{G}) = o(n^2)$, which is strictly smaller compared to the ambient dimension $n^2.$
This indicates that diffusion-based methods can sample efficiently from such graph distributions.

In fact, the analyses of Propositions~\ref{prop:manifold} and \ref{prop:sbm} extend naturally to generalized random geometric graphs. Consider the following example.
Let each vertex $i \in [n]$ be associated with latent variable $z^i\in \cZ$. For distinct vertices $i$ and $j$, an edge is placed independently with probability 
\begin{align*}
    \beta \exp\left(- \frac{d(z^i, z^j)}{r_0}\right),
\end{align*}
where $\beta \in [0, 1],$ $r_0 > 0$ and $d(\cdot, \cdot)$ is an appropriate metric in the latent space $\cZ$. 
\begin{itemize}
    \item When latent variables $\{z^i\}$ are discrete with $o(n)$ entropy, as is the case, for example, when it takes value in a fixed-dimensional latent space, the dual total correlation of the resulting random graph is $o(n^2)$.
    
    \item For continuous latent variables, suppose $\cZ = {\cS}^{d_z - 1}$, the unit sphere in $\mathbb{R}^{d_z}$.
    Under some regularity conditions, 
    the dual total correlation scales with $d_z \cdot n$, followed by an analogous covering number argument in Proposition~\ref{prop:manifold}. In particular, whenever $d_z = o(n)$, the complexity is again subquadratic, leading to sublinear (in $n^2$) convergence rates for diffusion-based sampling.
\end{itemize}

\paragraph{Latent parity model.}
A prototypical example of a distribution with small dual total correlation $\cB(q_0)$ and large total correlation $\cC(q_0)$ is the mixture of two Dirac measures: $$p_m \defn \frac{1}{2} \delta_{\mathbf{0}} + \frac{1}{2} \delta_{\mathbf{1}},$$ where \(\mathbf{0}\) and \(\mathbf{1}\) are vectors of all-zeros and all-ones, respectively. It can be easily computed that \(\cB(p_m) = \log (2)\), whereas, \(\cC(p_m) = (d - 1) \log (2)\).

The opposite happens, for instance, for the following \(\mathrm{XOR}\) distribution \(p_{\mathrm{XOR}}\): $$x^1, \ldots, x^{d-1} \overset{\mathrm{i.i.d.}}{\sim} \mathrm{Bern}(1/2) ~~\text{ and }~~ x^d = \sum_{i=1}^{d - 1} x^i \bmod 2.$$ In this case, \(\cB(p_{\mathrm{XOR}}) = (d - 1) \log (2)\), and \(\cC(p_{\mathrm{XOR}}) = \log (2)\).

Real-world data distributions can combine features of both extremes: a strong low-dimensional signal corrupted by weakly correlated noise. In such cases, both \(\cB(q_{\mathrm{data}})\) and \(\cC(q_{\mathrm{data}})\) can be large, while \(\cD(q_{\mathrm{data}})\) remains small. 
To illustrate this phenomenon, consider the following entrywise mixture of the two preceding examples. 
\begin{enumerate}
    \item Fix a bi-partition $[d] = I_0 \sqcup I_1$ for non-empty index sets $I_0$ and $I_1$;
    \item For all indices \(i \in I_0\), set \(x_i = b\) for $b \sim \mathrm{Bern}(1/2)$;
    \item Among all indices \(i \in I_1\), sample all but one \(x_i \sim \mathrm{Bern}(1/2)\) independently;
    \item For the last index \(i^\ast\), set \(x_{i^\ast} = \left(b + \sum_{i \neq i^\ast} x_i\right) \bmod 2\).
\end{enumerate}
Denote this distribution as $p_{\mathrm{ex}}$, and let $x = (x^1, \ldots, x^d) \sim p_{\mathrm{ex}}$.
\begin{proposition}\label{prop:structure-with-noise}
    Suppose that $\min\{|I_0|, |I_1|\}/d = \Theta(1)$. 
    Distribution $p_{\mathrm{ex}}$ satisfies 
    \begin{align}
        \cB(p_{\mathrm{ex}}) = \Theta(d),\quad  \cC(p_{\mathrm{ex}}) = \Theta(d) \quad\text{and}\quad \cD(p_{\mathrm{ex}}) = O(1).
    \end{align}
\end{proposition}
By Proposition~\ref{prop:structure-with-noise}, $p_{\mathrm{ex}}$, which can be viewed as a non-trivial mixing of $p_m$ and $p_{\mathrm{XOR}}$, satisfies $$\cD(p_{\mathrm{ex}}) \ll \min\{\cB(p_{\mathrm{ex}}), \cC(p_{\mathrm{ex}})\}.$$ 
This example highlights the fundamental role of the effective total correlation in characterizing sampling efficiency. 

\subsection{Proofs of results in Section~\ref{sec:examples-full}}
\label{sec:proofs-prop}

Variants of the following lemma will be used repeatedly throughout this section. We state it here for convenience and to streamline the proofs.

\begin{lemma}\label{lem:b-reduce-to-latent}
    Consider any $d$-dimensional discrete random variable $X$ and any random variable $W$ 
    such that $X^{i} \indep X^{-i} \mid W$ for any $i \in [d]$, where $X = (X^1, \ldots, X^d)$ and $X^{-i}$ is the $(d-1)$-dimensional marginal of $X$ with $i$-th coordinate excluded. Then, 
    \begin{align*}
        \cB(X) \leq \mi(X; W).
    \end{align*}
    If $W$ is discrete, we additionally have $\cB(X) \leq \cH(W)$.
\end{lemma}

\noindent
\begin{proof}[\textbf{Proof of Lemma~\ref{lem:b-reduce-to-latent}}]
    We first notice that for any random variable $W$ such that $X^{i} \indep X^{-i} \mid W$ for any $i \in [d]$, we have
\begin{align*}
    \mathcal{H}(X^i \mid X^{-i}) \geq \mathcal{H}(X^i \mid X^{-i}, W) = \mathcal{H}(X^i \mid W),
\end{align*}
where the first inequality follows from the definition of the entropy.
Recalling the definition of $\cB(\cdot)$, we obtain
\begin{align*}
    \mathcal{B}(X) = \cH(X) - \sum_{i=1}^d \cH(X^i | X^{-i}) 
    \leq \mathcal{H}(X) - \sum_{i=1}^d \mathcal{H}(X^i \mid W).
\end{align*}
Using the conditional independence condition again, we have
\begin{align*}
    \cH(X|W) = \cH\big((X_1, \ldots, X_d) | W\big) = \sum_{i=1}^d \mathcal{H}(X^i \mid W),
\end{align*}
which implies
\begin{align*}
    \cB(X) \le \cH(X) - \cH(X|W) = I(X; W) \overset{\mathrm{(a)}}{=} \cH(W) - \cH(W|X) \overset{\mathrm{(b)}}{\leq} \cH(W),
\end{align*}
where (a) and (b) apply when $W$ is a discrete random variable.
\end{proof}
 
\subsubsection{Proof of Proposition~\ref{prop:hmm_btc}}
\label{sec:pf-hmm}
    The hidden Markov structure of $\{(x^i, z^i)\}_{i \in [d]}$ satisfies $x^i \indep x^j \mid (z^i, z^j)$, since $\varepsilon^i \indep \varepsilon^j \mid (z^i, z^j)$.
    Considering Lemma~\ref{lem:b-reduce-to-latent} above, we can upper bound $\cB(q_{\mathrm{data}})$ by $\cH(z)$, which is the entropy of the latent Markov chain. By the additivity of the entropy, we have
    \begin{align*}
        \cB(q_{\mathrm{data}})\leq \cH(z) = \cH(z^1) + \sum_{i=2}^{d}\cH\big(z^{i} | \{z^{j}\}_{j \in [i-1]}\big) = \cH(z^1) + \sum_{i=2}^d\cH(z^{i} | z^{i-1}).
    \end{align*}
    When $\{z^i\}_{i \in [d]}$ is supported on a single point, we have $|\cZ| = 1$ and $\cH(z) = 0$.
    When the state space $\cZ$ satisfies $2 \le |\cZ| < \infty$, the maximum entropy distribution is achieved when 
    \begin{align*}
        z^1 \sim \mathrm{Unif}(\cZ) \quad \text{and} \quad  z^i | z^{i-1} \sim (1-p)\delta_{z^{i-1}} + p \mathrm{Unif}\big(\cZ\,\backslash\, \{z^{i-1}\}\big).
    \end{align*} 
    We obtain
    \begin{align*}
        \cH(z) &\leq \log(|\cZ|) + \sum_{i=2}^d \left[-(1-p)\log(1-p) + 
        -(|\cZ| - 1) \cdot \frac{p}{|\cZ| - 1}\log\left(\frac{p}{|\cZ| - 1}\right)\right]\\
        &\overset{\mathrm{(a)}}{\leq} \log(|\cZ|) + (d-1)\cdot \left(2p + p\log\left(\frac{|\cZ|}{p}\right)\right)\\
        &\overset{\mathrm{(b)}}{\leq} pd \log\left(\frac{|\cZ|}{p}\right),
    \end{align*}
    where in (a), we use $-\log(1-p) \leq 2p$, since $p \ll 1$; in (b), we use the condition $p \gtrsim 1/d$ and $|\cZ|/p \geq 2/p \gg 1$. This completes the proof of the desired result.
\qed

\subsubsection{Proof of Proposition~\ref{prop:manifold}}
    Write $\varepsilon_{\mathrm{noise}} = (\varepsilon_{\mathrm{noise}}^1, \ldots, \varepsilon_{\mathrm{noise}}^d)$. Since $\varepsilon_{\mathrm{noise}} \sim \mathcal{N}(0, \sigma^2 I_d)$, we have $\varepsilon_{\mathrm{noise}}^i \indep \varepsilon_{\mathrm{noise}}^{-i}$ for any $i \in [d]$. Processing through the decoder $f$, $[x^{\mathrm{con}}]^i = [f(z)]^i + \varepsilon_{\mathrm{noise}}^i$ for any $i \in [d]$, which leads to 
    \begin{align}\label{eq:manifold-continuous-cindep}
        [x^{\mathrm{con}}]^i \indep [x^{\mathrm{con}}]^{(-i)} \mid z,
    \end{align}
    where $[x^{\mathrm{con}}]^{(-i)}$ is the $(d-1)$-dimensional marginal of $x^{\mathrm{con}}$ with $i$-th coordinate excluded. 
    Note that $\cQ_{S}$ is a entry-wise quantization, i.e., we can write $\cQ_S(x) = (\widetilde{\cQ}_S(x^1), \ldots, \widetilde{\cQ}_S(x^d))$ for entry-wise deterministic quantization function $\widetilde{\cQ}_S: \mathbb{R} \to [S]$, and $x^i = \widetilde{\cQ}_S([x^{\mathrm{con}}]^i)$ by the generation process.
    Eqn.~\eqref{eq:manifold-continuous-cindep} therefore implies that for any $i \in [d]$,
    \begin{align*}
        x^{i} \indep x^{-i} \mid z.
    \end{align*}
     Applying Lemma~\ref{lem:b-reduce-to-latent}, we obtain
    \begin{align}\label{eq:manifold-b-i-mid}
        \cB(q_{\mathrm{data}}) = \cB(x) \leq \mi(x; z) \leq \mi(x^{\mathrm{con}}; z),
    \end{align}
    where the last inequality follows from the data processing inequality of the mutual information.

    In the following proof, we proceed to control  $\mi(x^{\mathrm{con}}; z)$. Since $\varepsilon_{\mathrm{noise}}$ is independent noise, using data-processing inequality, we reach
    \begin{align}\label{eq:manifold-i-mid}
        \mi(x^{\mathrm{con}}; z) \leq \mi\big(f(z) + \varepsilon_{\mathrm{noise}}; f(z)\big) = \mi\big(f(z); f(z) + \varepsilon_{\mathrm{noise}}\big).
    \end{align}
    Without loss of generality, we assume $\cZ \subseteq [0, \mathsf{D}]^k$.
    Partition $[0, \mathsf{D}]^k$ into hypercubes of size $h_J = \sigma/\mathsf{L}$, and write this partition as $\{C_1, \ldots, C_{\lceil \mathsf{D}/h_J \rceil^k}\}$ such that$$[0, \mathsf{D}]^k \subseteq \bigsqcup_{i=1}^{\lceil {\mathsf{D}}/{h_J} \rceil^k} C_i.$$
    Define $J = J(z)$ to be the hypercube index $i(z)$ such that $z \in C_{i(z)}$, and $\cF_J$ to be $\sigma$-algebra generated by $J(z)$. 
    By the chain rule and data processing inequality for mutual information, we have
    \begin{align}
        \mi\big(f(z); f(z) + \varepsilon_{\mathrm{noise}}\big) &\leq \mi\big(J(z), f(z); f(z) + \varepsilon_{\mathrm{noise}}\big)\notag\\
        &= \mi\big(J(z); f(z) + \varepsilon_{\mathrm{noise}}\big) + \mi\big(f(z); f(z) + \varepsilon_{\mathrm{noise}} \mid J\big)\notag\\
        &\leq k \log \left(1 + \frac{\mathsf{D}}{h_J}\right) + \mi\big(f(z); f(z) + \varepsilon_{\mathrm{noise}} \mid J\big),\label{eq:manifold-i-decomp}
    \end{align}
    where in the last line, we use $\mi(J(z); f(z) + \varepsilon_{\mathrm{noise}}) \leq \cH(J(z)) \leq \log(|\mathrm{supp}(J(z))|)$. To upper bound the second term above, we introduce the following lemma on Gaussian channel, whose proof is given in Section~\ref{sec:pf-i-variance}.
    \begin{lemma}\label{lem:I-variance}
        For any random variable $W \in \mathbb{R}^d$ and independent noise $\varepsilon_{\mathrm{noise}} \sim \cN(0, \sigma^2 I_d)$, we have
        \begin{align*}
            \mi(W; W + \varepsilon_{\mathrm{noise}}) \leq \frac{\tr\big(\mathrm{Var}[W]\big)}{2\sigma^2},
        \end{align*}
        where $\tr(\cdot)$ is the trace function.
    \end{lemma}
    In Lemma~\ref{lem:I-variance}, taking $W \overset{d}{=}  f(z) \mid \cF_J$, we arrive at
    \begin{align}\label{eq:manifold-i-var}
        \mi\big(f(z); f(z) + \varepsilon_{\mathrm{noise}} \mid J\big) \leq \frac{\tr\Big(\mathrm{Var}\big[f(z) \mid \cF_J\big]\Big)}{2 \sigma^2}.
    \end{align}
    To further control the right hand side, 
    direct calculations show
    \begin{align}
        \tr\Big(\mathrm{Var}\big[f(z) \mid \cF_J\big]\Big) &= \sum_{i=1}^d \mathrm{Var}\big[[f(z)]^i \,\big|\, \cF_J\big]
        = \E\bigg[\Big\|f(z) - \E\big[f(z) \mid \cF_J\big]\Big\|_2^2 \,\Big|\, \cF_J\bigg].
    \end{align}
It is therefore sufficient to consider the quantity $\|f(z) - \E[f(z) \mid \cF_J]\|_2^2.$ We make the observation that
    \begin{align*}
        \Big\|f(z) - \E\big[f(z) \mid \cF_J\big]\Big\|_2 
        &\overset{\mathrm{(a)}}{\leq} \sup_{w \in \mathrm{Conv}(f(C_{J(z)}))}\big\|f(z) - w\big\|_2 \\
        &\overset{\mathrm{(b)}}{=} \sup_{w \in f(C_{J(z)})}\big\|f(z) - w\big\|_2\\ 
        &\overset{\mathrm{(c)}}{\leq} \|f\|_{\mathrm{Lip}} \cdot \sup_{z' \in C_{J(z)}}\big\|z - z'\big\|_2 \overset{\mathrm{(d)}}{\leq} \mathsf{L} \sqrt{k} h_J,
    \end{align*}
    where $\|\cdot\|_2$ denotes Euclidean norm in $\mathbb{R}^d$, and $\mathrm{Conv}(\cdot)$ denotes the convex null of a given set.
    In (a), we use the fact that $\E[f(z) \mid \cF_J] \in \mathrm{Conv}(f(C_{J(z)}))$; in (b), we adopt $f$ is continuous and hence $f(C_{J(z)}))$ is bounded, and the property of the convex hull that 
    \begin{align*}
        \mathrm{diam}\big(\mathrm{Conv}(A)\big) = \mathrm{diam}(A) \quad \text{for any bounded subset }A \subseteq \mathbb{R}^d; 
    \end{align*} in (c), we recall the Lipschitz condition on $f$; in (d), we notice that $\mathrm{diam}(C_i) \leq \sqrt{k} h_J$ for any hypercube $C_i$.
    Putting pieces together gives 
    \begin{align}
        \tr\Big(\mathrm{Var}\big[f(z) \mid \cF_J\big]\Big) &\leq \big(\mathsf{L} \sqrt{k} h_J\big)^2 = k \sigma^2.\label{eq:manifold-var}
    \end{align}
    
    Finally, plugging Eqns.~\eqref{eq:manifold-i-var} and \eqref{eq:manifold-var} into \Cref{eq:manifold-i-decomp}, we obtain
    \begin{align*}
        \mi\big(f(z); f(z) + \varepsilon_{\mathrm{noise}}\big) \leq k \log\left(1 + \frac{\mathsf{D}\mathsf{L}}{\sigma} \right) + \frac{k}{2} \leq k \log\left(2 + \frac{\mathsf{2D}\mathsf{L}}{\sigma} \right).
    \end{align*}
    Combining the above inequality with Eqns.~\eqref{eq:manifold-b-i-mid} and \eqref{eq:manifold-i-mid}, we conclude
    \begin{align*}
        \cB(q_{\mathrm{data}}) \leq \mi(x^{\mathrm{con}}; z) = \mi\big(f(z); f(z) + \varepsilon_{\mathrm{noise}}\big) \leq k \log\left(2 + \frac{\mathsf{2D}\mathsf{L}}{\sigma} \right).
    \end{align*}
\qed

\subsubsection{Proof of Proposition~\ref{prop:regular_graph}}

Define the set of all $k$-regular graphs with $n$ vertices as $G_{n, k}$. Without loss of generality, we assume that \(nk\) is even, as otherwise \(G_{n,k}\) is empty.
By a corollary of \citet[Theorem 1.4]{liebenau2017asymptotic}, we have the following asymptotic result:
\begin{align*}
    |G_{n, k}| = \Theta\left({n-1 \choose k}^n {\frac{n(n-1)}{2} \choose m}{n(n-1) \choose 2m}^{-1}\right).
\end{align*}
where $m = kn/2$. By Stirling's formula of the form
\begin{align*}
    \log(a!) = a\log(a) - a + O(\log(a)),
\end{align*} 
we can compute that
\begin{align*}
    \log(|G_{n, k}|) &\lesssim n\log {n-1 \choose k} + \log {\frac{n(n-1)}{2} \choose m} - \log {n(n-1) \choose 2m}\\
    &= \frac{kn}{2} \log \left(\frac{n - 1 - k}{k}\right) + \frac{n(n-1)}{2} \log\left(\frac{n - 1}{n - 1-k}\right)\\
    &\leq \frac{kn}{2} \log\left(\frac{n}{k}\right) + \frac{n^2}{2}\log\left(1 + \frac{k}{n-1-k}\right)\\
    &\leq \frac{kn}{2} \log\left(\frac{n}{k}\right) + \frac{kn^2}{2(n - 1 - k)} \lesssim kn \log\left(\frac{n}{k}\right),
\end{align*}
where in the last line, we invoke the condition that $k \leq n/\log(n) \ll n - 1 - k$.
Recalling the definition of $\cB(\cdot)$, we can conclude 
\begin{align*}
    \mathcal{B}(\mathcal{G}) \leq \mathcal{H}(\mathcal{G}) \leq \log(|G_{n, k}|) \lesssim kn \log\left(\frac{n}{k}\right) = o(n^2).
\end{align*} 
\qed

\subsubsection{Proof of Proposition~\ref{prop:sbm}}
By definition of $r$-block SBM, the latent variable vector $(z^1, \ldots, z^n)$ is supported on $[r]^n$, which satisfies
\begin{align*}
    \cH\big((z^1, \ldots, z^n)\big) \leq \log\big(|[r]^n|\big) = n\log(r).
\end{align*}
Given the latent variable $(z^1, \ldots, z^n)$, the block structure is fixed and hence each edge is sampled independently from a Bernoulli distribution. Therefore, we have 
\begin{align*}
    e^{ij} \indep e^{k \ell} \mid \{z^i\}_{i \in [n]}
\end{align*}
for any $i, j, k, l \in [n]$, where $e^{ij}$ and $e^{k\ell}$ are the indicator variables of the existence of edges between vertices $i,j$ and between vertices $k, \ell$.
By Lemma~\ref{lem:b-reduce-to-latent}, we conclude 
\begin{align*}
    \mathcal{B}(\mathcal{G}) \leq \mathcal{H}((z^1, \ldots, z^n)) \leq n\log(r) \leq n \log(n) = o(n^2),
\end{align*}
where we use the convention that the number of blocks satisfies $r \leq n$.
\qed 

\begin{remark}
    The setting of Proposition~\ref{prop:sbm} can be viewed as a special case of the generalized random geometric graph model, in which the latent variable corresponds to the block index.
    More generally, the same conclusion holds under analogous assumptions, with essentially the same proof strategy.
\end{remark}

\subsubsection{Proof of Proposition~\ref{prop:structure-with-noise}}
    Let \( r := |I_0|/d \) be the proportion of coordinates in \( I_0 \). Throughout, we assume \( \min\{r, 1 - r\} = \Theta(1) \).
    
    \paragraph{Step 1: Establish $\cB(p_{\mathrm{ex}}) = \Theta(d)$ and $\cC(p_{\mathrm{ex}}) = \Theta(d)$.} For a random variable $x \sim p_{\mathrm{ex}}$, we shall demonstrate that
    \begin{align}\label{eq:all-3-h}
        \sum_{i=1}^d \cH(x^i) = d\log(2), \quad \log(2)(|I_1| - 1) \leq \cH(x) \leq \log(2)|I_1| \quad \text{and}\quad \sum_{i=1}^d \cH(x^i \mid x^{-i}) = 0.
    \end{align}
    Towards this goal, we make the observation that for any $i \in I_0$ or $x \in I_1 \backslash i^*$, $x^i \sim \mathrm{Bern}(1/2)$ and hence $\cH(x^i) = \log(2)$. For $i = i^*$, we assert that $x^{i^*} \sim \mathrm{Bern}(1/2)$. In fact, we have 
    \begin{align*}
        \P\left(\sum_{i \in I_1 \backslash i^*} x^i \equiv 0 \bmod 2\right) = \P\left(\mathrm{Bin}\left(|I_1| - 1, \frac{1}{2}\right)  \equiv 0 \bmod 2\right) = \frac{1}{2},
    \end{align*}
    where in the last equality, we invoke the following lemma.
    \begin{lemma}\label{lem:bin-even-odd}
        For any $n \in \mathbb{N}^+$ and $X \sim \mathrm{Bin}(n, 1/2)$, we have
        \begin{align*}
            \P(X \equiv 0 \bmod 2) = \P(X \equiv 1 \bmod 2) = \frac{1}{2}.
        \end{align*}
    \end{lemma}
 \noindent   As result, the distribution of $x^{i^*}$ satisfies 
    \begin{align*}
        \P(x^{i^*} = 0) = \P(b = 0)\cdot\P\left(\sum_{i \in I_1 \backslash i^*} x^i \equiv 0 \bmod 2\right) + \P(b = 1)\cdot\P\left(\sum_{i \in I_1 \backslash i^*} x^i \equiv 1 \bmod 2\right) = \frac{1}{2},
    \end{align*}
    which reveals that $x^{i^*} \sim \mathrm{Bern}(1/2)$ and hence $\cH(x^{i^*}) = \log(2)$.
    In conclusion, we obtain
    \begin{align}\label{eq:sum-xi-h}
        \sum_{i = 1}^d \cH(x^i) = \sum_{i \in [d]\backslash i^*} \cH(x^i) + \cH(i^*) = d\log(2) .
    \end{align}
    To upper bound $\cH(x)$, invoke the simple property for entropy function to get   
    \begin{align}\label{eq:h-upper}
        \cH(x) \leq \log(|\mathrm{supp}(x)|) \leq \log\left(2 \cdot 2^{|I_1| - 1}\right) = \log(2) |I_1|.
    \end{align}
    The lower bound can be obtained through
    \begin{align}\label{eq:h-lower}
        \cH(x) \geq \cH(\{x^{i}\}_{i \in I_1\backslash i^*}) = \log\left(2^{|I_1|-1}\right) = \log(2)(|I_1| - 1).
    \end{align}
    For any $i \in [d]$, when $x^{-i}$ is given, we can recover $x^i$ by first observing the value of $b$ from $x^j$ for any $j \in I_0$, then applying the formula
    $$x^i = b + \sum_{k \in I_1 \backslash i} x^k \bI\{i \in I_1\} \bmod 2.$$
    Thus, $x^i \mid x^{-i}$ is always a Dirac measure, which leads to
    \begin{align}\label{eq:sum-xi-x-i-h}
        \sum_{i=1}^d \cH(x^i \mid x^{-i}) = 0.
    \end{align}
    Combining Eqns.~\eqref{eq:sum-xi-h}, \eqref{eq:h-upper}, \eqref{eq:h-lower} and \eqref{eq:sum-xi-x-i-h} proves \Cref{eq:all-3-h}.
    
    Equipped with \Cref{eq:all-3-h}, we are ready to bound  $\cB(p_{\mathrm{ex}})$ and $\cC(p_{\mathrm{ex}})$. 
    It can easily seen that 
    \begin{align*}
        \cB(p_{\mathrm{ex}}) &= \cH(x) - \sum_{i=1}^d \cH(x^i \mid x^{-i}) \geq \log(2)\big((1-r)d - 1\big) =  \Omega(d),\\
        \cC(p_{\mathrm{ex}}) &= \sum_{i=1}^d \cH(x^i) - \cH(x) \geq \log(2)(d - |I_1|) = \log(2)rd = \Omega(d).
    \end{align*}
    For the reverse direction, we can prove the matching lower bound similarly, which leads to 
    \begin{align*}
        \cB(p_{\mathrm{ex}}) = \Theta(d), \quad \cC(p_{\mathrm{ex}}) = \Theta(d). 
    \end{align*}

\medskip
    \paragraph{Step 2: Show $\cD(p_{\mathrm{ex}}) = O(1)$.} Recall the definition of $\cD(\cdot)$ in Eqn.~\eqref{def:etc}: 
        \begin{align*}
        \cD(p_{\mathrm{ex}}) \defn \int_0^{\infty} \min(1, t) \cI(t) \d t
        \quad \text{with} \quad 
        \cI(t) \defn \sum_{i \neq j \in [d]} \mi(x_t^i; x_t^j \mid x_t^{-(i, j)}) \geq 0.
    \end{align*}
To upper bound $\cD(p_{\mathrm{ex}})$, let us write 
    \begin{align*}
        \cD(p_{\mathrm{ex}})
        &= \int_0^{\frac{1}{d}} t\, \cI(t) \d t + \int_{1/d}^{\log(d)} \min\{1, t\}\,\cI(t) \d t + \int_{\log(d)}^{\infty} \cI(t) \d t.
    \end{align*}
By direct calculations, one has 
    \begin{align*}
        \int_0^{\frac{1}{d}} t\, \cI(t) \d t &\leq \frac{1}{d} \int_0^{\frac{1}{d}} \cI(t) \d t \leq \frac{\cB(p_{\mathrm{ex}})}{d} = \Theta(1),\\
        \int_{\log(d)}^{\infty} \cI(t) \d t &\leq \frac{1}{d - 1} \int_{\log(d)}^{\infty} (e^{t} - 1) \cI(t) \d t \leq \frac{\cC(p_{\mathrm{ex}})}{d-1} = \Theta(1).
    \end{align*}
Therefore, it obeys 
       \begin{align*}
        \cD(p_{\mathrm{ex}})
        &= \int_{1/d}^{\log(d)} \min\{1, t\}\,\cI(t) \d t + O(1).
    \end{align*} 
  To prove $\cD(p_{\mathrm{ex}}) = O(1)$, it suffices to show that 
    \begin{align}\label{eq:int-min-it}
        \int_{1/d}^{\log(d)} \min\{1, t\}\,\cI(t) \d t  = O(1).
    \end{align}
   In view of the definition of $\cI(t)$, we can decompose it as 
    \begin{align*}
        \cI(t) &= \left(\sum_{i, j \in I_0, i\neq j} + \sum_{i, j \in I_1, i\neq j} + \sum_{i \in I_0, j \in I_1} + \sum_{i \in I_1, j \in I_0} \right)\mi(x^i_t; x^j_t \mid x_t^{-(i,j)})\\
        &\defn \cI_1(t) + \cI_2(t) + \cI_3(t) + \cI_4(t),
    \end{align*}
    and we shall bound these four terms separately.

    Before diving into the proofs, we make the observation that the mutual information can be computed via 
    \begin{align}\label{eq:mi-i-j-formula}
            \mi(x_t^i; x_t^j \mid x_t^{-(i,j)}) = \cH(x_t^i \mid x_t^{-(i,j)}) - \cH(x_t^i \mid x_t^{-i}). 
    \end{align}
    To further compute each entropy terms, 
    let us introduce two quantities below 
    \begin{subequations}
    \label{eqn:ravel}
    \begin{align}
        \cH_t^{1} &= \cH(e^{-t} \delta_0 + (1-e^{-t}) \delta_{\mask}) = \cH(e^{-t} \delta_1 + (1-e^{-t}) \delta_{\mask}) = te^{-t} - \log(1-e^{-t}) (1- e^{-t}), \\
        \cH_t^{2} &= \cH\left(\frac{1}{2}e^{-t} \delta_0 + \frac{1}{2}e^{-t} \delta_1 + (1-e^{-t}) \delta_{\mask}\right) = (t + \log(2))e^{-t} - \log(1-e^{-t}) (1- e^{-t}).
    \end{align}
    \end{subequations}
    We shall relate our quantities of interest to these terms below. 
        \paragraph{Case 1: $i, j \in I_0$, $i \neq j$.} 
        For any given $x_t^{-(i,j)}$, it always holds true that $$\P(x_t^i = \mask) = 1-e^{-t},$$ since the noising process is time-homogeneous and independent between coordinates. Recall the definition $m(x) = \{i \in [d]:x^i = \mask\}$. Define the event $\cE_{t,1}^{i, j} \in \cF_t^{-(i,j)}$, where $\cF_t^{-(i,j)}$ is the $\sigma$-algebra generated by $x_t^{-(i,j)}$, as follows:
        \begin{align*}
            \cE_{t,1}^{i, j} \defn \Bigg\{x_t^{-(i,j)} :\Bigg(\bigvee_{k \in I_0 \backslash \{i, j\}}\{k \notin m(x_t)\}\Bigg) \bigvee \left(\bigwedge_{\ell \in I_1} \{\ell \in m(x_t)\}\right) = 1\Bigg\},
        \end{align*}
        where $\wedge$ is the logical operator AND, and $\vee$ is the logical operator OR.
        By construction of $p_{\mathrm{ex}}$, it can be checked that 
        \begin{align*}
            \left(x_t^i \,\big|\, x_t^{-(i,j)} \in \cE_{t, 1}^{i, j}\right) &\sim e^{-t} \delta_{0/1} + (1-e^{-t}) \delta_{\mask};\\
            \left(x_t^i \,\big|\, x_t^{-(i,j)} \in (\cE_{t, 1}^{i, j})^c\right) &\sim \frac{1}{2}e^{-t} \delta_0 + \frac{1}{2}e^{-t} \delta_1 + (1-e^{-t}) \delta_{\mask},
        \end{align*}
        where $\delta_{0/1}$ represents either $\delta_0$ or $\delta_1$. Therefore, by the definition of the conditional entropy, we have
        \begin{align}\label{eq:i1-conditional-entropy-1}
            \cH(x_t^i \mid x_t^{-(i,j)}) = \cH_t^1 \cdot \P(\cE_{t,1}^{i, j}) + \cH_t^2 \cdot \big(1 - \P(\cE_{t, 1}^{i, j})\big).
        \end{align}
        Define the event $\cE_{t,1}^{i} \in \cF_t^{-i}$, where $\cF_t^{-i}$ is the $\sigma$-algebra generated by $x_t^{-i}$, as follows:
        \begin{align*}
            \cE_{t,1}^{i} \defn \Bigg\{x_t^{-i}: \Bigg(\bigvee_{k \in I_0 \backslash \{i\}}\{k \notin m(x_t)\}\Bigg) \bigvee \left(\bigwedge_{\ell \in I_1} \{\ell \in m(x_t)\}\right) = 1\Bigg\}.
        \end{align*}
        Then, it can be checked similarly that
        \begin{align*}
            \left(x_t^i \,\big|\, x_t^{-i} \in \cE_{t, 1}^{i}\right) &\sim e^{-t} \delta_{0/1} + (1-e^{-t}) \delta_{\mask};\\
            \left(x_t^i \,\big|\, x_t^{-i} \in (\cE_{t, 1}^{i})^c\right) &\sim \frac{1}{2}e^{-t} \delta_0 + \frac{1}{2}e^{-t} \delta_1 + (1-e^{-t}) \delta_{\mask},
        \end{align*}
        which leads to 
        \begin{align}\label{eq:i1-conditional-entropy-2}
             \cH(x_t^i \mid x_t^{-i}) = \cH_t^1 \cdot \P(\cE_{t, 1}^{i}) + \cH_{t}^2 \cdot \big(1 - \P(\cE_{t, 1}^{i})\big).
        \end{align}
        Plugging Eqns.~\eqref{eq:i1-conditional-entropy-1} and \eqref{eq:i1-conditional-entropy-2} into \Cref{eq:mi-i-j-formula} gives that for any $i, j \in I_0$, $i \neq j$,
        \begin{align*}
            \mi(x_t^i; x_t^j \mid x_t^{-(i,j)}) &= \cH(x_t^i \mid x_t^{-(i,j)}) - \cH(x_t^i \mid x_t^{-i})\\
            &= \cH_t^1 \cdot \P(\cE_{t,1}^{i, j}) + \cH_t^2 \cdot \big(1 - \P(\cE_{t, 1}^{i, j})\big) - \cH_t^1 \cdot \P(\cE_{t,1}^{i}) - \cH_t^2 \cdot \big(1 - \P(\cE_{t, 1}^{i})\big)\\
            &= (\cH_t^2 - \cH_t^1)\big(\P(\cE_{t, 1}^{i}) - \P(\cE_{t, 1}^{i, j})\big)\\
            &= \log(2) e^{-2t} \left(1-e^{-t}\right)^{|I_0| - 2}\left(1- e^{-|I_1|t}\right)\\
            &= O\left(e^{-2t} (1-e^{-t})^{rd/2}\right),
        \end{align*}
        whose value is independent of the indices $i$ and $j$.
        Since $|\{i, j \in I_0: i \neq j\}| = rd(rd - 1) = \Theta(d^2)$, quantity $\cI_1(t)$ satisfies 
        \begin{align}\label{eq:i1t}
            \cI_1(t) = \sum_{i,j \in I_0, i \neq j}\mi(x_t^i; x_t^j \mid x_t^{-(i,j)}) = O\left(d^2 e^{-2t} (1-e^{-t})^{rd/2}\right).
        \end{align}
        
        \paragraph{Case 2: $i, j \in I_1$, $i \neq j$.} 
        Following the proof strategy in Case 1, for any given $x_t^{-(i,j)}$, it holds that
        \begin{align*}
            x_t^i \mid x_t^{-(i,j)} \sim \frac{1}{2}e^{-t} \delta_0 + \frac{1}{2}e^{-t} \delta_1 + (1-e^{-t}) \delta_{\mask},
        \end{align*}
        which implies that
        \begin{align}\label{eq:i2-conditional-entropy-1}
            \cH(x_t^i \mid x_t^{-(i,j)}) = \cH_t^2.
        \end{align}
        Define the event $\cE_{t,2}^{i} \in \cF_t^{-i}$ as follows:
        \begin{align*}
            \cE_{t,2}^{i} \defn \Bigg\{x_t^{-i}: \Bigg(\bigvee_{k \in I_0}\{k \notin m(x_t)\}\Bigg) \bigwedge \Bigg(\bigwedge_{\ell \in I_1 \backslash \{i\}} \{\ell \in m(x_t)\}\Bigg) = 1\Bigg\},
        \end{align*}
        which induces 
        \begin{align*}
            \left(x_t^i \,\big|\, x_t^{-i} \in \cE_{t, 2}^{i}\right) &\sim e^{-t} \delta_{0/1} + (1-e^{-t}) \delta_{\mask};\\
            \left(x_t^i \,\big|\, x_t^{-i} \in (\cE_{t, 2}^{i})^c\right) &\sim \frac{1}{2}e^{-t} \delta_0 + \frac{1}{2}e^{-t} \delta_1 + (1-e^{-t}) \delta_{\mask},
        \end{align*}
        and the conditional entropy formula
        \begin{align}\label{eq:i2-conditional-entropy-2}
             \cH(x_t^i \mid x_t^{-i}) = \cH_t^1 \cdot \P(\cE_{t, 2}^{i}) + \cH_{t}^2 \cdot \big(1 - \P(\cE_{t, 2}^{i})\big).
        \end{align}
        Plugging Eqns.~\eqref{eq:i2-conditional-entropy-1} and \eqref{eq:i2-conditional-entropy-2} into \Cref{eq:mi-i-j-formula} gives that for any $i, j \in I_0$, $i \neq j$,
        \begin{align*}
            \mi(x_t^i; x_t^j \mid x_t^{-(i,j)}) &= \cH(x_t^i \mid x_t^{-(i,j)}) - \cH(x_t^i \mid x_t^{-i})\\
            &= \cH_t^2 - \cH_t^1 \cdot \P(\cE_{t, 2}^{i}) - \cH_t^2 \cdot \big(1 - \P(\cE_{t, 2}^{i})\big)\\
            &= \left(\cH_t^2 -\cH_t^1\right)\P(\cE_{t, 2}^{i})
            = O\left(e^{-(1-r)dt}\right),
        \end{align*}
        whose value is, again, independent of the indices $i$ and $j$. Since $|\{i, j \in I_1: i \neq j\}| = (1-r)d((1-r)d - 1) = \Theta(d^2)$, we reach
        \begin{align}\label{eq:i2t}
            \cI_2(t) = \sum_{i,j \in I_1, i \neq j}\mi(x_t^i; x_t^j \mid x_t^{-(i,j)}) = O(d^2 e^{-(1-r)dt}).
        \end{align}
        
        \paragraph{Case 3: $i \in I_0$, $j \in I_1$.} Define the function $\cH_{\mathrm{B}}(p) \defn -p\log(p) - (1-p)\log(1-p)$ to be the entropy of the distribution $\mathrm{Bern}(p)$.
        Following the proofs of the two cases above, let us define events
        \begin{align*}
            \cE_{t,3}^{i, j} &\defn \Bigg\{x_t^{-(i,j)}: \Bigg(\bigvee_{k \in I_0 \backslash \{i\}}\{k \neq m(x_t)\}\Bigg) \bigvee \Bigg(\bigwedge_{\ell \in I_1 \backslash \{j\}} \{\ell \in m(x_t)\}\Bigg) = 1\Bigg\};\\
            \cE_{t,3}^{i} &\defn \Bigg\{x_t^{-i}: \Bigg(\bigvee_{k \in I_0 \backslash \{i\}}\{k \neq m(x_t)\}\Bigg) \bigvee \left(\bigwedge_{\ell \in I_1}\{\ell \in m(x_t)\}\right) = 1\Bigg\}.
        \end{align*}
        Similar calculations yield
        \begin{align*}
            \cH(x_t^i \mid x_t^{-(i,j)}) = \cH_t^1 \cdot \P(\cE_{t,3}^{i, j}) + \cH_t^2 \cdot \big(1 - \P(\cE_{t, 3}^{i, j})\big);\\
            \cH(x_t^i \mid x_t^{-i}) = \cH_t^1 \cdot \P(\cE_{t,3}^{i}) + \cH_t^2 \cdot \big(1 - \P(\cE_{t, 3}^{i})\big).
        \end{align*}
        Therefore, we obtain
        \begin{align*}
            \mi(x_t^i; x_t^j \mid x_t^{-(i,j)}) 
            &= (\cH_t^2 - \cH_t^1)\big(\P(\cE_{t, 3}^{i}) - \P(\cE_{t, 3}^{i, j})\big)\\
            &= \log(2) e^{-|I_1|t} \left(1-e^{-t}\right)^{|I_0|} = O\left(e^{-\cH_{\mathrm{B}}(r)d}\right),
        \end{align*}
        where the last equality is due to the fact that $e^{-|I_1|t} \left(1-e^{-t}\right)^{|I_0|}$ is maximized at $t = -\log(1 - r)$.
        Finally, with $|\{i \in I_0, j \in I_1\}| = r(1-r)d^2$, we can bound 
        \begin{align}\label{eq:i3t}
            \cI_3(t) = \sum_{i \in I_0, j \in I_1} I(x_t^i; x_t^j \mid x_t^{-(i,j)}) = O\left(d^2 e^{-\cH_{\mathrm{B}}(r)d}\right).
        \end{align}
        
        \paragraph{Case 4: $i \in I_1$, $j \in I_0$.} Notice that $\cI_3(t)$ and $\cI_4(t)$ are invariant under swapping $i$ and $j$. We can show in the same way as above that
        \begin{align}\label{eq:i4t}
            \cI_4(t) = O\left(d^2 e^{-\cH_{\mathrm{B}}(r)d}\right).
        \end{align}
    
  \paragraph{Putting everything together.}  Combining Eqns.~\eqref{eq:i1t}, \eqref{eq:i2t}, \eqref{eq:i3t} and \eqref{eq:i4t}, we arrive at
    \begin{align}\label{eq:it}
        \cI(t) \lesssim d^2 \left(e^{-2t} (1-e^{-t})^{rd/2} + e^{-(1-r)dt} + e^{-\cH_{\mathrm{B}}(r)d}\right).
    \end{align}
    
    We are now in a position to prove \Cref{eq:int-min-it}. Let us begin with the integration over the time interval $t \in [1/d, 1]$. 
    Direct calculation yields that $e^{-2t}(1-e^{-t})^{rd/2}$ is maximized at $t^* = \log(1 + \frac{rd}{4}) > 1$, which reveals that
    \begin{align}\label{eq:t1-1/d-1}
        d^2 e^{-2t} (1-e^{-t})^{rd/2} \leq d^2 e^{-2t^*} = d^2\left(1 + \frac{rd}{4}\right)^{-2} = O(1).
    \end{align}
    For the term from $\cI_2(t)$, we obtain
    \begin{align}\label{eq:t2-1/d-1}
        \int_{\frac{1}{d}}^1 t\cdot d^2 e^{-(1-r)dt} \d t \overset{\mathrm{(a)}}{=} \int_1^d s e^{-(1-r)s} \d s \leq \int_0^\infty s e^{-(1-r)s} \d s = \frac{1}{(1-r)^2}= O(1),
    \end{align}
    where in (a), we use the change of variable formula with $s = dt$.
    Similarly, we can show that
    \begin{align}\label{eq:t34-1/d-1}
        \int_{\frac{1}{d}}^1 t\cdot d^2 e^{-\cH_{\mathrm{B}}(r)d} \d t\leq \int_{0}^1 t\cdot d^2 e^{-\cH_{\mathrm{B}}(r)d} \d t = \frac{1}{2} d^2 e^{-\cH_{\mathrm{B}}(r)d} = O(1),
    \end{align}
    where the condition $\min\{r, 1-r\} = \Theta(1)$  ensures $\cH_{\mathrm{B}}(r) = \Theta(1)$.
    Taking collectively Eqns.~\eqref{eq:it}, \eqref{eq:t1-1/d-1}, \eqref{eq:t2-1/d-1} and \eqref{eq:t34-1/d-1}, we arrive at
    \begin{align}\label{eq:int-min-it-1/d-1}
        \int_{\frac{1}{d}}^1 \min\{1, t\} \, \cI(t) \d t = \int_{\frac{1}{d}}^1 t \, \cI(t) \d t = O(1).
    \end{align}
    Let us move on to the integration over time interval $t \in [1, \log(d)]$. The integral computation yields that
    \begin{align}\label{eq:i1-int-1-logd}
        \int_1^{\log(d)} d^2 e^{-2t} (1-e^{-t})^{rd/2}\d t \overset{(a)}{=} \int_1^{d/e} s\left(1 - \frac{s}{d}\right)^{rd/2} \d s \overset{(b)}{\leq} \int_0^\infty s e^{-rs/2} \d s =  O(1),
    \end{align}
    where in (a), we use the change of variable formula with $s = de^{-t}$, and in (b), we use the inequality $(1-x)^{1/x} \leq e^{-1}$ for $x \in (0, 1]$.
    For the remaining terms, we have
    \begin{align}
        \int_1^{\log(d)} d^2 e^{-(1-r)dt}\d t &\leq \int_0^{\log(d)} d^2 e^{-(1-r)d}\d t = d^2 \log(d) e^{-(1-r)d} = O(1);\label{eq:i2-int-1-logd}\\
        \int_1^{\log(d)} d^2 e^{-\cH_{\mathrm{B}}(r)d}\d t &\leq d^2 \log(d) e^{-\cH_{\mathrm{B}}(r)d} = O(1),\label{eq:i34-int-1-logd}
    \end{align}
    where the condition $\min\{r, 1-r\} = \Theta(1)$  ensures $\cH_{\mathrm{B}}(r) = \Theta(1)$.
    Now, combining Eqns.~\eqref{eq:it}, \eqref{eq:i1-int-1-logd}, \eqref{eq:i2-int-1-logd} and \eqref{eq:i34-int-1-logd} yields 
    \begin{align}\label{eq:int-min-1-logd}
        \int_{1}^{\log(d)} \min\{1, t\} \, \cI(t) \d t = \int_{1}^{\log(d)} \cI(t) \d t = O(1).
    \end{align}
    
    Finally, equipped with Eqns.~\eqref{eq:int-min-it-1/d-1} and \eqref{eq:int-min-1-logd}, we conclude
    \begin{align*}
        \int_{\frac{1}{d}}^{\log(d)} \min\{1, t\} \, \cI(t) \d t = \int_{\frac{1}{d}}^1 \min\{1, t\} \, \cI(t) \d t + \int_{1}^{\log(d)} \min\{1, t\} \,\cI(t) \d t = O(1),
    \end{align*}
    which proves $\cD(p_{\mathrm{ex}}) = O(1)$.
    \qed

%% file: appendix/app_technical.tex
\section{Technical preparations}

\subsection{Score functions}

Below, we present an equivalent formulation of the score functions.

\begin{proposition}\label{prop:score_function}
    Let \(q_0\) be an initial distribution on \(\cX_0\). Let \(x, y \in \cX\) be such that \(Q(y, x) > 0\). Then,
    \begin{enumerate}
        \item for the uniform noising process,
        \begin{equation}
        \label{eq:unif_score}
            s_t(y, x) = \frac{\bE_{x_0 \sim q_0} \alpha_t^{\ham(y, x_0)}}{\bE_{x_0 \sim q_0} \alpha_t^{\ham{(x, x_0)}}},
        \end{equation}
        where \(\alpha_t \defn \frac{1 - e^{-t}}{1 + (S - 1)e^{-t}}\).
        \item for the masking noising process,
        \begin{equation}
        \label{eq:mask_score}
            s_t(y, x) = \frac{1}{e^t - 1} \frac{q_0(y)}{q_0(x)},
        \end{equation}
        where for \(x \in \cX \setminus \cX_0\), \(q_0(x)\) is the marginal probability of the \emph{unmasked} coordinates of \(x\) under \(q_0\).
    \end{enumerate}
\end{proposition}

\noindent
\textbf{Proof of Proposition~\ref{prop:score_function}.} \enspace
By the definition of the score function, one can write
\begin{align*}
    s_t(y, x) = \frac{q_t(y)}{q_t(x)} = \frac{\sum_{x_0} q_{t \mid 0}(y \mid x_0) q_0(x_0)}{\sum_{x_0} q_{t\mid 0}(x \mid x_0) q_0(x_0)}.
\end{align*}
For the uniform noising process, one can solve the Kolmogorov forward equation for every dimension. As a result, the transition can be written as  
\begin{align*}
    q_{t \mid 0}(y \mid x_0) = \left(\frac{1 - e^{-t}}{S}\right)^{\ham(y, x_0)}\left(\frac{1+(S-1)e^{-t}}{S}\right)^{d - \ham(y, x_0)} = \left(\frac{1+(S-1)e^{-t}}{S}\right)^d \alpha_t^{\ham(y, x_0)},
\end{align*}
which proves~\Cref{eq:unif_score}. More details of this relation can be found in (e.g.,~\cite{zhang2024convergence}, Proposition 1).

For the masking noising process, for notational convenience, given any \(x \in ([S] \cup \{\mask\})^d\), define 
\begin{align}
\label{eqn:mx-def}
    m(x) \defn \{i \in [d]: x^i = \mask\}.
\end{align}
In view of this piece of notation, as \(\Pr(x_t^i = \mask) = e^{-t}\), and coordinates evolve independently, one can write 
\begin{align*}
    q_{t \mid 0}(y \mid x_0) = (1 - e^{-t})^{\abs{m(y)}} e^{-t(d - \abs{m(y)})} \bI\{\,\text{for all }i \in [d], y^i \in \{x_0^i, \mask\}\,\}.
\end{align*}
As \(Q(y, x) > 0\), it must be that \(\ham(x, y) = 1\), and for \(i\), such that \(x^i \neq y^i\), \(x^i = \mask\) and \(y^i \neq \mask\).
This implies that \(\abs{m(x)} = \abs{m(y)} + 1\), and we can write
\begin{align*}
    \frac{\sum_{x_0} q_{t \mid 0}(y \mid x_0) q_0(x_0)}{\sum_{x_0} q_{t\mid 0}(x \mid x_0) q_0(x_0)} = \frac{e^{-t}}{1 - e^{-t}} \frac{\sum_{x_0} q_0(x_0) \bI\{\,\text{for all }i \in [d], y^i \in \{x_0^i, \mask\}\,\}}{\sum_{x_0} q_0(x_0) \bI\{\,\text{for all }i \in [d], x^i \in \{x_0^i, \mask\}\,\}} = \frac{1}{e^{t} - 1} \frac{q_0(y)}{q_0(x)}.
\end{align*}
\qed

\subsection{Technical lemmas}
\begin{lemma}[Chain rule of KL divergence]
\label{lem:chain_kl}
    For \(N > 0\), let \(a_{0:N}, b_{0:N}\) be the joint distributions of two Markov processes. Then,
    \begin{equation*}
        \KL (a_{0:N} \| b_{0:N}) = \KL (a_0 \| b_0) + \sum_{k=0}^{N-1} \bE_{x \sim a_k} \KL \left(a_{k+1 \mid k}(\cdot \mid x) \| b_{k+1 \mid k}(\cdot \mid x)\right).
    \end{equation*}
\end{lemma}

\begin{proof}
Invoking the definition of KL divergence with some direct calculations yields 
    \begin{align*}
        \KL (a_{0:N} \| b_{0:N}) &= 
        \bE_{x_{0:N} \sim a_{0:N}} \log \frac{a_{0:N}(x_{0:N})}{b_{0:N}(x_{0:N})} \\
        &= 
        \bE_{x_{0:N} \sim a_{0:N}} \log \left(\frac{a_{0}(x_0)}{b_{0}(x_0)} \prod_{k=0}^{N - 1} \frac{a_{k+1 \mid k}(x_{k+1} \mid x_k)}{b_{k+1 \mid k}(x_{k+1} \mid x_k)}\right) \\
        &= \bE_{x_0 \sim a_0} \log \frac{a_0(x_0)}{b_0(x_0)} + \sum_{k=0}^{N - 1} \bE_{x_k \sim a_k} \bE_{x_{k+1} \sim a_{k + 1 \mid k} (\cdot \mid x_k)} \log \frac{a_{k+1 \mid k}(x_{k+1} \mid x_k)}{b_{k+1 \mid k}(x_{k+1} \mid x_k)} \\
        &= \KL(a_0 \| b_0) + \sum_{k=0}^{N - 1}\bE_{x_k \sim a_k} \KL\left(a_{k+1 \mid k}(\cdot \mid x_k) \| b_{k+1 \mid k}(\cdot \mid x_k)\right).
    \end{align*}
\end{proof}

\begin{lemma}
\label{lem:kl_der_ub}
Let \((q_t)\) and \((p_t)\) be the marginals of CTMCs with rate matrices \((Q_t)\) and \((\wh Q_t)\), respectively, and \(\back q_t \equiv q_{T - t}\) be the marginals of the reverse process.  Then, for any \(t > t_k\) and for any \(z\),
\begin{align*}
    &\frac{\partial}{\partial t}  \KL\left(\back q_{t \mid t_k}( \cdot \mid z) \| p_{t \mid t_k}(\cdot \mid z)\right) \leq \bE_{x_{t} \sim \back q_{t \mid t_k}(\cdot \mid z)} \sum_{y \neq x_t}\left[ \wh Q_t(x_t, y) - \back Q_t(x_t, y) + \back Q_t(x_t, y) \log \frac{\back Q_t(x_t, y)}{\wh Q_t(x_t, y)}\right].
\end{align*}
\end{lemma}

\begin{proof} 
Let us omit the conditioning on \(z\) for the notation brevity.
By direct calculations, one can write 
    \begin{align*}
        A &\coloneqq \frac{\partial}{\partial t}  \KL\left(\back q_{t \mid t_k} \| p_{t \mid t_k}\right) = \sum_{x \in \cX} \left(\frac{\partial }{\partial t} \back q_{t\mid t_k}(x)\right) \log \frac{\back q_{t\mid t_k}(x)}{p_{t\mid t_k}(x)} - \sum_{x \in \cX}  \back q_{t\mid t_k}(x) \frac{
        \frac{\partial}{\partial t} p_{t\mid t_k}(x)}{p_{t\mid t_k}(x)}.
    \end{align*}
    Recall the Kolmogorov equation:
    \begin{equation*}
        \frac{\partial }{\partial t} \back q_{t\mid t_k}(x) = \sum_{y \in \cX} \back Q_t(y, x) \back q_{t\mid t_k}(y) \quad \text{and} \quad \frac{\partial }{\partial t} p_{t\mid t_k}(x) = \sum_{y \in \cX} \wh Q_t(y, x) p_{t\mid t_k}(y) 
    \end{equation*}
    Putting the above together, we obtain (relabeling \(x\) and \(y\))
    \begin{equation}
    \begin{aligned}
        A &= \bE_{x \sim \back q_{t \mid t_k}} \sum_{y \in \cX}  \left[\back Q_t(x, y) \log \left(\frac{\back q_{t\mid t_k}(y)}{p_{t\mid t_k}(y)}\right) - \wh Q_t(x, y)\frac{\back q_{t\mid t_k}(y)}{\back q_{t\mid t_k}(x)} \cdot \frac{p_{t\mid t_k}(x)}{p_{t\mid t_k}(y)} \right] \\
        &= \bE_{x \sim \back q_{t \mid t_k}} \sum_{y \neq x}  \left[\back Q_t(x, y) \log \left(\frac{\back q_{t\mid t_k}(y)}{p_{t\mid t_k}(y)}\right) - \back Q_t(x, y) \log \left(\frac{\back q_{t\mid t_k}(x)}{p_{t\mid t_k}(x)}\right) - \wh Q_t(x, y)\frac{\back q_{t\mid t_k}(y)}{\back q_{t\mid t_k}(x)} \cdot \frac{p_{t\mid t_k}(x)}{p_{t\mid t_k}(y)} \right] - \wh Q_t(x,x) \\
        &= \bE_{x \sim \back q_{t \mid t_k}} \sum_{y \neq x}  \left[\back Q_t(x, y) \log \left(\frac{\back q_{t\mid t_k}(y)}{\back q_{t\mid t_k}(x)}\cdot\frac{p_{t\mid t_k}(x)}{p_{t\mid t_k}(y)}\right) - \wh Q_t(x, y)\frac{\back q_{t\mid t_k}(y)}{\back q_{t\mid t_k}(x)} \cdot \frac{p_{t\mid t_k}(x)}{p_{t\mid t_k}(y)} \right] - \wh Q_t(x,x) \\
        &= \bE_{x \sim \back q_{t \mid t_k}} \sum_{y \neq x}  \left[\back Q_t(x, y) \log \left(\frac{\back q_{t\mid t_k}(y)}{\back q_{t\mid t_k}(x)}\cdot\frac{p_{t\mid t_k}(x)}{p_{t\mid t_k}(y)}\right) - \wh Q_t(x, y)\frac{\back q_{t\mid t_k}(y)}{\back q_{t\mid t_k}(x)} \cdot \frac{p_{t\mid t_k}(x)}{p_{t\mid t_k}(y)}  + \wh Q_t(x, y)\right],
    \end{aligned}
    \end{equation}
    where we invoke the the property that $\back Q_t(x, x) = -\sum_{y \neq x} \back Q_t(x, y)$
    and $\wh Q_t(x, x) = -\sum_{y \neq x} \wh Q_t(x, y).$
    Then, letting \(C_{xy}\) be such that (recall that \(z\) is fixed)
    \begin{equation}
        \frac{\back q_{t\mid t_k}(y \mid z)}{\back q_{t\mid t_k}(x \mid z)} = \frac{\back Q_t(x, y)}{C_{xy}},
    \end{equation}
    it satisfies that 
    \begin{equation}\label{eq:tight-decomp}
        \begin{aligned}
                A = \bE_{x \sim \back q_{t \mid t_k}} \sum_{y \neq x}  \Biggl[\wh Q_t(x, y) + \back Q_t(x, y) \log \left(\frac{\back Q_t(x, y)}{\wh Q_t(x, y)}\right) &+ \back Q_t(x, y) \log \left(\frac{\wh Q_t(x, y)}{C_{xy}}\cdot\frac{p_{t\mid t_k}(x)}{p_{t\mid t_k}(y)}\right)\\ &- \wh Q_t(x, y)\frac{\back Q_t(x, y)}{C_{xy}} \cdot \frac{p_{t\mid t_k}(x)}{p_{t\mid t_k}(y)} \Biggr]  
        \end{aligned}
    \end{equation}
    Finally, since \(\log z \leq z - 1\),
    \begin{equation*}
        \begin{aligned}
                A &\leq \bE_{x \sim \back q_{t \mid t_k}} \sum_{y \neq x}  \Biggl[\wh Q_t(x, y) + \back Q_t(x, y) \log \left(\frac{\back Q_t(x, y)}{\wh Q_t(x, y)}\right) - \back Q_t(x, y) \\
                & \qquad \qquad \qquad \qquad \qquad\ + \back Q_t(x, y) \frac{\wh Q_t(x, y)}{C_{xy}}\cdot\frac{p_{t\mid t_k}(x)}{p_{t\mid t_k}(y)} - \wh Q_t(x, y)\frac{\back Q_t(x, y)}{C_{xy}} \cdot \frac{p_{t\mid t_k}(x)}{p_{t\mid t_k}(y)} \Biggr] \\
                &= \bE_{x \sim \back q_{t \mid t_k}} \sum_{y \neq x} \left[\wh Q_t(x, y) - \back Q_t(x, y) + \back Q_t(x, y) \log \left(\frac{\back Q_t(x, y)}{\wh Q_t(x, y)}\right)\right].
        \end{aligned}
    \end{equation*}
\end{proof}
\begin{lemma}[It\^{o}'s Lemma for Poisson jump process]\label{lem:ito}
    For the Poisson jump process $\{x_t\}_{t \ge 0}$ with generator $\{L_t\}_{t \ge 0}$ and rate matrix $\{R_t\}_{t \ge 0}$. It\^{o}'s Lemma formula can be written as
    \begin{align}
    f(t, x_t) = f(0, x_0) + \int_{0}^t \big[\partial_s f(s, x_{s^-}) + \left(L_s f\right)(s, x_{s^-})\big] \d t + M_t,
\end{align}
where $x_{s-} = \lim_{u \to s^-} x_{s}$, which exists for almost everywhere $s \in [0, t)$ under the Lebesgue measure. The compensation process $\{M_u\}_{u \in [\ell, t]}$ is defined as
\begin{align*}
    M_u = \sum_{y_s : y_s \neq x_s}\int_\ell^u \big(f(s, y_s) - f(s, x_s)\big)\big(\d N_s^{x_s, y_s} - \lambda_s^{x_s, y_s} \d s\big),
\end{align*}
where $N_s^{x, y}$ is the counting process of jumps from $x$ to $y$ up to time $t$ and $\lambda_s^{x, y}$ is the $\cF_{s^-}$-intensity of $N_s^{x, y}$, i.e., $\lambda_s^{x, y} = \bI\{x_{s^-} = y\} R_t(x, y)$.
\end{lemma}
See \citep[Appendix A.5]{conforti2025non} for more details.

%% file: appendix/app_main_proofs.tex
\section{Proofs of results in Section~\ref{sec:uniform}}

\subsection{Proof of Theorem~\ref{thm:uniform}}\label{sec:proof_uniform}
We first decompose the KL divergence between the output distribution \(p_T\) and the target distribution \(q_0\) as
    \begin{align}
        \KL(q_0 \| p_T) &\leq \KL(q_{T - t_0, \ldots, T - t_N} \| p_{t_0, \ldots, t_N})\notag\\
        &= \KL(q_T \| p_0) + \sum_{k=0}^{N - 1} \bE_{x_{t_k} \sim \back q_{t_k}} \left[\KL\left(\back q_{t_{k+1}\mid t_k}(\cdot|x_{t_k}) \Vert p_{t_{k+1}\mid t_k}(\cdot|x_{t_k})\right)\right],\label{eq:uniform_kl_first_decomp}
    \end{align}
    where the first inequality follows from the data-processing inequality for KL divergence and the second inequality follows from the chain rule for KL divergence in Lemma~\ref{lem:chain_kl}. 
    The first term is the initialization error, which can be upper bounded by the following lemma.
\begin{lemma}\label{lem:uniform_log_sobolev}
    For the uniform noising process, for any initial distribution \(q_0 \in \cP(\cX)\), time index \(t \geq 0\), one has the same limit distribution
    \begin{align*}
        q_t \overset{d}{\to} p_0 = \mathrm{Unif}(\cX), \quad \text{as } t \to \infty.
    \end{align*}
    Further, the modified log-Sobolev constant \footnote{$C_{\mathrm{LSI}}$ is defined as the smallest number such that for any $q \in \cP(\cX)$, $\KL(q \mid \mathrm{Unif}(\cX)) \leq C_{\mathrm{LSI}}/2 \cdot \cE(q, \log(q))$, where $\cE$ is the Dirichlet form associated with the uniform noising process, i.e., $\cE(f, g) = -(2 |\cX|)^{-1} \sum_{x, y \in \cX}(f(x) - f(y))(g(x) - g(y)) Q(x, y)$.} of \(q_t\) satisfies \(C_{\mathrm{LSI}} = 2\), which leads to 
    \begin{align*}
        \KL(q_t \| p_0) \le e^{-t} \KL(q_0 \| p_0) \le e^{-t} d \log(S).
    \end{align*}
\end{lemma}
The proof of Lemma~\ref{lem:uniform_log_sobolev} can be found in previous works, e.g., \citet[Proposition 2]{zhang2024convergence}. Applying the lemma above together with  Lemma~\ref{lem:kl_der_ub} and~\Cref{eq:tau-leaping} on the second term in Eqn.~\eqref{eq:uniform_kl_first_decomp}, we obtain
\begin{align}
    &\KL(p_0 \| q_T)\notag\\
    &= \KL(q_T \| p_0) + \sum_{k=0}^{N - 1} \bE_{x_{t_k} \sim \back q_{t_k}} \left[\int_{t_k}^{t_{k+1} }\frac{\partial}{\partial t}\KL\left(\back q_{t\mid t_k}(\cdot|x_{t_k}) \Vert p_{t \mid t_k}(\cdot|x_{t_k})\right)\d t\right]\notag\\
    &\leq e^{-T}d\log(S)\notag\\ 
    &\quad +  \sum_{k=0}^{N - 1}\bE_{x_{t_k} \sim \back q_{t_k}} \int_{t_k}^{t_{k+1} } \bE_{x_{t} \sim \back q_{t \mid t_k}} \sum_{y \neq x_t}\left[ \wh Q_t(x_t, y) - \back Q_t(x_t, y) + \back Q_t(x_t, y) \log \frac{\back Q_t(x_t, y)}{\wh Q_t(x_t, y)}\right]\d t\notag\\
    &\leq e^{-T}d\log(S) + \frac{1}{S} \sum_{k=0}^{N-1} \int_{t_k}^{t_{k+1}} \bE_{x_{t_k}, x_t \sim \back q_{t_k, t}} \bigg[\sum_{i \in [d]}\sum_{c \in [S]} s_{T - t}(x_t  \oplus_i c, x_t)\notag\\ 
    &\hspace{20em} \breg \big(\wh s_{T - t_k}(x_{t_k} \oplus_i c, x_{t_k}), s_{T - t}(x_t \oplus_i c, x_t)\big)\d t\bigg].\label{eq:uniform_kl_main}
\end{align}
In the following, we focus on the quantity
\begin{align}\label{eq:uniform_tk_quantity}
    \bE_{x_{t_k}, x_t \sim \back q_{t_k, t}} \sum_{i \in [d]}\sum_{c \in [S]} s_{T - t}(x_t  \oplus_i c, x_t) \breg \big(\wh s_{T - t_k}(x_{t_k} \oplus_i c, x_{t_k}), s_{T - t}(x_t \oplus_i c, x_t)\big).
\end{align}
For simplicity, we write $t_k \defn \ell$.
Direct calculations yield the following decomposition
\begin{align*}
    &\sum_{i \in [d]}\sum_{c \in [S]} s_{T - t}(x_t  \oplus_i c, x_t) \breg \big(\wh s_{T - t_k}(x_{t_k} \oplus_i c, x_{t_k}), s_{T - t}(x_t \oplus_i c, x_t)\big)\\
    &= \underbrace{\sum_{y_\ell: \ham(y_\ell, x_\ell) = 1} s_{T - \ell}(y_\ell, x_\ell)\breg \left(\wh s_{T - \ell}(y_\ell , x_\ell), s_{T - \ell}(y_\ell, x_{\ell})\right)}_{T_{1}^{t,\ell}} \\
    &\quad + \underbrace{\sum_{i \in [d]}\sum_{c \in [S]} \left(s_{T - \ell}(x_\ell \oplus_i c, x_\ell) - s_{T - t}(x_t \oplus_i c, x_t)\right) \log \wh s_{T - \ell}(x_\ell \oplus_i c, x_\ell)}_{T_2^{t, \ell}}\\
    &\quad + \underbrace{\sum_{y_t: \ham(y_t, x_t) = 1} h\left(s_{T-t}(y_t, x_t)\right) -  \sum_{y_\ell: \ham(y_\ell, x_\ell) = 1}h\left(s_{T - \ell}(y_{\ell}, x_{\ell})\right)}_{T^{t, \ell}_3},
\end{align*}
where $h(x) = x\log x - x + 1$. 
We proceed by bounding each term separately. 

\begin{itemize}
    \item For term $T_1^{t, \ell}$, notice that $T_1^{t, \ell}$ is independent of $t$. In view of definition of score entropy loss, we have
    \begin{align}
        &\bE_{x_{\ell}, x_t \sim \back q_{\ell, t}}\left[T_1^{t, \ell}\right]\\ 
        &= \bE_{x_{\ell}, x_t \sim \back q_{\ell, t}}\left[\sum_{y_\ell: \ham(y_\ell, x_\ell) = 1} s_{T - \ell}(y_\ell, x_\ell)\breg \left(\wh s_{T - \ell}(y_\ell , x_\ell), s_{T - \ell}(y_\ell, x_{\ell})\right)\right]\notag\\
        &= S\cdot\bE_{x_{\ell}, x_t \sim \back q_{\ell, t}}\left[\sum_{y_\ell: \ham(y_\ell, x_\ell) = 1} Q_{T-\ell}(y_\ell, x_\ell) s_{T - \ell}(y_\ell, x_\ell)\breg \left(\wh s_{T - \ell}(y_\ell , x_\ell), s_{T - \ell}(y_\ell, x_{\ell})\right)\right]\notag\\
        &= S\cdot\cL_{\mathrm{SE}}(T - \ell, \wh s_{T-\ell}, s_{T-\ell}),\label{eq:uniform_t1}
    \end{align}
    where we use the fact that $Q_{T-t}(y, x) = S^{-1}$ for any $\ham(y, x) = 1$.
\item For term $T_2^{t, \ell}$, we establish the following lemma, whose proof is provided in Section~\ref{pf:lem-uniform-martingale}.
\begin{lemma}\label{lem:uniform_martingale}
    Consider the uniform noising process and let \(0 \leq \ell < t < T\). Then, for any \(c \in [S]\), \(i \in [d]\) and $x_\ell \in \cX$, it obeys 
    \begin{align*}
        \bE_{x_t \sim \back q_{t|\ell}(\cdot \mid x_{\ell})} \Big[\big(s_{T - \ell}(x_\ell \oplus_i c, x_\ell) - s_{T - t}(x_t \oplus_i c, x_t)\big) \log \wh s_{T - \ell}(x_\ell \oplus_i c, x_\ell)\Big] = 0.
    \end{align*}
\end{lemma}
With Lemma~\ref{lem:uniform_martingale}, it is easily seen that
\begin{align}
    &\bE_{x_{\ell}, x_t \sim \back q_{\ell, t}}\left[T_2^{t, \ell}\right]\notag\\
    &= \bE_{x_{\ell}, x_t \sim \back q_{\ell, t}}\Bigg[\sum_{i \in [d]} \sum_{c \in [S]}\big(s_{T - \ell}(x_\ell \oplus_i c, x_\ell) - s_{T - t}(x_t \oplus_i c, x_t)\big) \log \wh s_{T - \ell}(x_\ell \oplus_i c, x_\ell)\Bigg]\notag\\
    &= \sum_{i \in [d]} \sum_{c \in [S]}\bE_{x_{\ell} \sim \back q_\ell}\bigg[\bE_{x_t \sim \back q_{t|\ell}(\cdot \mid x_{\ell})} \Big[\big(s_{T - \ell}(x_\ell \oplus_i c, x_\ell) - s_{T - t}(x_t \oplus_i c, x_t)\big) \log \wh s_{T - \ell}(x_\ell \oplus_i c, x_\ell)\Big]\bigg]\notag\\
    &= 0. \label{eq:uniform_t2}
\end{align}
\item For term $T_3^{t, \ell}$, we make the crucial observation that $\bE_{x_t \sim \back q_{t}}[\sum_{y_t: \ham(y_t, x_t) = 1} h(s_{T-t}(y_t, x_t))]$ admits a simple representation.
The statement is formalized in the following lemma.
\begin{lemma}
\label{lem:uniform_t3_sim}
    For any $t \in [0, T]$, we have
    \begin{align*}
        &\bE_{x_t \sim \back q_{t}}\left[\sum_{y_t: \ham(y_t, x_t) = 1} h\left(s_{T-t}(y_t, x_t)\right)\right] = \bE_{x_t \sim \back q_{t}}\left[\sum_{y_t: \ham(y_t, x_t) = 1} -\log\left(s_{T-t}(y_t, x_t)\right)\right].
    \end{align*}
\end{lemma}

In view of this lemma, we can further express the term $T_3^{t, \ell}$ as
\begin{align}
    &\bE_{x_{\ell}, x_t \sim \back q_{\ell, t}}\left[T_3^{t, \ell}\right]\\
    &= \bE_{x_t \sim \back q_{t}}\left[\sum_{y_t: \ham(y_t, x_t) = 1} h\left(s_{T-t}(y_t, x_t)\right)\right] - \bE_{x_{\ell}\sim \back q_{\ell}}\left[\sum_{y_\ell: \ham(y_\ell, x_\ell) = 1} h\left(s_{T-\ell}(y_\ell, x_\ell)\right)\right]\notag\\
    &=\bE_{x_t \sim \back q_{t}}\left[\sum_{y_t: \ham(y_t, x_t) = 1} -\log\left(s_{T-t}(y_t, x_t)\right)\right] - \bE_{x_{\ell}\sim \back q_{\ell}}\left[\sum_{y_\ell: \ham(y_\ell, x_\ell) = 1} -\log\left(s_{T-\ell}(y_\ell, x_\ell)\right)\right].\label{eq:uniform_t3}
\end{align}
\end{itemize}

Plugging Eqns.~\eqref{eq:uniform_t1}, \eqref{eq:uniform_t2}, and \eqref{eq:uniform_t3} into Eqn.~\eqref{eq:uniform_tk_quantity}, we end up with 
\begin{align*}
    &\bE_{x_{t_k}, x_t \sim \back q_{t_k, t}} \sum_{i \in [d]} \sum_{c \in [S]} s_{T - t}(y_t, x_t) \breg \left(\wh s_{T - t_k}(x_{t_k} \oplus_i c, x_{t_k}), s_{T - t}(y_t, x_t)\right)\\
    &= S \cdot \cL_{\mathrm{SE}}(T-\ell, \wh s_{T-\ell}, s_{T-\ell}) + \bE_{x_t \sim \back q_{t}}\left[\sum_{y_t: \ham(y_t, x_t) = 1} -\log\left(s_{T-t}(y_t, x_t)\right)\right]\\
    &\hspace{20em} - \bE_{x_{\ell}\sim \back q_{\ell}}\left[\sum_{y_\ell: \ham(y_\ell, x_\ell) = 1} -\log\left(s_{T-\ell}(y_\ell, x_\ell)\right)\right]\\
    &= S \cdot \cL_{\mathrm{SE}}(T-\ell, \wh s_{T-\ell}, s_{T-\ell}) + S\big(\varphi(T-t) - \varphi(T-\ell)\big), 
\end{align*}
where we define $\varphi(t)$ as
\begin{align}\label{eq:defn-phi}
    \varphi(t) \defn \frac{1}{S}\bE_{x_t \sim q_{t}}\left[\sum_{y_t: \ham(y_t, x_t) = 1} -\log\left(s_t(y_t, x_t)\right)\right].
\end{align}
Returning to Eqn.~\eqref{eq:uniform_kl_main}, we conclude that
\begin{align}
    \KL(p_0 \| q_T) 
    &\leq e^{-T}d\log(S) + \sum_{k=0}^{N-1} \int_{t_k}^{t_{k+1}} \cL_{\mathrm{SE}}(T - t_k, \wh s_{T-t_k}, s_{T-t_k}) + \big(\varphi(T-t) - \varphi(T-t_k)\big)\,\d t\notag\\
    &\leq \escore + e^{-T}d\log(S) + \sum_{k=0}^{N-1} \int_{t_k}^{t_{k+1}} \big(\varphi(T-t) - \varphi(T-t_k)\big) \,\d t.\label{eq:uniform_kl_final}
\end{align}

To establish Theorem~\ref{thm:uniform}, it is only left for us to control the last term in Eqn.~\eqref{eq:uniform_kl_final}. 
First, by Jensen's inequality, $\varphi(t)$ is lower bounded by 
\begin{align}\label{eq:phi_lower}
    \varphi(t) \geq -\frac{1}{S}\sum_{i \in [d]}\sum_{c \in [S]}\log\big(\bE_{x_t \sim  q_{t}}\left[s_t(x_t \oplus_i c, x_t)\right]\big) = 0.
\end{align}
For the upper bound, from the definition of $\varphi(t)$, it satisfies that 
\begin{align}\label{eq:phi_upper_1}
    \varphi(t) \leq \frac{1}{S}\bE_{x_t \sim q_{t}}\left[\big|\{y_t: \ham(y_t, x_t) = 1\}\big| \cdot \sup_{x, y: \ham(x,y) = 1}\big|\log\left(s_t(y, x)\right)\big|\right],
\end{align}
where $|\{y_t: \ham(y_t, x_t) = 1\}|$ denotes the cardinality of the set $\{y_t: \ham(y_t, x_t) = 1\}$, which equals $d(S-1)$ for any $x_t \in \cX$.
It therefore suffices to control the quantity $|\log\left(s_t(y_t, x_t)\right)|$, which is achieved through the following lemma. 

\begin{lemma}\label{lem:uniform_log_score_bound}
    For any distribution \(q_0\) on \(\cX\), let \(q_t\) be the marginal distribution of the uniform noising process at time \(t\). Then, for any $x, y \in \cX$ such that $\ham(x, y) = 1$, it holds that
    \begin{align*}
        |\log s_t(y, x)| \lesssim \log(S) + \max\{\log(t^{-1}), 0\}.
    \end{align*}
\end{lemma}

As a consequence of Lemma~\ref{lem:uniform_log_score_bound}, we arrive at 
\begin{align}\label{eq:phi_upper}
    \varphi(t) \le \frac{d(S-1)}{S}\cdot \sup_{x, y: \ham(x, y) = 1}|\log\left(s_t(y, x)\right)|
    \lesssim d \big(\log(S) + \max\{\log(t^{-1}), 0\}\big).
\end{align}
In addition, we make the observation in Lemma~\ref{lem:phi'_lower_bound} that $\varphi(t)$ is a non-increasing function in $t$.

Now we are ready to combine everything and bound the last term of Eqn.~\eqref{eq:uniform_kl_final}.
Define $\Delta = \max_k\{t_{k+1} - t_k\}$, and choose $1 \le M \le N-1$ such that $T - t_M \in [\Delta, 2\Delta]$.
Armed with Eqns.~\eqref{eq:phi_lower}, \eqref{eq:phi_upper} and the monotonicity of $\varphi(t)$, we obtain
\begin{align*}
    &\sum_{k=0}^{N-1} \int_{t_k}^{t_{k+1}} \big(\varphi(T-t) - \varphi(T-t_k)\big) \,\d t\\
    &\leq \int_{0}^{T-t_{M}} \varphi(t) \,\d t + \sum_{k=0}^{M-1}\int_{t_k}^{t_{k+1}} \big(\varphi(T-t_{k+1}) - \varphi(T-t_k)\big) \,\d t\\
    &\leq 2\Delta d\big(\log(S) +  \log(2/\Delta) + 1\big) + \Delta\sum_{k=0}^{N-2} \big(\varphi(T-t_{k+1}) - \varphi(T-t_k)\big)\\
    &\le 2\Delta d\big(\log(S) +  \log(2/\Delta) + 1\big) + \Delta \varphi(\Delta) \lesssim \Delta d\log(S/\Delta).
\end{align*}

Combining the inequality above with Eqn.~\eqref{eq:uniform_kl_final} achieves
\begin{align*}
    \KL(p_0 \| q_T) \leq \escore + e^{-T}d\log(S) + \Delta d\log(S/\Delta),
\end{align*}
which completes the proof of Theorem~\ref{thm:uniform}.
\qed
\subsection{Proof of Corollary~\ref{cor:uniform}}
\label{sec:proof-uniform-cor}
Choose time horizon $T = \log(d \log(S)/\varepsilon)$ and number of discretization steps
\begin{align*}
    N = \Theta\left(\frac{d\log(S)\log^3(d\log(S)/\varepsilon)}{\varepsilon}\right) = \widetilde{\Theta}\left(\frac{d}{\varepsilon}\right).
\end{align*}
Adopting the upper bound in Theorem~\ref{thm:uniform} leads to
\begin{align*}
    \KL(q_{\mathrm{data}} \| p_{\mathrm{output}}) = \KL(q_{0} \| p_{T})
    &\lesssim \escore + e^{-T} d\log(S) + \frac{Td}{N} \log\left(\frac{SN}{T}\right)\\
    &\lesssim \escore + \varepsilon + \frac{\varepsilon}{\log(S) T^2} \Big(\log(S) + 3T\Big)\\
    &\lesssim \escore + \varepsilon.
\end{align*}
\qed

\subsection{Proof of Theorem~\ref{thm:uniform_lower}}\label{sec:uniform_lower_proof}

Recall that the path measures of the backward process and the sampling process
are denoted by $Q \overset{d}{=} \{\back q_t\}_{t \in [0,T-\delta]}$
and $P \overset{d}{=} \{p_t\}_{t\in[0,T-\delta]}$, respectively. 
It can be checked that the path measure $Q$ is absolutely continuous with respect to $P$. By Girsanov's theorem for the backward process (e.g.~\citet[Corollary 3.4]{ren2024discrete}), it satisfies 
\begin{align*}
    &\KL(Q \,\|\, P)\\
    &= \KL(\back q_0 \Vert p_0) + \frac{1}{S} \sum_{k=0}^{N-1} \int_{t_k}^{t_{k+1}} \bE_{x_{t_k}, x_t \sim \back q_{t_k, t}}\bigg[\sum_{i \in [d]}\sum_{c \in [S]} s_{T - t}(x_t \oplus_i c, x_t) \\
    &\hspace{20em} \breg \left(\wh s_{T - t_k}(x_{t_k} \oplus_i c, x_{t_k}), s_{T - t}(x_t \oplus_i c, x_t)\right)\d t\bigg].
\end{align*}
Following the same analysis as in Eqns.~\eqref{eq:uniform_tk_quantity} to~\eqref{eq:uniform_kl_final} in Appendix~\ref{sec:proof_uniform},  we arrive at
\begin{align*}
    \KL(Q \,\|\, P) &= \varepsilon_{\mathrm{score}} +\KL(\back q_0 \Vert p_0) +  \sum_{k=0}^{N-1} \int_{t_{k}}^{t_{k+1}} \left(\varphi(T-t) - \varphi(T-t_k)\right)\d t\\
    &= \varepsilon_{\mathrm{score}} + \KL(q_T \Vert p_0) + \sum_{k=0}^{N-1} \int_{t_{k}}^{t_{k+1}} \left(\varphi(T-t) - \varphi(T-t_k)\right)\d t,
\end{align*}
where the function $\varphi(\cdot)$ is defined as in Eqn.~\eqref{eq:defn-phi}
\begin{align*}
    \varphi(t) \defn \frac{1}{S}\bE_{x_t \sim q_{t}}\left[\sum_{y_t: \ham(y_t, x_t) = 1} -\log\left(s_t(y_t, x_t)\right)\right].
\end{align*}

Thus, to achieve $\KL(Q\,\|\,P) \leq \escore + O(1)$, we need to select $N, T$ and step size schedule such that 
    \begin{align}
        \KL(q_T \Vert p_0) + \sum_{k=0}^{N-1} \int_{t_{k}}^{t_{k+1}} \left(\varphi(T-t) - \varphi(T-t_k)\right)\d t = O(1).\label{eq:uniform_lower_mid}
    \end{align}
In order to understand the first term in Eqn.~\eqref{eq:uniform_lower_mid},
let us first consider the case when $T=1$. By the assumption $q_0 \in \cP^{\gamma}(\cX)$, therefore, we are ensured that 
    \begin{align*}
        \KL(q_1 \| p_0) = \sum_{x \in \cX} q_1(x) \log\big(q_1(x)) - \sum_{x \in \cX} q_1(x) \log\big(p_0(x)) \geq \gamma d \log(S) \gg 1. 
    \end{align*}
Hence, it implies that $T$ must satisfy $T > 1$.

We then focus on the analysis of the second term in Eqn.~\eqref{eq:uniform_lower_mid}.
We aim to show that the changing rate, i.e., $-\varphi'(t)$, is lower bounded as we come close to the target data distribution (i.e., $t \in [0,1])$, which in turn leads to a lower bound on the difference $\varphi(T-t) - \varphi(T-t_k)$.
We proceed with our analysis under the information-theoretic framework.

For notational convenience, given every $i \in [d]$ and $c \in [S]$, let us define 
\begin{align}\label{eq:defn_phi_ic}
    \varphi_{i,c}(t) = \E_{x_t \sim q_t}\left[-\log \big(s_t(x_t \oplus_i c, x_t)\right] = \E_{x_t \sim q_t}\left[-\log \big(s_t(N_{i,c}(x_t), x_t)\right]
\end{align}
where the operator $N_{i,c}: \cX \to \cX$ is defined as $N_{i, c}(x) = x \oplus_i c$. It is easy to check that $$\varphi(t) = \frac{1}{S} \sum_{i \in [d]} \sum_{c \in [S]} \varphi_{i, c}(t).$$ Notice that $N_{i, c}$ is a bijection in $\cX$.
We define $N_{i, -c} = (N_{i, c})^{-1} = N_{i, S-c}$, where $(N_{i, c})^{-1}$
is denoted as the inverse function of $N_{i,c}$.

Since $\varphi(t)$ can be written as a linear combination of $\varphi_{i, c}(t)$, it suffices to study the properties of the individual $\varphi_{i, c}(t)$ to characterize $\varphi(t)$.
To begin with, the following lemma provides a characterization of $\varphi(t)$ and $\varphi_{i, c}(t)$ as information-theoretic quantities.
\begin{lemma}\label{lem:dkl_dt_phi}
For $\varphi(t)$ and $\varphi_{i,c}(t)$ defined in Eqns.~\eqref{eq:defn-phi} and \eqref{eq:defn_phi_ic}, we have
    \begin{align*}
        \varphi_{i,c}(t) &= \KL(q_t \,\|\, (N_{i,-c})_{\#}q_t);\\
        \varphi(t) &= -\frac{\partial}{\partial t} \KL(q_t \| p_0) = \sum_{i \in [d]}\sum_{c \in [S]} \KL(q_t \,\|\, (N_{i,c})_{\#}q_t),
    \end{align*}
    where $(N_{i, c})_{\#}$ is denoted as the pushforward measure of $q_t$ under operator $N_{i,c}$.
\end{lemma}
Here, the pushforward measure gives, for any $x \in \cX$,
\begin{align}
    (N_{i,-c})_{\#}q_t(x) = q_t\left(N_{i,c}(x)\right).\label{eq:pushforward_ic}
\end{align}
Lemma~\ref{lem:dkl_dt_phi} allows us to write $\varphi_{i,c}(t)$ as the KL divergence between the marginal forward process and its pushforward under $N_{i, c}$. By viewing $N_{i,c}$ as an information channel, we can show it is in a special family of channels, named $S$-ary symmetric channel \citep{makur2018comparison}, which satisfies strong data processing inequality. Through this idea, we can prove the following lemma. The details are provided in Section~\ref{sec:pf-lemma-phi'}.
\begin{lemma}\label{lem:phi'_lower_bound}
For $t \in (0, T]$, $\varphi_{i,c}(t)$ is differentiable in $t$ and it holds that
    \begin{align*}
        -\varphi_{i,c}'(t) \geq \varphi_{i,c}(t).
    \end{align*}
\end{lemma}
Consequently, Lemma~\ref{lem:phi'_lower_bound} leads to 
\begin{align*}
    -\varphi'(t) = -\frac{1}{S}\sum_{i\in[d]}\sum_{c \in [S]} \varphi'_{ic}(t) \ge \frac{1}{S}\sum_{i\in[d]}\sum_{c \in [S]} \varphi_{ic}(t) = \varphi(t).
\end{align*}
Recall the log-Sobolev inequality in Lemma~\ref{lem:uniform_log_sobolev}. We have, for any target distribution \(q_0 \in \cP^{\gamma}(\cX)\) and \(t \in (0, 1)\),
\begin{align*}
    -\varphi'(t) \geq \varphi(t) \geq \KL(q_t \Vert p_0) \ge \KL(q_1 \Vert p_0) \ge \gamma d\log(S).
\end{align*}

Equipped with the above relation, we are ready to control the second term in Eqn.~\eqref{eq:uniform_lower_mid}. By the fundamental theorem of calculus, we obtain 
    \begin{align*}
        \int_{t_{k}}^{t_{k+1}} \left(\varphi(T-t) - \varphi(T-t_k)\right)\d t &= \int_{T-t_{k+1}}^{T-t_k}\int_{T-t}^{T-t_k} -\varphi'(\tau) \d \tau \d t\\
        &\gtrsim \int_{T-t_{k+1}}^{T-t_k} (t - t_k) \gamma d\log(S) \d t = \frac{1}{2}(t_{k+1} - t_k)^2 \gamma d \log(S).
    \end{align*}
    Choose $M$ such that $T - t_M \in [\frac{1}{2}, 1]$. Such $M$ exists due to the fact that $T > 1$ and $\max_k\{t_{k+1} - t_k\} \leq \frac{1}{2}$. It holds in this case that
    \begin{align}
        O(1) &= \sum_{k=0}^{N-1} \int_{t_{k}}^{t_{k+1}} \left(\varphi(T-t) - \varphi(T-t_k)\right)\d t\notag\\
        &\gtrsim \sum_{k=M}^{N-1} \int_{t_{k}}^{t_{k+1}} \left(\varphi(T-t) - \varphi(T-t_k)\right)\d t\notag\\
        &\gtrsim \sum_{k=M}^{N-1} (t_{k+1} - t_k)^2 \gamma d \log (S).\label{eq:lower-O(1)-phi}
    \end{align}
    By Cauchy-Schwarz inequality, it is direct to show that
    \begin{align}\label{eq:lower-cs}
        \sum_{k=M}^{N-1} (t_{k+1} - t_k)^2 \ge \frac{1}{N-M} \left(\sum_{k=M}^{N-1} (t_{k+1} - t_k)\right)^2 = \frac{(T - t_M - \delta)^2}{N - M} \gtrsim \frac{1}{N - M},
    \end{align}
    where in the last inequality, we use the fact that $T - t_M \ge \frac{1}{2}$ and $\delta \ll 1$.
    Plugging Eqn.~\eqref{eq:lower-cs} into Eqn.~\eqref{eq:lower-O(1)-phi} leads to
    \begin{align*}
        N \ge N - M =\Omega(\gamma d \log(S))  = \widetilde{\Omega}(d).
    \end{align*}
    \qed

\subsection{Efficient sampling for high-entropy distributions}
\label{sec:uniform-adaptive}

In the discussion following Theorem~\ref{thm:uniform_lower}, we pointed out that the $\tau$-leaping algorithm can attain sublinear iteration complexity in $d$ for the uniform noising process when the target distribution is close to the uniform distribution on $\cX$. 
We now state this result formally.

\begin{theorem}\label{thm:uniform_adaptive}
    Let $q_0 \in \cP(\cX)$ denote the data distribution. Choose time points \(0 = t_0 < t_1 < \ldots < t_N = T-\delta\) with exponential-then-constant step size schedule, i.e., \(t_{k+1} - t_k \leq \kappa \min(1, T - t_{k+1})\) for \(k = 0, \ldots, N-2\). Suppose $0 < \kappa < 0.9$. Then,
    \begin{align*}
        \KL(q_{T-\delta} \,\|\, p_{\mathrm{output}}) \lesssim \escore + \left(e^{-T} + \kappa  \log(\delta^{-1})\right) \cdot \KL(q_0 \,\|\, \mathrm{Unif}(\cX)).
    \end{align*}
\end{theorem}
Theorem~\ref{thm:uniform_adaptive} reveals that, with an exponential-then-constant schedule and early stopping time $\delta$, the error upper bound depends only on the initial KL divergence $\KL(q_0 \,\|\, \mathrm{Unif}(\cX))$, which can potentially be small if $q_0$ is close to the forward limit distribution $\mathrm{Unif}(\cX)$. 

To be more concrete, we can choose $T = \log(\KL(q_0 \,\|\, \mathrm{Unif}(\cX))/\varepsilon)$, $\delta^{-1} = \mathrm{poly}(d)$ and $\kappa = e^{-T}/\log(d)$ to achieve 
$$\KL(q_{T-\delta} \,\|\, p_{\mathrm{output}}) \lesssim \escore + \varepsilon,$$
with iteration complexity
\begin{align*}
    N = \widetilde{\Theta}\left(\frac{\KL(q_0 \,\|\, \mathrm{Unif}(\cX))}{\varepsilon}\right).
\end{align*}
In particular, this bound is sublinear in $d$ when $\KL(q_0 \,\|\, \mathrm{Unif}(\cX)) = o(d)$.

\bigskip
\noindent
\textbf{Proof of Theorem~\ref{thm:uniform_adaptive}.} \quad
    The proof proceeds along the same lines as the proof of Theorem~\ref{thm:uniform}.

    Write $p_0 = \mathrm{Unif}(\cX)$ as the initial distribution of the sampling process. Following the proof of Eqn.~\eqref{eq:uniform_kl_final}, we bound 
    \begin{align}\label{eq:uniform_delta_kl_final}
        \KL(q_{T-\delta} \,\|\, p_{T-\delta}) &\leq \escore + e^{-T}\KL(q_0 \,\|\, p_0) + \sum_{k=0}^{N-1} \int_{t_k}^{t_{k+1}} \big(\varphi(T-t) - \varphi(T-t_k)\big) \,\d t,
    \end{align}
    where as shown in Lemma~\ref{lem:dkl_dt_phi}, $\varphi(t) = -\partial_t \KL(q_t \,\|\, p_0) \geq 0$. 
    By Lemma~\ref{lem:phi'_lower_bound}, $\varphi(t)$ is a non-increasing function of $t \in (0, T]$, which leads to 
    \begin{align}\label{eq:uniform_phi_adaptive}
        \varphi(t) \le \frac{1}{t} \int_{0}^t \partial_s \KL(q_s \,\|\, p_0) \d s = \frac{\KL(q_0 \, \|\, p_0)}{t}.
    \end{align}
    Without loss of generality, Choose $M$ such that $1 \le M \le N-1$ such that $T - t_M = 1$. For $1\le k < M$, $t_{k+1} - t_k = t_k - t_{k-1} = \kappa$.
    With Eqn.~\eqref{eq:uniform_phi_adaptive}, it can be seen that 
    \begin{align*}
        &\sum_{k=0}^{N-1} \int_{t_k}^{t_{k+1}} \big(\varphi(T-t) - \varphi(T-t_k)\big) \,\d t\\
        &\leq \sum_{k=0}^{N-1} ({t_{k+1}} - t_k) \big(\varphi(T-t_{k+1}) - \varphi(T-t_k)\big) \,\d t\\
        &= (t_N - t_{N-1})\varphi(T-t_{N}) + \sum_{k=1}^{N - 1} \big((t_k - t_{k-1}) + (t_k - t_{k+1})\big) \varphi(T-t_k) - (t_1 - t_0)\varphi(T)\\
        &\overset{\mathrm{(a)}}{\leq} \frac{\kappa\delta}{1-\kappa}\cdot \frac{\KL(q_0 \,\|\, p_0)}{\delta} + \sum_{k=M}^{N - 1}\frac{\kappa^2}{1-\kappa}(T-t_k) \cdot \frac{\KL(q_0 \,\|\, p_0)}{T-t_k}\\
        &\lesssim (N-M) \kappa^2 \KL(q_0 \,\|\, p_0) = \log_{(1-\kappa)}(\delta) \kappa^2 \KL(q_0 \,\|\, p_0) \overset{\mathrm{(b)}}{\leq} \kappa \log(\delta^{-1}) \KL(q_0 \,\|\, p_0),
    \end{align*}
    where we apply Eqn.~\eqref{eq:uniform_phi_adaptive} in (a) and $\log(1-\kappa) \leq -\kappa$ in (b).
    Putting the above bound and Eqn.~\eqref{eq:uniform_delta_kl_final} together proves the desired result. \qed

\section{Proofs of results in Section~\ref{sec:masking_results}}

\subsection{Proof of~\Cref{thm:masking_main}}
\label{sec:proof-masking}

\noindent
For \(t \in \{t_0, \ldots, t_N\}\), let \(p_t\) denote the marginal distribution of \(x_{t_k}\) in~\Cref{alg:modified_ttl}.
Using the data-processing inequality  \(\KL(\back q_T 
\| p_T) \leq \KL(\back q_{t_0,\ldots,t_N}\|p_{t_0,\ldots,t_N})\) and~\Cref{lem:chain_kl}, we decompose the KL divergence between the target distribution \(q_0 \equiv \back q_T\) and the output distribution \(p_T\) as follows:
\begin{align}
\label{eq:chain-rule-masking1}
            \KL (\back q_{T} \| p_{T}) \leq \KL (\back q_{0} \| p_{0}) + \sum_{k=0}^{N-1} \bE_{x_{t_k} \sim \back q_{t_k}} \left[ \KL \left(\back q_{t_{k+1}|t_k}(\cdot|x_{t_k}) \| p_{t_{k+1}|t_k}(\cdot|x_{t_k})\right) \right].
\end{align}

The first term, \emph{initialization error}, was bounded in~\cite{conforti2025non,liang2025absorb} as follows:
    \begin{align}
    \label{eq:mask_init_error}
        \KL(\back q_0 \| p_0) \leq e^{-T}d(1 + \log S + T) \lesssim e^{-T} d \log S.
    \end{align}

Next, we move on to control the second term. 
The following lemma states that for each \(k\), conditioned on \(x_{t_k}\), one can consider a CTMC on the interval \([t_k, t_{k+1}]\), with marginals \(p_{t_{k+1} \mid t_k}(\cdot \mid x_{t_k})\) at time \(t_{k+1}\). The proof is given in~\Cref{sec:proof_is_ctmc}.
\begin{lemma}
\label{lem:alg_is_ctmc}
Fix \(k = 0, \ldots, N - 1\). Let \(x_{t_k}\) and \(x_{t_{k+1}}\) be as in~\Cref{alg:modified_ttl}.
Let \((y_t)_{t \in [t_k, t_{k+1}]}\) be a CTMC with \(y_{t_k} = x_{t_k}\) and the following rate matrix:
    \begin{align}
    \label{eq:modified_ttl}
    \wh Q_t(a, b) \coloneqq \begin{cases}
    \wh s_{T - t_k}(y_{t_k} \odot_{i} b^{i}, y_{t_k} ) \frac{e^{T - t_k} - 1}{e^{T - t} - 1}\bI\{a^{i} = \mathrm{MASK}\}, &\enspace \text{if } \ham(a, b) = 1, a^{i} \neq b^{i}, \text{ and } y_{t_k}^{i} = \mask, \\
    -\sum_{c \neq a} \wh Q_t(a, c), &\enspace \text{if } a = b, \\
    0, &\enspace \text{otherwise.}\\
    \end{cases}
\end{align}
Then, \(x_{t_{k+1}}\) has the same distribution as \(y_{t_{k+1}}\).
\end{lemma}

    Armed with this result, we rewrite the right hand side of~\Cref{eq:chain-rule-masking1} with marginals \(p_{t \mid t_k}(\cdot \mid x_{t_k})\) of this CTMC:

    \begin{align}
    \KL(\back q_T \| p_T) &\lesssim 
     e^{-T} d \log S + \sum_{k=0}^{N-1} \bE_{x_{t_k} \sim \back q_{t_k}} \left[ \KL \left(\back q_{t_{k+1}|t_k}(\cdot|x_{t_k}) \| p_{t_{k+1}|t_k}(\cdot|x_{t_k})\right) \right] \\
    & = e^{-T} d \log S + \sum_{k=0}^{N - 1} \bE_{x_{t_k} \sim \back q_{t_k}} \left[\int_{t_k}^{t_{k+1}} \frac{\partial}{\partial t} \KL \left(\back q_{t \mid t_k}(\cdot \mid x_{t_k}) \| p_{t\mid t_k}(\cdot \mid x_{t_k})\right) \d t\right]. \label{eq:masking_girsanov}
    \end{align}
    To further control the second term above,  we apply~\Cref{lem:kl_der_ub} with rate matrices specified  in~\Cref{lem:alg_is_ctmc}. We can write 
    \begin{align}
        &\KL(\back q_T \| p_T) \notag\\
        &\lesssim e^{-T} d \log S + \sum_{k=0}^{N - 1}  \int_{t_k}^{t_{k+1}} \bE_{x_{t_k},x_t \sim \back q_{{t_k, t}}} \sum_{y \neq x_t}\left[\wh Q_{t}(x_{t}, y) - \back Q_t(x_t, y) + \back Q_t(x_t, y) \log \left(\frac{\back Q_t(x_t, y)}{\wh Q_{t}(x_{t}, y)}\right)\right] \d t. \label{eq:masking_after_girsanov}
    \end{align}
    Fix \(k \in \{0, \ldots, N - 1\}\) and \(t \in [t_k, t_{k+1})\). Let \(\ell \defn t_k\).
    Invoking Eqn.~\eqref{eq:modified_ttl} further leads to 
    \begin{align}
        &\sum_{y \neq x_t}\left[\wh Q_{t}(x_{t}, y) - \back Q_t(x_t, y) + \back Q_t(x_t, y) \log \left(\frac{\back Q_t(x_t, y)}{\wh Q_{t}(x_{t}, y)}\right)\right] \notag\\
        &\enspace = \sum_{i \in m(x_t)} \sum_{c \in [S]} \left[\wh Q_{t}(x_{t}, x_t \odot_i c) - \back Q_t(x_t, x_t \odot_i c) + \back Q_t(x_t, x_t \odot_i c) \log \left(\frac{\back Q_t(x_t, x_t \odot_i c)}{\wh Q_{t}(x_{t}, x_t \odot_i c)}\right)\right] 
        \notag\\
        &\enspace = 
        \sum_{i \in m(x_t)} \sum_{c \in [S]} \Bigg[\frac{e^{T - \ell} - 1}{e^{T - t} - 1}\wh s_{T - \ell}(x_{\ell} \odot_i c, x_{\ell}) - s_{T - t}(x_t \odot_i c, x_t)\notag\\
        &\hspace{15em} + s_{T - t}(x_t \odot_i c, x_t) \log \left(\frac{s_{T - t}(x_t \odot_i c, x_t)}{\frac{e^{T - \ell} - 1}{e^{T - t} - 1}\wh s_{T - \ell}(x_{\ell} \odot_i c, x_{\ell})}\right)\Bigg] \notag\\
        &\enspace = 
        \sum_{i \in m(x_t)} \sum_{c \in [S]} s_{T - t}(x_t \odot_i c, x_t) \breg\left(\frac{e^{T - \ell} - 1}{e^{T - t} - 1}\wh s_{T - \ell}(x_{\ell} \odot_i c, x_{\ell}), s_{T - t}(x_t \odot_i c, x_t)\right).\label{eq:kl_to_bregman}
    \end{align}
    
    To proceed, we make the observation that the Bregman divergence satisfies the following law of cosines:
    \begin{align*}
        \breg(\alpha, \gamma) = \breg(\alpha, \beta) + \breg(\beta, \gamma) + (\alpha - \beta)\frac{\beta - \gamma}{\beta \gamma}.
    \end{align*}
    We apply this decomposition to each term of~\Cref{eq:kl_to_bregman} with
    \begin{align*}
        \alpha = \frac{e^{T - \ell} - 1}{e^{T - t} - 1} \wh s_{T - \ell}(x_{\ell} \odot_i c, x_{\ell}), \qquad \beta = \frac{e^{T - \ell} - 1}{e^{T - t} - 1} s_{T - \ell}(x_{\ell} \odot_i c, x_{\ell}), \quad \text{and} \quad \gamma = s_{T - t}(x_t \odot_i c, x_t).
    \end{align*}
    In the following, we slightly abuse the notation and write \(x_t \defn (x_t \odot_i c, x_t)\) and \(x_\ell \defn (x_\ell \odot_i c, x_\ell)\) whenever \(i \in m(x_t)\) and \(c \in [S]\) are fixed.
    For fixed \(i, c\), each term in Eqn.~\eqref{eq:kl_to_bregman} can be decomposed as 
\begin{align*}
    s_{T - t}(x_t)\breg\left(\frac{e^{T - \ell} - 1}{e^{T - t} - 1} \wh s_{T - \ell}( x_\ell), s_{T - t}(x_t)\right)
    &= s_{T - t}(x_t) \breg(\wh s_{T - \ell}(x_\ell), s_{T - \ell}(x_\ell)) \\
    &\quad + s_{T - t}(x_t) \breg\left(\frac{e^{T - \ell} - 1}{e^{T - t} - 1} s_{T - \ell}(x_\ell), s_{T - t}(x_t)\right)\\
    &\qquad + \frac{\wh s_{T - \ell}(x_\ell) - s_{T - \ell}(x_\ell )}{s_{T - \ell}(x_\ell)}\left(\frac{e^{T - \ell} - 1}{e^{T - t} - 1} s_{T - \ell}(x_\ell) - s_{T - t}(x_t)\right).
\end{align*}
Note that we simplified the first term as \(\breg (\alpha x, \alpha y) = \breg(x, y)\).
Observing that \(\frac{e^{T - \ell} - 1}{e^{T -t} - 1} s_{T - \ell} \equiv s_{T - t}\) by~\Cref{eq:mask_score}, this can be rearranged as follows:
\begin{align*}
    s_{T - t}(x_t)\breg\left(\frac{e^{T - \ell} - 1}{e^{T - t} - 1} \wh s_{T - \ell}(x_\ell), s_{T - t}(x_t)\right)
    & = \underbrace{\frac{e^{T - \ell} - 1}{e^{T - t} - 1}s_{T - \ell}(x_\ell) \breg(\wh s_{T - \ell}( x_\ell), s_{T - \ell}(x_\ell))}_{\eqqcolon\, T_1} \\
    &\qquad \qquad + \underbrace{\big(s_{T - t}(x_\ell) - s_{T - t}(x_t)\big) \cdot \log\frac{\wh s_{T - \ell}(x_\ell)}{s_{T - \ell}(x_\ell)}}_{\eqqcolon\, T_2}\\
    &\qquad\qquad\qquad\quad + \underbrace{s_{T - t}(x_t) \breg\left( s_{T - t}( x_\ell), s_{T - t}(x_t)\right)}_{\eqqcolon\, T_3}.
\end{align*}
Taking the above collectively with Eqns.~\eqref{eq:kl_to_bregman} and \eqref{eq:masking_after_girsanov} leads to 
    \begin{align*}
        \KL(\back q_T \| p_T)
        &\lesssim e^{-T} d \log S + \sum_{k=0}^{N - 1}  \int_{t_k}^{t_{k+1}} \bE_{x_{t_k},x_t \sim \back q_{{t_k, t}}}
        \sum_{i \in m(x_t)} \sum_{c \in [S]} (T_1 + T_2 + T_3).
    \end{align*}
Now, it suffices to control each term on the right, respectively. 
\begin{itemize}
\item After taking a summation over $i 
\in m(x_t)$ and $c \in [S]$ we connect the first term, \(T_1\), to the score entropy loss. 
To see that, direct calculations show 
\begin{align*}
    & \bE_{x_t, x_\ell \sim \back q_{t, \ell}} \sum_{i \in m(x_t)} \sum_{c \in [S]}  \frac{e^{T - \ell} - 1}{e^{T - t} - 1}s_{T - \ell}(x_\ell \odot_i c, x_\ell) \breg(\wh s_{T - \ell}(x_\ell \odot_i c, x_\ell), s_{T - \ell}(x_\ell \odot_i c, x_\ell)) \\
    & = \bE_{x_\ell \sim \back q_\ell} \sum_{i \in m(x_\ell)} \sum_{c \in [S]} e^{t - \ell} s_{T - \ell}(x_\ell \odot_i c, x_\ell) \breg(\wh s_{T - \ell}(x_\ell \odot_i c, x_\ell), s_{T - \ell}(x_\ell \odot_i c, x_\ell)) \\
    & = e^{t - \ell} \cL_{\mathrm{SE}}(T - \ell, \wh s_{T - \ell}, s_{T - \ell}),
\end{align*}
where we used in the second line that \(\Pr(x_{t}^i = \mask \mid x_{\ell}^i = \mask) = \frac{1 - e^{-(T - t)}}{1 - e^{-(T - \ell)}}\).
Therefore, recalling that \(\ell \defn t_k\),
\begin{align}
        \sum_{k=0}^{N - 1}  \int_{t_k}^{t_{k+1}} e^{t - t_k} \cL_{\mathrm{SE}}(T - t_k, \wh s_{T - t_k}, s_{T - t_k})\d t
        &= \sum_{k=0}^{N - 1} (e^{t_{k+1} - t_k} - 1) \cL_{\mathrm{SE}}(T - t_k, \wh s_{T - t_k}, s_{T - t_k}) \lesssim \escore, \label{eq:masking_approx_error}
    \end{align}
where we used \(\Delta = O(1)\) and~\Cref{asm:escore} in the last inequality. 

\item To control the second term, $T_2$, the next lemma describes a martingale property of the score function. The proof is given in~\Cref{sec:proof_mask_mark}.
\begin{lemma}
\label{lem:mask_mart}
    Consider the masking noising process and let \(0 \leq \ell < t < T\). Then, for any \(c \in \cV\) and \(i \in m(x_\ell)\),
    \begin{align*}
        \bE_{x_t \sim \back q_{t \mid \ell}(\cdot \mid x_\ell)} \left[\left(s_{T - t}(x_\ell \odot_i c, x_\ell) - s_{T - t}(x_t \odot_i c, x_t)\right)\bI\{i \in m(x_t)\}\right] = 0.
    \end{align*}
\end{lemma}
In view of~\Cref{lem:mask_mart}, we conclude that the second term, \(T_2\), contributes zero after conditioning on \(x_{t_k}\):
\begin{align*}
    \sum_{i \in [d]} \sum_{c \in [S]} \bE_{x_t \sim \back q_{t \mid t_k}{(\cdot \mid x_{t_k})}} \bI\{i \in m(x_t)\}(s_{T - t}(x_{t_k} \odot_i c, x_{t_k}) - s_{T - t}(x_t \odot_i c, x_t)) = 0.
\end{align*}

\item Lastly, we move on to control the last term, \(T_3\). Towards this goal, we introduce the following lemma, whose proof is provided in~\Cref{sec:proof_control_at_t}.
\begin{lemma}
\label{lem:control_at_t}
Let \(0 \leq \ell < t \leq T\). Then, for \(\cI(t)\) defined in~\Cref{def:etc},
\begin{align}
    &\bE_{x_{\ell}, x_t \sim \back q_{\ell, t}} \sum_{i \in m(x_t)} \sum_{c \in [S]} s_{T - t}(x_t \odot_i c, x_t)\breg \left( s_{T - t}(x_{\ell} \odot_i c, x_{\ell}), s_{T - t}(x_t \odot_i c, x_t)\right) \notag\\
    &\enspace = \int_{\ell}^t e^{t - v} \cI(T - v) \d v.
\end{align}
\end{lemma}

After the summation over \(i \in m(x_t)\) and \(c \in [S]\), we express the contributions of the term \(T_3\) using~\Cref{lem:control_at_t} as follows:
\begin{align}
    &\sum_{k=0}^{N - 1}\int_{t_k}^{t_{k+1}}\bE_{x_{t_k}, x_t \sim \back q_{t_k, t}} \sum_{i \in m(x_t)} \sum_{c \in [S]} s_{T - t}(x_t \odot_i c, x_t)\breg \left(s_{T - t}(x_{t_k} \odot_i c, x_{t_k}), s_{T - t}(x_t \odot_i c, x_t)\right) \d t \notag\\
    &\quad = \sum_{k=0}^{N - 1} \int_{t_k}^{t_{k+1}} \int_{t_k}^t e^{t - v} \cI(T - v) \d v \d t \leq \sum_{k=0}^{N - 1} h_k \int_{T - t_{k+1}}^{T - t_k} \cI(t) \d t,\label{eq:masking_discr_error}
\end{align}
where we used \(\Delta = O(1)\) and non-negativity of conditional mutual information in the last inequality.
\end{itemize}

Collecting Eqns.~(\ref{eq:mask_init_error}),~(\ref{eq:masking_approx_error}), and~(\ref{eq:masking_discr_error}) proves
    \begin{align*}
        \KL(q_0 \,\|\, p_T) \lesssim \escore + e^{-T} d \log S + \sum_{k=0}^{N - 1} h_k \int_{T - t_{k+1}}^{T - t_k} \cI(t) \d t.
    \end{align*}
    \qed

\subsection{Proof of Corollary~\ref{cor:masking}}
\label{sec:proof-masking-cor}
    We upper bound the last term of~\Cref{eq:kl_ub_masking} under uniform and exponential-then-constant step size schedules. 
    First, under the constant step size schedule, the quantity of interest satisfies  
    \begin{align*}
        \sum_{k=0}^{N - 1}h_k \int_{T - t_{k+1}}^{T - t_k} \cI(t) \d t = \frac{T}{N} \int_{0}^T \cI(t) \d t \leq \frac{T}{N} \int_{0}^{\infty} \cI(t) \d t = \frac{T}{N} \cB(q_0),
    \end{align*}
    where the last inequality follows from~\Cref{lem:int_dtc_tc}.
    Therefore, as long as 
    \begin{align*}
        N \geq \frac{T\cB(q_0)}{\varepsilon} = \wt O\left(\frac{\cB(q_0)}{\varepsilon}\right),
    \end{align*}
    \Cref{eq:kl_ub_masking} leads to \(\KL(q_0 \| p_T) \lesssim \escore + \varepsilon\).

    Next, under exponential-then-constant step size schedule, we bound the last term of~\Cref{eq:kl_ub_masking} as follows:
    \begin{align*}
        &\sum_{k=0}^{N - 1}h_k \int_{T - t_{k+1}}^{T - t_k} \cI(t) \d t = \frac{\varepsilon}{d \log(S)} \int_{0}^{\varepsilon / (d \log S)} \cI(t) \d t + \sum_{k = 0}^{N - 2} \int_{T - t_{k+1}}^{T - t_k}(t_{k+1} - t_k) \cI(t) \d t \\
        & \quad \leq \frac{\varepsilon^2}{\log(S)} + \kappa \sum_{k=0}^{N - 2} \int_{T - t_{k+1}}^{T-t_k} \min(1, T - t_{k+1}) \cI(t) \d t \leq \varepsilon + \kappa \int_0^T \min(1, t) \cI(t) \d t \leq \varepsilon + \kappa \cD(q_0).
    \end{align*}
    For \(N > 0\), such step size schedule is possible with 
\(
        \kappa = O\left(\frac{T + \log(\varepsilon^{-1} d \log (S))}{N}\right)\).
    Thus, choosing 
    \begin{align*}
        N \geq \frac{(T + \log (\varepsilon^{-1} d \log(S)))\cD(q_0)}{\varepsilon} = \wt O\left(\frac{\cD(q_0)}{\varepsilon}\right)
    \end{align*}
    gives \(\KL(q_0 \,\|\, p_T) \lesssim \escore + \varepsilon\). \qed

\subsection{\(\tau\)-leaping for masking discrete diffusion}
\label{app:mask_ttl}

In this section, we prove the analogue of~\Cref{thm:masking_main} for the truncated \(\tau\)-leaping algorithm. 
Note that since applying multiple jumps on a single coordinate is ill-defined in masking noising process (where should we transition if the \(\tau\)-leaping algorithm requires two transitions \(\mask \to 1\) at some coordinate?), we analyse the truncated version (\Cref{eq:Qhat_ttl}) instead of the classical \(\tau\)-leaping algorithm. 

\begin{theorem}
\label{thm:masking_ttl}
Let \(q_{\mathrm{data}} = q_0\) be the target distribution on \([S]^d\). Let \(0 < \delta < T\) and \(0 = t_0 < t_1 < \ldots < t_N = T - \delta\), such that \(h_k \defn t_k - t_{k-1} \leq \kappa \min(1, T - t_k)\) for \(k \in [N]\), and \(\kappa = O(1)\).
Let
    \begin{align*}
    p_0 \defn \left(\big(1 - e^{-T}\big) \delta_{\mask} +  \frac{e^{-T}}{S} \sum_{k=1}^S \delta_k\right)^{\otimes d}.
    \end{align*} 
    Under~\Cref{asm:escore}, truncated \(\tau\)-leaping~\Cref{eq:Qhat_ttl} initialized at \(p_0\) 
    produces a sample from \(p_{\mathrm{output}} \defn p_{T-\delta}\), such that
    \begin{align}
    \label{eq:mask_ttl}
        &\KL(q_{\delta} \,\|\, p_{\mathrm{output}}) \lesssim \escore + e^{-T} d \log(S) 
                + 
                \sum_{k=0}^{N - 1} h_{k+1} \int_{T - t_{k+1}}^{T - t_k} \cI(t) \d t
                +
                \kappa^3 N d
                + \kappa C ,
    \end{align}
    where 
    \begin{align*}
        C \defn \sum_{k=0}^{N - 1} (t_{k+1} - t_k) \bE_{x_{t_k} \sim \back q_{t_k}}  \sum_{i \in m(x_{t_k})} \sum_{c \in [S]} s_{T - {t_k}}(x_{t_k} \odot_i c, x_{t_k}) \abs*{\log \frac{\wh s_{T - {t_k}}(x_{t_k} \odot_i c, x_{t_k})}{s_{T - {t_k}}(x_{t_k} \odot_i c, x_{t_k})}}.
    \end{align*}
 Consequently, with exponential-then-constant step size schedule where \(\kappa = O\left(\frac{T + \log (\delta^{-1})}{N}\right)\), for any \(\varepsilon > 0\), for \(T = O(\log (\varepsilon^{-1} d \log S))\) and 
    \begin{align}
        N = \wt O\left(\frac{\cD(q_{\mathrm{data}}) + C}{\varepsilon} + \sqrt{\frac{d}{\varepsilon}}\right),
        \label{eq:masking_ttl_N_bound}
    \end{align}
   it satisfies  
    \begin{align*}
    \KL(q_{\delta}\| p_{\mathrm{output}}) \lesssim \varepsilon_{\mathrm{score}} + \varepsilon.
    \end{align*}
\end{theorem}

Note that the guarantee in~\Cref{eq:mask_ttl} closely parallels~\Cref{eq:kl_ub_masking} in Theorem~\ref{thm:masking_main} for Algorithm~\ref{alg:modified_ttl}. In particular, two additional terms arise in the analysis of the truncated $\tau$-leaping algorithm. 
We expect the constant \(C\) to be small and remark that it also arises in the analysis of~\cite{conforti2025non}, as \(C_2^M\) in Theorem 3.2.1, in the form of the maximum rather than the average. Under the assumption of (one-sided) boundedness \(\wh s_{T - t_k} \geq M^{-1}\), the constant \(C\) can be upper bounded via the Cauchy-Schwarz inequality, as the next corollary shows.
\begin{corollary}
    Consider the setting of~\Cref{thm:masking_ttl}. If, additionally, there exists \(M > 0\) such that for all \(k \in \{0, \ldots, N - 1\}\), \(x \in ([S] \cup \{\mask\})^d\), \(i \in m(x)\), and \(c \in [S]\) it holds \(\log \wh s_{T - t_k}(x \odot_i c, x) \geq - \log M\), it is sufficient for  
    \begin{align*}
        N = \wt O\left(\frac{\cD(q_{\mathrm{data}}) + \sqrt{\varepsilon_{\mathrm{score}} d \log M }}{\varepsilon} + \sqrt{\frac{d}{\varepsilon}} \right)
    \end{align*}
    to ensure 
    \begin{align*}
        \KL(q_{\delta}\| p_{T - \delta}) \lesssim \varepsilon_{\mathrm{score}} + \varepsilon.
    \end{align*}
\end{corollary}

\begin{proof}
    For fixed \(k \in \{0, \ldots, N\}\) and \(\ell \defn t_k\), by the Cauchy-Schwarz inequality, it satisfies 
    \begin{align*}
        &\bE_{x_\ell \sim \back q_\ell}  \sum_{i \in m(x_\ell)} \sum_{c \in [S]} s_{T - \ell}(x_\ell \odot_i c, x_\ell) \abs*{\log \frac{\wh s_{T - \ell}(x_\ell \odot_i c, x_\ell)}{s_{T - \ell}(x_\ell \odot_i c, x_\ell)}} \\
        &\quad \leq \left(\bE_{x_\ell \sim \back q_\ell}  \sum_{i \in m(x_\ell)} \sum_{c \in [S]} s_{T - \ell}(x_\ell \odot_i c, x_\ell)\right)^{1/2}\left(\bE_{x_\ell \sim \back q_\ell}  \sum_{i \in m(x_\ell)} \sum_{c \in [S]} s_{T - \ell}(x_\ell \odot_i c, x_\ell) \left(\log \frac{\wh s_{T - \ell}(x_\ell \odot_i c, x_\ell)}{s_{T - \ell}(x_\ell \odot_i c, x_\ell)}\right)^2\right)^{1/2} \\
        &\quad \leq \sqrt{d} \left(\bE_{x_\ell \sim \back q_\ell}  \sum_{i \in m(x_\ell)} \sum_{c \in [S]} s_{T - \ell}(x_\ell \odot_i c, x_\ell) \left(\log \frac{\wh s_{T - \ell}(x_\ell \odot_i c, x_\ell)}{s_{T - \ell}(x_\ell \odot_i c, x_\ell)}\right)^2\right)^{1/2}.
    \end{align*}
    Next, using \(z - 1 - \log z \gtrsim \frac{(\log z)^2}{B}\) for \(\log z \geq -B\), together with
    \begin{align*}
        \log \frac{\wh s_{T - \ell}(x_\ell \odot_i c, x_\ell)}{s_{T - \ell}(x_\ell \odot_i c, x_\ell)} \geq - \log M + \log \left(e^{T - \ell} - 1\right) \geq -\log M + \log (T - \ell) \geq -\log (M \delta^{-1}),
    \end{align*}
    we upper bound
    \begin{align*}
        &\bE_{x_\ell \sim \back q_\ell}  \sum_{i \in m(x_\ell)} \sum_{c \in [S]} s_{T - \ell}(x_\ell \odot_i c, x_\ell) \left(\log \frac{\wh s_{T - \ell}(x_\ell \odot_i c, x_\ell)}{s_{T - \ell}(x_\ell \odot_i c, x_\ell)}\right)^2 \\
        &\quad \leq \log (M \delta ^{-1}) \times \bE_{x_\ell \sim \back q_\ell}  \sum_{i \in m(x_\ell)} \sum_{c \in [S]} s_{T - \ell}(x_\ell \odot_i c, x_\ell) \breg(\wh s_{T - \ell}(x_\ell \odot_i c, x_\ell), s_{T - \ell}(x_\ell \odot_i c, x_\ell)).
    \end{align*}
    As a result, $C$ can be controlled as 
    \begin{align}
        C &\defn \sum_{k=0}^{N-1} (t_{k+1} - t_k) \bE_{x_{t_k} \sim \back q_{t_k}}  \sum_{i \in m(x_{t_k})} \sum_{c \in [S]} s_{T - {t_k}}(x_{t_k} \odot_i c, x_{t_k}) \abs*{\log \frac{\wh s_{T - {t_k}}(x_{t_k} \odot_i c, x_{t_k})}{s_{T - {t_k}}(x_{t_k} \odot_i c, x_{t_k})}} \notag \\
        & \leq \sqrt{d \log (M \delta^{-1})} \times \notag \\
        & \qquad  \sum_{k=0}^{N-1} (t_{k+1} - t_k) \left(\bE_{x_{t_k} \sim \back q_{t_k}}  \sum_{i \in m(x_{t_k})} \sum_{c \in [S]} s_{T - {t_k}}(x_{t_k} \odot_i c, x_{t_k}) \breg(\wh s_{T - {t_k}}(x_{t_k} \odot_i c, x_{t_k}), s_{T - {t_k}}(x_{t_k} \odot_i c, x_{t_k}))\right)^{1/2} \notag \\
        & \overset{(a)}{\leq} \sqrt{\kappa N  d \log (M \delta^{-1})}
        \times 
        \notag \\
        &\qquad \left(\sum_{k=0}^{N- 1} (t_{k+1} - t_k)\bE_{x_{t_k} \sim \back q_{t_k}}  \sum_{i \in m(x_{t_k})} \sum_{c \in [S]} s_{T - {t_k}}(x_{t_k} \odot_i c, x_{t_k}) \breg(\wh s_{T - {t_k}}(x_{t_k} \odot_i c, x_{t_k}), s_{T - {t_k}}(x_{t_k} \odot_i c, x_{t_k})) \right)^{1/2} \notag \\
        & \leq \sqrt{\kappa N d \log (M \delta^{-1}) \varepsilon_{\mathrm{score}}},
        \label{eq:masking_ttl_C}
    \end{align}
    where in \((a)\) we used \(t_{k+1} - t_k \leq \kappa\) together with the Cauchy-Schwarz inequality.
    Combining the bound of~\Cref{eq:masking_ttl_C} with~\Cref{eq:masking_ttl_N_bound} and \(\kappa = \wt O(1 / N)\) completes the proof.
\end{proof}
We emphasize that, in contrast to~\Cref{thm:masking_main}, in~\Cref{thm:masking_ttl} we require early stopping for some \(\delta > 0\), which in turn leads to the exponential-then-constant step size schedule. 
We now elaborate on the difference between Theorem~\ref{thm:masking_ttl} and Theorem~\ref{thm:masking_main}, and provide some intuition for the appearance of the two additional  terms in Eqn.~\eqref{eq:mask_ttl}.
\begin{remark}
\label{rem:masking_ttl}
To obtain an accurate sampler, it is natural to require that \(\wh Q_t \approx \back Q_t\) uniformly for all \(t \in [0, T]\). The main challenge is that we only have access to score estimates at discrete time points.
In the truncated \(\tau\)-leaping algorithm analyzed in this section, this results in approximating 
\begin{align*}
    \wh s_{T - t} \defn \wh s_{T - t_k} \approx s_{T - t}.
\end{align*}
Informally, we establish this by showing
\begin{align}
\label{eqn:brahms}
    \wh s_{T - t_k} \approx s_{T - t_k} \approx s_{T - t},
\end{align}
where the first approximation is ensured by~\Cref{asm:escore} and the second results from the properties of the score function for the masking noising process, requiring the step size \(t_{k+1} - t_k\) to be small. 

In contrast, for Algorithm~\ref{alg:modified_ttl} considered in  \Cref{thm:masking_main}, the condition \(\wh Q_t \approx \back Q_t\) translates to 
\begin{align*}
    \wh s_{T - t} \defn \frac{e^{T - t_k} - 1}{e^{T - t} - 1} \wh s_{T - t_k} \approx s_{T - t}.
\end{align*}
In view of Proposition~\ref{prop:score_function}, the above condition is equivalent to 
\begin{align*}
    \wh s_{T - t} \defn \frac{e^{T - t_k} - 1}{e^{T - t} - 1} \wh s_{T - t_k} \approx \frac{e^{T - t_k} - 1}{e^{T - t} - 1} s_{T - t_k} = s_{T - t},
\end{align*}
which is guaranteed by \Cref{asm:escore}. 
Notably, this simple rescaling eliminates the need for the second approximation step that is required in the truncated $\tau$-leaping analysis. This distinction explains why \Cref{thm:masking_ttl} contains two additional error terms and necessitates early stopping, in contrast to the cleaner guarantee obtained in \Cref{thm:masking_main}.
\end{remark}

\bigskip
\noindent
\textbf{Proof of~\Cref{thm:masking_ttl}.} \enspace
    The proof follows the proof of~\Cref{thm:masking_main} closely with several additional steps. We begin with \Cref{eq:masking_after_girsanov}
        \begin{align}
        \label{eq:mask_ttl_with_init}
        &\KL(\back q_T \| p_T)\notag\\
        &\lesssim e^{-T} d \log S + \!\!\sum_{k=0}^{N - 1}  \int_{t_k}^{t_{k+1}} \!\!\bE_{x_{t_k},x_t \sim \back q_{{t_k, t}}} \sum_{y \neq x_t}\left[\wh Q_{t}(x_{t}, y) - \back Q_t(x_t, y) + \back Q_t(x_t, y) \log \left(\frac{\back Q_t(x_t, y)}{\wh Q_{t}(x_{t}, y)}\right)\right] \d t.
    \end{align}
    Next, since for this sampler, the rate matrices \(\wh Q_t\) are given by the following:
        \begin{align*}
    \wh Q_t(x, y) \coloneqq \begin{cases}
    \wh s_{T - t_k}(x_{t_k} \odot_{i} y^{i}, x_{t_k} ) \bI\{x^{i} = \mathrm{MASK}\}, &\quad \text{if } \ham(x, y) = 1, x^{i} \neq y^{i}, \text{ and } x_{t_k}^{i} = \mask, \\
    -\sum_{z \neq x} \wh Q_t(x, z), &\quad \text{if } y = x, \\
    0, &\quad \text{otherwise,}\\
    \end{cases}
\end{align*}
we can therefore bound 
\begin{align*}
        &\sum_{y \neq x_t}\left[\wh Q_{t}(x_{t}, y) - \back Q_t(x_t, y) + \back Q_t(x_t, y) \log \left(\frac{\back Q_t(x_t, y)}{\wh Q_{t}(x_{t}, y)}\right)\right] \\
        &\enspace = \sum_{i \in m(x_t)} \sum_{c \in [S]} \left[\wh Q_{t}(x_{t}, x_t \odot_i c) - \back Q_t(x_t, x_t \odot_i c) + \back Q_t(x_t, x_t \odot_i c) \log \left(\frac{\back Q_t(x_t, x_t \odot_i c)}{\wh Q_{t}(x_{t}, x_t \odot_i c)}\right)\right] 
        \\
        &\enspace = 
        \sum_{i \in m(x_t)} \sum_{c \in [S]} \left[\wh s_{T - \ell}(x_{\ell} \odot_i c, x_{\ell}) - s_{T - t}(x_t \odot_i c, x_t)  + s_{T - t}(x_t \odot_i c, x_t) \log \left(\frac{s_{T - t}(x_t \odot_i c, x_t)}{\wh s_{T - \ell}(x_{\ell} \odot_i c, x_{\ell})}\right)\right] \\
        &\enspace = 
        \sum_{i \in m(x_t)} \sum_{c \in [S]} s_{T - t}(x_t \odot_i c, x_t) \breg\left(\wh s_{T - \ell}(x_{\ell} \odot_i c, x_{\ell}), s_{T - t}(x_t \odot_i c, x_t)\right).
    \end{align*}

To proceed, we again apply the law of cosines \(\breg(\alpha, \gamma) = \breg(\alpha, \beta) + \breg(\beta, \gamma) + (\alpha - \beta)\frac{\beta - \gamma}{\beta \gamma}\) with
    \begin{align*}
        \alpha = \wh s_{T - \ell}(x_{\ell} \odot_i c, x_{\ell}), \qquad \beta = s_{T - \ell}(x_{\ell} \odot_i c, x_{\ell}), \quad \text{and} \quad \gamma = s_{T - t}(x_t \odot_i c, x_t).
    \end{align*}
    In the following, we slightly abuse the notation and write \(x_t \defn (x_t \odot_i c, x_t)\) and \(x_\ell \defn (x_\ell \odot_i c, x_\ell)\) whenever \(i \in m(x_t)\) and \(c \in [S]\) are fixed.
   As a result, for fixed \(i, c\), one has 

\begin{align*}
    &s_{T - t}(x_t)\breg\left(\wh s_{T - \ell}(x_\ell), s_{T - t}(x_t)\right)\\
    &= s_{T - t}(x_t) \breg(\wh s_{T - \ell}(x_\ell), s_{T - \ell}(x_\ell)) + s_{T - t}(x_t) \breg\left( s_{T - \ell}(x_\ell), s_{T - t}(x_t)\right) \\
    &\qquad + \frac{\wh s_{T - \ell}(x_\ell) - s_{T - \ell}(x_\ell)}{s_{T - \ell}(x_\ell)}\left( s_{T - \ell}(x_\ell) - s_{T - t}(x_t)\right).
\end{align*}
This can be rearranged as follows:
\begin{align*}
    s_{T - t}(x_t)\breg\left( \wh s_{T - \ell}(x_\ell), s_{T - t}(x_t)\right) &= \underbrace{s_{T - \ell}(x_\ell) \breg(\wh s_{T - \ell}(x_\ell), s_{T - \ell}(x_\ell))}_{\eqqcolon\, T_1}\\
    &\quad + \underbrace{s_{T - t}(x_t) \breg\left(s_{T - \ell}(x_\ell), s_{T - t}(x_t)\right)}_{\eqqcolon\, T_2} \\
    &\qquad + \underbrace{\left( s_{T - \ell}(x_\ell) - s_{T - t}(x_t)\right) \log\frac{\wh s_{T - \ell}(x_\ell)}{s_{T - \ell}(x_\ell)}}_{\eqqcolon\, T_3}.
\end{align*}

It is therefore sufficient to control each term separately. 
Similar to the proof of Theorem~\ref{thm:masking_main}, the first term, \(T_1\), after the summation is upper bounded by the score entropy loss
\begin{align*}
     &\bE_{x_t, x_\ell \sim \back q_{t, \ell}} \sum_{i \in m(x_t)} \sum_{c \in [S]}  s_{T - \ell}(x_\ell \odot_i c, x_\ell) \breg(\wh s_{T - \ell}(x_\ell \odot_i c, x_\ell), s_{T - \ell}(x_\ell \odot_i c, x_\ell)) \\
     &\quad \leq \bE_{x_\ell \sim \back q_\ell} \sum_{i \in m(x_\ell)} \sum_{c \in [S]}  s_{T - \ell}(x_\ell \odot_i c, x_\ell) \breg(\wh s_{T - \ell}(x_\ell \odot_i c, x_\ell), s_{T - \ell}(x_\ell \odot_i c, x_\ell)) \\
    &\quad  = \cL_{\mathrm{SE}}(T - \ell, \wh s_{T - \ell}, s_{T - \ell}),
\end{align*}
and, by~\Cref{asm:escore},
\begin{align}
    \label{eq:mask_ttl_approx}
        \sum_{k=0}^{N - 1}  \int_{t_k}^{t_{k+1}} \cL_{\mathrm{SE}}(T - \ell, \wh s_{T - \ell}, s_{T - \ell})\d t \leq \escore.
\end{align}
Now, we turn to controlling the terms \(T_2\) and \(T_3\) which require a different treatment than in the proof of~\Cref{thm:masking_main}.

For the term \(T_2\), we again apply the law of cosines with
\begin{align*}
    \alpha = s_{T - \ell}(x_\ell),\qquad \beta = s_{T - t}(x_\ell), \quad \text{and} \quad \gamma = s_{T - t}(x_t).
\end{align*}
This leads to the following decomposition
\begin{align*}
    \begin{aligned}
        s_{T - t}(x_t)  \breg(s_{T - \ell}(x_\ell), s_{T - t}(x_t))   &= \underbrace{s_{T - t}(x_t)  \breg(s_{T - \ell}(x_\ell), s_{T - t}(x_\ell))}_{\eqqcolon\, T_{21}}\\
        &\quad + \underbrace{s_{T - t}(x_t)  \breg(s_{T - t}(x_\ell), s_{T - t}(x_t))}_{\eqqcolon\, T_{22}} \\
        &\qquad + \underbrace{(s_{T - t}(x_\ell) - s_{T - t}(x_t))\frac{s_{T - \ell}(x_\ell) - s_{T - t}(x_\ell)}{s_{T -t}(x_\ell)}}_{\eqqcolon\, T_{23}}.
    \end{aligned}
\end{align*}

\begin{itemize}
\item 
For \(T_{21}\), using~\Cref{eq:mask_score}, observe that
\begin{align*}
        \breg(s_{T - \ell}(x_\ell), s_{T - t}(x_\ell)) &= \frac{e^{T - t} - 1}{e^{T - \ell} - 1} - 1 - \log \frac{e^{T - t} - 1}{e^{T - \ell} - 1} \leq \frac{(e^{T - \ell} - e^{T - t})^2}{2(e^{T - \ell} - 1)(e^{T - t} - 1)} \lesssim \kappa^2,
\end{align*}
where \(\kappa\) is a parameter of the step size schedule: \(t_{k+1} - t_k \leq \kappa \min(1, T - t_{k+1})\). 
The total contribution of terms \(T_{21}\) is:
\begin{align}
    &\sum_{k=0}^{N - 1} \int_{t_k}^{t_{k+1}} \!\!\bE_{x_t, x_{t_k} \sim \back q_{t, t_k}} \!\!\sum_{i \in m(x_t)} \sum_{c \in [S]}\! s_{T - t}(x_t \odot_i c, x_t) \breg(s_{T - t_k}(x_{t_k} \odot_i c, x_{t_k}), s_{T - t}(x_{t_k} \odot_i c, x_{t_k})) \d t \notag\\
    &\quad \lesssim \kappa^3 \sum_{k=0}^{N - 1} \bE_{x_t \sim \back q_t} \sum_{i \in m(x_t)} \sum_{c \in [S]} s_{T - t}(x_t \odot_i c, x_t) \notag\\
    &\quad = \kappa^3 \sum_{k=0}^{N - 1} \bE_{x_t \sim \back q_t} \sum_{i \in m(x_t)} \frac{ \sum_{c \in [S]} q_t(x_t \odot_i c)}{q_t(x_t)} \leq \kappa^3 N d, \label{eq:mask_ttl_disc_dn2}
\end{align}
 as \(\sum_{c \in [S]} q_t(x_t \odot_i c) = q_t(x_t)\). 

\item The term \(T_{22}\) is identical to the term \(T_2\) from the proof of~\Cref{thm:masking_main}, thus we use~\Cref{lem:control_at_t} and obtain that after summation of over \(i \in m(x_t)\) and \(c \in [S]\):
\begin{align}
    &\sum_{k=0}^{N - 1}\int_{t_k}^{t_{k+1}}\bE_{x_{t_k}, x_t \sim \back q_{t_k, t}} \sum_{i \in m(x_t)} \sum_{c \in [S]} s_{T - t}(x_t \odot_i c, x_t)\breg \left(s_{T - t}(x_{t_k} \odot_i c, x_{t_k}), s_{T - t}(x_t \odot_i c, x_t)\right) \d t. \notag\\
    &\quad = \sum_{k=0}^{N - 1} \int_{t_k}^{t_{k+1}} \int_{t_k}^t e^{t - v} \cI(T - v) \d v \d t \lesssim \sum_{k=0}^{N - 1} h_{k+1} \int_{T - t_{k+1}}^{T - t_k} \cI(t) \d t.
\end{align}
Here, we invoke the assumption \(\kappa = O(1)\) and the non-negativity of conditional mutual information in the last inequality.

\item To control \(T_{23}\), observe that by~\Cref{eq:mask_score}, the score function satisfies the relation
\begin{align*}
    \frac{s_{T - \ell}(x_\ell) - s_{T - t}(x_\ell)}{s_{T - t}(x_\ell)} = \frac{e^{T - t} - e^{T - \ell}}{e^{T - \ell} - 1},
\end{align*}
and importantly it does not depend on \(x_\ell\).
This implies that upon summation over \(i \in m(x_t)\) and \(c \in [S]\), the term \(T_{23}\) contributes zero, i.e.,
\begin{align*}
    &\sum_{i \in m(x_t)}\sum_{c \in [S]} (s_{T - t}(x_\ell \odot_i c, x_\ell) - s_{T - t}(x_t \odot_i c, x_t)) \frac{e^{T - t} - e^{T - \ell}}{e^{T - \ell} - 1} \\
    &\quad = \frac{e^{T - t} - e^{T - \ell}}{e^{T - \ell} - 1} \sum_{i\in m(x_t)} \left(\frac{\sum_{c \in [S]} q_t(x_\ell \odot_i c)}{q_t(x_\ell)} - \frac{\sum_{c \in [S]} q_t(x_t \odot_i c)}{q_t(x_t)}\right) \\
    &\quad = \frac{e^{T - t} - e^{T - \ell}}{e^{T - \ell} - 1} \sum_{i\in m(x_t)} \left(1 - 1\right)
    \\&\quad = 0.
\end{align*}
\end{itemize}
Putting pieces together, we can conclude
\begin{align}
    \sum_{k=0}^{N - 1}  \int_{t_k}^{t_{k+1}} \!\!\bE_{x_{t_k},x_t \sim \back q_{{t_k, t}}}  \sum_{i \in m(x_t)} \sum_{c \in [S]} T_2
    \leq 
    \kappa^3 N d + \sum_{k=0}^{N - 1} h_{k+1} \int_{T - t_{k+1}}^{T - t_k} \cI(t) \d t.
\end{align}

It therefore remains to control term \(T_3\). Recall the definition
\begin{align*}
    T_3 \defn \left( s_{T - \ell}(x_\ell) - s_{T - t}(x_t)\right) \log \frac{\wh s_{T - \ell}(x_\ell)}{s_{T - \ell}(x_\ell)}.
\end{align*}
Crucially, unlike in the proof of~\Cref{thm:masking_main}, we no longer have a martingale property for this term. 
However, we can decompose 
\begin{align*}
    T_3 = \left( s_{T - \ell}(x_\ell) - s_{T - t}(x_{\ell})\right) \log \frac{\wh s_{T - \ell}(x_\ell)}{s_{T - \ell}(x_\ell)} + \underbrace{\left( s_{T - t}(x_\ell) - s_{T - t}(x_t)\right) \log \frac{\wh s_{T - \ell}(x_\ell)}{s_{T - \ell}(x_\ell)}}_{\text{contributes zero by~\Cref{lem:mask_mart}}}.
\end{align*}
It remains to bound the first term, which can be written as 
\begin{align*}
    \left( s_{T - \ell}(x_\ell) - s_{T - t}(x_{\ell})\right) \log \frac{\wh s_{T - \ell}(x_\ell)}{s_{T - \ell}(x_\ell)} = \frac{e^{T - t} - e^{T - \ell}}{e^{T - t} - 1} s_{T - \ell}(x_\ell) \log \frac{\wh s_{T - \ell}(x_\ell)}{s_{T - \ell}(x_\ell)}.
\end{align*}
In view of the step size assumption \(t_{k+1} - t_k \leq \kappa \min(1, T - t_{k+1})\), it satisfies 
\begin{align*}
      \abs*{\frac{e^{T - t} - e^{T - \ell}}{e^{T - t} - 1}} \lesssim \kappa.
\end{align*}
The total contribution of terms \(T_3\) can therefore be upper bounded by the following:
\begin{align}
    &\kappa\sum_{k=0}^{N - 1} \int_{t_k}^{t_{k+1}} \bE_{x_\ell \sim \back q_\ell} \sum_{i\in m(x_\ell)} \sum_{c \in [S]} s_{T - \ell}(x_\ell \odot_i c, x_\ell) \abs*{\log \frac{\wh s_{T - \ell}(x_\ell \odot_i c, x_\ell)}{s_{T - \ell}(x_\ell \odot_i c, x_\ell)}} \d t\notag\\ 
    &= \kappa \sum_{k=0}^{N - 1} (t_{k+1} - t_k)\bE_{x_\ell \sim \back q_\ell}  \sum_{i\in m(x_\ell)} \sum_{c \in [S]} s_{T - \ell}(x_\ell \odot_i c, x_\ell) \abs*{\log \frac{\wh s_{T - \ell}(x_\ell \odot_i c, x_\ell)}{s_{T - \ell}(x_\ell \odot_i c, x_\ell)}}.\label{eq:mask_ttl_disc_extra_c}
\end{align}

Collecting~Eqns.~(\ref{eq:mask_ttl_with_init}),~(\ref{eq:mask_ttl_approx}),~(\ref{eq:mask_ttl_disc_dn2}),~(\ref{eq:mask_ttl_disc_inf}), and~(\ref{eq:mask_ttl_disc_extra_c}) proves

\begin{align*}
        &\KL(q_{\delta} \,\|\, p_{T - \delta}) \lesssim \escore + e^{-T} d \log(S)  + \sum_{k=0}^{N - 1}h_{k+1} \int_{T - t_{k+1}}^{T - t_k} \cI(t) \d t
        + \kappa^3 N d + \kappa C, 
\end{align*}
    where 
    \begin{align*}
        C \defn \sum_{k=0}^{N - 1} (t_{k+1} - t_k) \bE_{x_\ell \sim \back q_\ell}  \sum_{i \in m(x_\ell)} \sum_{c \in [S]} s_{T - \ell}(x_\ell \odot_i c, x_\ell) \abs*{\log \frac{\wh s_{T - \ell}(x_\ell \odot_i c, x_\ell)}{s_{T - \ell}(x_\ell \odot_i c, x_\ell)}}.
    \end{align*}

For our step size schedule, as in~\Cref{cor:masking}, we upper bound
\begin{align}
            &\sum_{k=0}^{N - 1}h_{k+1} \int_{T - t_{k+1}}^{T - t_k} \cI(t) \d t  \leq \kappa \sum_{k=0}^{N - 1} \int_{T - t_{k+1}}^{T-t_k} \min(1, T - t_{k+1}) \cI(t) \d t \leq \kappa \int_{\delta}^T \min(1, t) \cI(t) \d t \leq \kappa \cD(q_{\mathrm{data}}).\label{eq:mask_ttl_disc_inf}
\end{align}
Plugging in the choices 
\(\kappa = O\left(\frac{T + \log (\delta^{-1})}{N}\right)\), \(T = O(\log (\varepsilon^{-1} d \log S))\), and 
    \begin{align*}
        N = \wt O\left(\frac{\cD(q_{\mathrm{data}}) + C}{\varepsilon} + \sqrt{\frac{d}{\varepsilon}}\right),
\end{align*}
yields 
$\KL(q_{\delta}\| p_{T - \delta}) \lesssim \varepsilon_{\mathrm{score}} + \varepsilon.$
\qed

%% file: appendix/app_main_lemmas.tex
\section{Proofs of the main lemmas}

\subsection{Characterization of \(\cB(q_{\mathrm{data}})\) and \(\cC(q_{\mathrm{data}})\)}
\label{sec:proof-lemma-int}

The characterization of \(\cB(q_{\mathrm{data}})\) and \(\cC(q_{\mathrm{data}})\) is summarized in the following lemma.
\begin{lemma}
\label{lem:int_dtc_tc}
Consider a masking noising process with initial distribution \(q_0 = q_{\mathrm{data}}\). Let \(\cC(q_{\mathrm{data}})\) and \(\cB(q_{\mathrm{data}})\) be the total correlation and dual total correlation. Then,
\begin{align*}
    \cB(q_{\mathrm{data}}) = \int_0^\infty \cI(t) \d t \quad \text{and} \quad \cC(q_{\mathrm{data}}) = \int_0^\infty (e^t - 1) \cI(t) \d t.
\end{align*}
Consequently, \(\cD(q_{\mathrm{data}}) \leq \min(\cB(q_{\mathrm{data}}), \cC(q_{\mathrm{data}}))\).
\end{lemma}
\begin{proof}
  Let \(p \equiv p(t) =  e^{-t}\) be the probability that at time \(t\) a coordinate is unmasked and \(X(p) \equiv (X_1, \ldots, X_d) \coloneqq (x_t^1, \ldots, x_t^d)\). We also denote \(X_{\cR} \coloneqq x_t^{-(i,j)}\) and \(X_R \coloneqq (X_i)_{i \in R}\) for \(R \subseteq [d]\). With a slight abuse of notation we write \(\cI(p) \defn \cI(t(p))\).
    We have
    \begin{equation*}
        \cI(p) \coloneqq \sum_{i \neq j} \mi(X_i; X_j \mid X_{\cR}) = \sum_{i \neq j} p^2 \sum_{R \subseteq [d] \setminus \{i, j\}} p^{\abs{R}} (1 - p)^{d - 2 - \abs{R}} \mi(X_i; X_j \mid X_R),
    \end{equation*}
    where \(p^2\) appears since for the term \(\mi(X_i; X_j \mid X_\cR)\) to be non-zero, both \(X_i\) and \(X_j\) must be unmasked.
    For \(i \in [d]\), define
    \begin{equation*}
        h_i(p) \coloneqq \sum_{R \subseteq [d] \setminus \{i\}} p^{\abs{R}} (1 - p)^{d - 1 - \abs{R}} \ent(X_i \mid X_R),
    \end{equation*}
    with 
    \begin{equation*} 
    \begin{aligned}
        \frac{\d h_i(p)}{\d p} &= \sum_{R \subseteq [d] \setminus \{i\}} \left[\abs{R} p^{\abs{R} - 1} (1 - p)^{d - 1 - \abs{R}} - (d - 1 - \abs{R}) p^{\abs{R}}(1 - p)^{d - 2 - \abs{R}}\right] \ent(X_i \mid X_R) \\
        & = \sum_{j \neq i} \sum_{R \subseteq [d] \setminus \{i, j\}} p^{\abs{R}}(1 - p)^{d - 2 - \abs{R}} (\ent(X_i \mid X_{R \cup \{j\}}) - \ent(X_i \mid X_R)) \\
        & = -\sum_{j \neq i} \sum_{R \subseteq [d] \setminus \{i, j\}} p^{\abs{R}}(1 - p)^{d - 2 - \abs{R}} \mi(X_i; X_j \mid X_R).
    \end{aligned}
    \end{equation*}
    Therefore, 
    \begin{equation*}
        \cI(p) = \sum_{i=1}^d p^2 \left(- \frac{\d h_i(p)}{\d p}\right).
    \end{equation*}
    Since \(p = e^{-t}\) we have that \(\d t = - \frac{\d p}{p}\) and we can write
    \begin{equation*}
    \begin{aligned}
         \int_{0}^\infty \sum_{i \neq j}\mi(X_i; X_j \mid X_{\cR}) \d t &= \int_{0}^{1} \sum_{i = 1}^d p \left(- \frac{\d h_i(p)}{\d p}\right) \d p = \sum_{i = 1}^d \big(-p h_i (p)\big) \Bigl\vert_{0}^1 + \int_{0}^1 \sum_{i=1}^d  h_i(p) \d p.
    \end{aligned}
    \end{equation*}
    Observe that 
    \begin{equation*}
        \frac{\d \ent(X(p))}{\d p} = \sum_{i = 1}^d h_i(p),
    \end{equation*}
    therefore,
    \begin{equation*}
        \int_{0}^1 \sum_{i=1}^d h_i(p) \d p = \ent(X(1)) - \ent(X(0)) = \ent(x_0).
    \end{equation*}
    Since  \(
        \sum_{i=1}^d h_i(1) = \sum_{i=1}^d \ent(x_0^i \mid x_0^{-i})\), we proved the first part:
        \begin{equation*}
            \int_0^{\infty} \cI(t) \d t = \ent(x_0) - \sum_{i=1}^d \ent(x_0^i \mid x_0^{-i}) = \cB(q_0).
        \end{equation*}
    \noindent
    We proceed similarly for the total correlation:
    \begin{equation*}
    \begin{aligned}
        &\int_{0}^\infty (e^t  -1)\sum_{i \neq j}\cI(t) \d t = \int_0^{1} (1 - p) \sum_{i=1}^d \left(- \frac{\d h_i(p)}{\d p}\right) \d p = - \left(\sum_{i=1}^d (1 - p) h_i(p) \Bigl\vert_0^1 + \int_0^1 \sum_{i=1}^d h_i(p) \d p\right) \\
        &\quad = \sum_{i=1}^d \ent(x_0^i) - \ent(x_0) = \cC(q_0).
    \end{aligned}
    \end{equation*}
 
\end{proof}
\subsection{Proof of Lemma~\ref{lem:uniform_martingale}}
\label{pf:lem-uniform-martingale}

For any $i \in [d]$ and $c \in [S]$, let us define
\begin{align}
\label{eqn:def-f-tmp}
    f_{i, c}(t, x_t) \defn s_{T-t}(x_t \oplus_i c, x_t).
\end{align}
The following analysis holds for all $i \in [d]$ and $c \in [S]$, so we safely omit the index $i, c$ in the following analysis, and write it as $f(t, x_t)$.

Consider the case that the backward process $\{x_t\}_{t \in [0,T]} \sim \{\back q_t\}_{t \in [0,T]}$, which is a Poisson jump process with generator $\back L_t$ such that
\begin{align*}
    \left(\back L_t f\right)(t, x) &= \sum_{y: \ham(y, x) \leq 1} Q_{T-t}(y, x) s_{T-t}(y, x)\big(f(t, y) - f(t, x)\big)\\
    &= \frac{1}{S} \sum_{y: \ham(y, x) = 1} s_{T-t}(y, x)\big(f(t, y) - f(t, x)\big).
\end{align*}
By It\^{o}'s formula for Poisson point process in Lemma~\ref{lem:ito}, $f(t, x_t)$ satisfies the following stochastic differential equation: for $0 \leq \ell \leq t < T$,
\begin{align}\label{eq:uniform_score_ito}
    f(t, x_t) = f(\ell, x_\ell) + \int_{\ell}^t \left[\partial_t f(s, x_{s^-}) + \left(\back L_s f\right)(s, x_{s^-})\right]\d s + M_t,
\end{align}
where $x_{s-} = \lim_{u \to s^-} x_{s}$, which exists for almost everywhere $s \in [0, T)$ under the Lebesgue measure, since we have finite number of jumps almost surely. The compensation process $\{M_u\}_{u \in [\ell, t]}$ in Eqn.~\eqref{eq:uniform_score_ito} is defined as
\begin{align}\label{eq:uniform_mart_def}
    M_u = \sum_{y_s : \ham(y_s, x_s) = 1}\int_\ell^u \big(f(s, y_s) - f(s, x_s)\big)\big(\d N_s^{x_s, y_s} - \lambda_s^{x_s, y_s} \d s\big),
\end{align}
where $N_s^{x, y}$ is the counting process of jumps from $x$ to $y$ and we write the random counting measure as $\d N_s^{x, y}$. We define $\lambda^{x, y}_s = S^{-1}s_{T-t}(y,x)\bI\{x_{s^-} \!\!= x\}$ to be intensity of the process $N_s^{x, y}$. 
Since $x_{s^-} = x_{s}$ almost everywhere $s \in (0,T)$ due to the finite number of jumps for each path almost surely, we can rewrite Eqn.~\eqref{eq:uniform_score_ito} as 
\begin{align}\label{eq:uniform_score_ito_2}
    f(t, x_t) - f(\ell, x_\ell) = \int_{\ell}^t \left[\partial_t f(s, x_{s}) + \left(\back L_s f\right)(s, x_{s})\right]\d s + M_t.
\end{align}

To further simplify the right hand side, we assert that
\begin{align}\label{eq:uniform_score_drift_0}
    \partial_t f(s, x_{s}) + \left(\back L_s f\right)(s, x_{s}) = 0.
\end{align}
In order to see this, first, recall the definition~\eqref{eqn:def-f-tmp} and direct calculations give  
\begin{align*}
    &\partial_t f(s, x_{s}) + \left(\back L_s f\right)(s, x_{s})\\
    &= \frac{\partial}{\partial s} \left(\frac{q_{T-s}(x_s \oplus_i c)}{q_{T-s}(x_s)}\right)\\
    &\qquad\qquad + \frac{1}{S}\sum_{i' \in [d]}\sum_{c' \in [S]} s_{T-s}(x_s \oplus_{i'} c', x_s)\Big(s_{T-s}(x_s \oplus_{i'} c' \oplus_i c, x_s \oplus_{i'} c') - s_{T-s}(x_s \oplus_i c, x_s)\Big)\\
    &\overset{\mathrm{(a)}}{=} \frac{1}{S} \sum_{i' \in [d]}\sum_{c' \in [S]} s_{T-s}(x_s \oplus_i c, x_s) \Big(s_{T-s}(x_s \oplus_{i'} c', x_s) - s_{T-s}(x_s \oplus_i c \oplus_{i'} c', x_s \oplus_i c)  \Big)\\
    &\qquad \qquad + \frac{1}{S}\sum_{i' \in [d]}\sum_{c' \in [S]} s_{T-s}(x_s \oplus_{i'} c', x_s)\Big(s_{T-s}(x_s \oplus_{i'} c' \oplus_i c, x_s \oplus_{i'} c') - s_{T-s}(x_s \oplus_i c, x_s)\Big)\\
    &= \frac{1}{S} \sum_{i' \in [d]}\sum_{c' \in [S]} \Big(s_{T-s}(x_s \oplus_i c, x_s) s_{T-s}(x_s \oplus_{i'} c', x_s) - s_{T-s}(x_s \oplus_{i'} c', x_s) s_{T-s}(x_s \oplus_i c, x_s)\Big)\\
    &\qquad + \frac{1}{S} \sum_{i' \in [d]}\sum_{c' \in [S]} \Big(s_{T-s}(x_s \oplus_{i'} c', x_s)s_{T-s}(x_s \oplus_{i'} c' \oplus_i c, x_s \oplus_{i'} c')\\
    &\qquad\qquad\qquad\qquad\qquad\qquad - s_{T-s}(x_s \oplus_i c, x_s)s_{T-s}(x_s \oplus_i c \oplus_{i'} c', x_s \oplus_i c)\Big)\\
    &\overset{\mathrm{(b)}}{=} \frac{1}{S} \sum_{i' \in [d]}\sum_{c' \in [S]} \Big(s_{T-s}(x_s \oplus_{i'} c' \oplus_i c, x_s) - s_{T-s}(x_s \oplus_i c \oplus_{i'} c', x_s)\Big),
\end{align*}
where in equality (a), we apply  the Kolmogorov forward equation on $q_{T-t}$; in equality (b), we use the fact that $s_{T-t}(x, y) s_{T-t}(y,z) = s_{T-t}(x, z)$ for any $x, y, z \in \cX$.
It is direct to check that the $\oplus$ operators commute, i.e., for any $x_s \in \cX$,
\begin{align*}
    x_s \oplus_{i'} c' \oplus_i c = x_s \oplus_i c \oplus_{i'} c'.
\end{align*}
for $i \neq i'$, and the relation holds trivially when $i = i'$.
This relation directly reveals that
\begin{align*}
    &\partial_t f(s, x_{s}) + \left(\back L_s f\right)(s, x_{s})\\
    &= \frac{1}{S} \sum_{i' \in [d]}\sum_{c' \in [S]} \Big(s_{T-s}(x_s \oplus_{i'} c' \oplus_i c, x_s) - s_{T-s}(x_s \oplus_i c \oplus_{i'} c', x_s)\Big) = 0,
\end{align*}
which completes the proof of Eqn.~\eqref{eq:uniform_score_drift_0}.

Taking $u = \ell$ in Eqn.~\eqref{eq:uniform_mart_def}, we have $M_\ell = 0$ almost surely, and $M_u$ is a local martingale for $u \in [\ell, t]$ by definition. 
Recalling Lemma~\ref{lem:uniform_log_score_bound}, we can bound 
\begin{align*}
    \sup_{s \in [\ell, t]} \sup_{x \in \cX} f(s, x) \leq \log(S) + \max\Big\{\log\big((T-t)^{-1}\big), 0\Big\} < \infty.
\end{align*}
Similarly, for the intensity of the counting process, it satisfies 
\begin{align*}
    \sup_{s \in [\ell, t]}\sup_{x, y \in \cX}\lambda_s^{x,y} \leq \frac{1}{S} s_{T-t}(y,x) \leq \frac{1}{S}\left(\log(S) + \max\Big\{\log\big((T-t)^{-1}\big), 0\Big\}\right) < \infty.
\end{align*}
Now it is direct to check that
\begin{align*}
    \sup_{s \in [\ell, t]}\E[|M_s|] \lesssim (t - \ell)d(S-1) \cdot \sup_{s \in [\ell, t]}\sup_{x, y \in \cX} \left[f(s, x) \cdot \lambda_s^{x, y}\right] < \infty.
\end{align*}
As a result, we can conclude $\{M_u\}_{u \in [\ell, t]}$ is $L^1$ and hence a martingale.
By the definition of the martingale, we arrive at 
\begin{align*}
    \E_{\back q_{t\mid \ell}(\cdot \mid x_\ell)}[M_t] = M_\ell = 0.
\end{align*}

Returning to Eqn.~\eqref{eq:uniform_score_ito_2}, we obtain
\begin{align}\label{eq:uniform_mart_mean}
    \E_{x_t \sim \back q_{t\mid\ell}(\cdot\mid x_\ell)}\big[f(t, x_t) - f(\ell, x_\ell)\big] = \E_{x_t \sim \back q_{t\mid\ell}(\cdot\mid x_\ell)}[M_{t}] = 0.
\end{align}
Thus, we conclude that
\begin{align*}
    &\bE_{x_t \sim \back q_{t|\ell}(\cdot \mid x_{\ell})} \Big[\big(s_{T - \ell}(x_\ell \oplus_i c, x_\ell) - s_{T - t}(x_t \oplus_i c, x_t)\big) \log \wh s_{T - \ell}(x_\ell \oplus_i c, x_\ell)\Big]\\
    &= \bE_{x_t \sim \back q_{t|\ell}(\cdot \mid x_{\ell})} \big[f(\ell, x_\ell) - f(t, x_t)\big] \cdot \log \wh s_{T - \ell}(x_\ell \oplus_i c, x_\ell) = 0,
\end{align*}
where we plug in Eqn.~\eqref{eq:uniform_mart_mean} in the last line.
\qed

\subsection{Proof of Lemma~\ref{lem:uniform_t3_sim}}
The proof of Lemma~\ref{lem:uniform_t3_sim} follows directly from exchanging the order of summation. Specifically, we can write 
\begin{align*}
    &\bE_{x_t \sim \back q_{t}}\left[\sum_{y_t: \ham(y_t, x_t) = 1} h\left(s_{T-t}(y_t, x_t)\right)\right]\\
    &= \bE_{x_t \sim \back q_t} \left[\sum_{y_t: \ham(y_t, x_t) = 1} s_{T-t}(y_t, x_t)\log(s_{T-t}(y_t, x_t)) - s_{T-t}(y_t, x_t) + 1\right]\\
    &= \bE_{x_t \sim \back q_t} \left[\sum_{y_t: \ham(y_t, x_t) = 1} \left(\frac{q_{T-t}(y_t)}{q_{T-t}(x_t)}\right)\log(s_{T-t}(y_t, x_t))\right] - \bE_{x_t \sim \back q_t} \left[\sum_{y_t: \ham(y_t, x_t) = 1} \left(\frac{q_{T-t}(y_t)}{q_{T-t}(x_t)}\right)\right] + d(S-1)\\
    &= \sum_{x_t \in [S]^d} \sum_{y_t: \ham(y_t, x_t) = 1} q_{T-t}(y_t) \log(s_{T-t}(y_t, x_t)) - \sum_{x_t \in [S]^d} \sum_{y_t: \ham(y_t, x_t) = 1} q_{T-t}(y_t) + d(S-1)\\
    &\overset{\mathrm{(a)}}{=} -\sum_{x_t \in [S]^d} \sum_{y_t: \ham(y_t, x_t) = 1} q_{T-t}(x_t) \log(s_{T-t}(y_t, x_t))-\sum_{y_t \in [S]^d} \sum_{x_t: \ham(y_t, x_t) = 1} q_{T-t}(y_t) + d(S-1)\\
    &= -\bE_{x_t \sim \back q_t} \left[\sum_{y_t: \ham(y_t, x_t) = 1} \log(s_{T-t}(y_t, x_t))\right] - d(S-1) + d(S-1)\\
    &= \bE_{x_t \sim \back q_{t}}\left[\sum_{y_t: \ham(y_t, x_t) = 1} -\log\left(s_{T - t}(y_t, x_t)\right)\right],
\end{align*}
where in equality (a), we switch the positions of $x_t$ and $y_t$ in the summations.
\qed
\subsection{Proof of Lemma~\ref{lem:uniform_log_score_bound}}
Lemma~\ref{lem:uniform_log_score_bound} is a direct consequence of \citet[Lemma 2]{liang2025discrete}. 
Here, we present a simplified proof based on Proposition~\ref{prop:score_function}. It can be easily checked that 
\begin{align*}
    \alpha_t = \frac{1 - e^{-t}}{1 + (S-1)e^{-t}} \in (0, 1).
\end{align*}
By Eqn.~\eqref{eq:unif_score}, one has, for $\ham(x, y) = 1$,
\begin{align*}
    s_t(y, x) = \frac{\bE_{x_0 \sim q_0} \alpha_t^{\ham(y, x_0)}}{\bE_{x_0 \sim q_0} \alpha_t^{\ham{(x, x_0)}}} 
    & \leq \alpha_t^{-\sup_{y,x,x_0}|\ham(y, x_0) - \ham(y, x_0)|} \\
    &= \exp\left(-\log(\alpha_t)\cdot \sup_{y,x,x_0}|\ham(y, x_0) - \ham(y, x_0)|\right)\\
    &\leq \exp\left(-\log(\alpha_t)\cdot \sup_{y,x}\ham(y, x)\right) = \exp\big(-\log(\alpha_t)\big).
\end{align*}
With similar calculation, one can establish the reversed inequality
\begin{align*}
    s_t(y, x) \geq \exp\left(\log(\alpha_t)\cdot \sup_{y,x}\ham(y, x)\right) = \exp\big(\log(\alpha_t)\big).
\end{align*}
As a result, we conclude 
\begin{align*}
    |\log\big(s_t(y, x)\big)| \le -\log(\alpha_t) \lesssim \log(S) + \max\big\{\log(t^{-1}), 0\big\}.
\end{align*}
\qed
\subsection{Proof of Lemma~\ref{lem:dkl_dt_phi}}

Let us start by proving the first equation, i.e.,
\begin{align*}
    \varphi_{i, c}(t) = \KL(q_t \,\|\, (N_{i,-c})_{\#}q_t).
\end{align*}
Recall the definition of $\varphi_{i, c}(t)$ that
\begin{align}
    \varphi_{i, c}(t) = \E_{x_t \sim q_t}\left[-\log\left(\frac{q_t\left(N_{i,c}(x_t)\right)}{q_t(x_t)}\right)\right] = \sum_{x_t \in \cX} q_t(x) \log\left(\frac{q_t(x)}{q_t\left(N_{i,c}(x_t)\right)}\right).\label{eq:phi_ic_N}
\end{align}
As in Eqn.~\eqref{eq:pushforward_ic}, the pushforward measure satisfies 
\begin{align*}
    (N_{i,-c})_{\#}q_t(x) = q_t\left(N_{i,c}(x)\right), \text{ for any } x \in \cX.
\end{align*}
As such, we can express Eqn.~\eqref{eq:phi_ic_N} as
\begin{align*}
    \varphi_{i, c}(t) = \sum_{x_t \in \cX} q_t(x) \log\left(\frac{q_t(x)}{(N_{i,-c})_{\#}q_t(x_t)}\right) = \KL(q_t \,\|\, (N_{i,-c})_{\#}q_t),
\end{align*}
which proves the first equation.

For the second relation, the definition of KL divergence gives
    \begin{align}
        -\frac{\partial}{\partial t} \KL (q_t \Vert p_0) = -\frac{\partial}{\partial t} \sum_{x \in [S]^d} q_t(x) \log(q_t(x)) &= -\sum_{x \in [S]^d} \frac{\d q_t(x)}{\d t} \big(\log(q_t(x)) + 1\big) = -\sum_{x \in [S]^d} \frac{\d q_t(x)}{\d t} \log(q_t(x)).\label{eq:uniform_dkl_dt}
    \end{align}
    Using the Kolmogorov forward equation for the forward noising process, we have
    \begin{align*}
        \frac{\d q_t(x)}{\d t} = \sum_{y \in \cX}Q(x, y) q_t(x) = \frac{1}{S}\sum_{y: \hm(y, x) = 1}q_t(y) -\frac{d(S-1)}{S}q_t(x).
    \end{align*}
    Plugging the equation above into Eqn.~\eqref{eq:uniform_dkl_dt}, we arrive at
    \begin{align*}
        -\frac{\partial}{\partial t} \KL (q_t \Vert p_0) 
        &= -\sum_{x \in [S]^d} \left(\sum_{y: \hm(y, x) = 1} \left(\frac{1}{S}q_t(y)\right) -  \frac{d(S-1)}{S}q_t(x)\right) \log(q_t(x))\\
        &= -\frac{1}{S}\sum_{x \in [S]^d} \sum_{y: \hm(y, x) = 1} (q_t(y) - q_t(x)) \log(q_t(x))\\
        &= -\frac{1}{S}\sum_{x \in [S]^d} \sum_{y: \hm(y, x) = 1} q_t(x) \big(\log(q_t(y)) - \log(q_t(x))\big) = \varphi(t).
    \end{align*}
In addition, recall $\varphi(t) = 1/S \sum_{i \in [d]} \sum_{c \in [S]} \varphi_{i, c}(t)$. We reach
\begin{align*}
    \varphi(t) = \frac{1}{S} \sum_{i \in [d]} \sum_{c \in [S]} \KL(q_t \,\|\, (N_{i,-c})_{\#}q_t) = \frac{1}{S} \sum_{i \in [d]} \sum_{c \in [S]} \KL(q_t \,\|\, (N_{i,c})_{\#}q_t).
\end{align*}
\qed

\subsection{Proof of Lemma~\ref{lem:phi'_lower_bound}}
\label{sec:pf-lemma-phi'}
        Let $L$ be the time-homogeneous infinitesimal generator of the forward process. Since each coordinate $i \in [d]$ is updated independently in the forward process, we can write $L = L_i + L_{-i}$, where $L_i$ only updates coordinate $i$, and $L_{-i}$ updates all other coordinates. It is direct to show that $L_i$ and $L_{-i}$ commute, therefore, we have for any $u \geq 0$,
        \begin{align*}
            q_{t+u} = q_t e^{uL_i} e^{uL_{-i}}, \quad (N_{i,-c})_{\#} q_{t+u} = ((N_{i,-c})_{\#} q_t) e^{uL_i} e^{uL_{-i}},
        \end{align*}
        where the second equation is due to the operator $N_{i,-c}$ commutes with the semigroup $\{e^{uL}\}_{u \geq 0}$. With this formulation, we reach
        \begin{align}\label{eq:dpi-i-coord}
            \varphi_{i,c}(t + u) = \KL(q_{t+u} \,\Vert\, (N_{i,-c})_{\#} q_{t+u}) \le \KL(q_t e^{uL_i} \,\Vert\, ((N_{i,-c})_{\#} q_t) e^{uL_i}),
        \end{align}
        where in the last inequality, we apply the weak data processing inequality for KL divergence.
        Since both $N_{i,-c}$ and $L_i$ only operate on the coordinate $i$, we arrive at the decomposition
        \begin{align}\label{eq:kl-semigroup-decomp}
            \KL(q_t e^{uL_i} \,\Vert\, ((N_{i,-c})_{\#} q_t) e^{uL_i}) = \bE_{x^{-i} \sim (q_t)^{-i}}\left[\KL\big(q_t(\cdot|x^{-i}) e^{uL_i} \,\Vert\, ((N_{i,-c})_{\#} q_t(\cdot|x^{-i})) e^{uL_i}\big)\right],
        \end{align}
        where $(q_t)^{-i}$ is the marginal distribution of $q_t$ with coordinate $i$ excluded. 
        Define $K_u$ to be the transition kernel on $[S] \times [S]$ induced by $e^{uL_i}$.
        It is shown that 
        \begin{align*}
            K_u(v_1, v_2) = \begin{cases}
                \frac{1}{S} + (1 - \frac{1}{S})e^{-u} \quad &\text{if } v_1 = v_2;\\
                \frac{1}{S}(1 - e^{-u}) \quad &\text{if } v_1 \neq v_2.
            \end{cases}
        \end{align*}
        It can be directly checked that $K_u$ is a $S$-ary symmetric channel with noise scale $\sigma_u = (1-S^{-1})(1-e^{-u})$.
        By \citet[Proposition 12]{makur2018comparison}, a strong data processing inequality holds for the channel $K_u$, i.e., for any distribution $p, q$ supported on $[S]$,
        \begin{align*}
            \KL(p e^{uL_i} \,\Vert\, q e^{uL_i}) \le \eta_{\KL}(K_u) \KL(p \,\Vert\, q),
        \end{align*}
        where $\eta_{\KL}(K_u)$ satisfies
        \begin{align*}
            \eta_{\KL}(K_u) \le \bigg|1 - \sigma_u - \frac{\sigma_u}{S-1}\bigg| = 1 - \frac{S}{S-1}(1-S^{-1})(1- e^{-u}) = e^{-u}.
        \end{align*}
        Taking this strong data processing inequality with Eqn.~\eqref{eq:kl-semigroup-decomp} yields 
        \begin{align*}
            \KL(q_t e^{uL_i} \Vert             ((N_{i,-c})_{\#} q_t) e^{uL_i}) &= \bE_{x^{-i} \sim (q_t)^{-i}}\left[\KL(q_t(\cdot|x^{-i}) e^{uL_i} \,\Vert\, ((N_{i,-c})_{\#} q_t(\cdot|x^{-i})) e^{uL_i})\right]\\
            &\leq \bE_{x_{-i} \sim (q_t)^{-i}}\left[e^{-u}\,\KL(q_t(\cdot|x^{-i})\,\Vert\, ((N_{i,-c})_{\#} q_t(\cdot|x_{-i})) \right]\\
            &\leq e^{-u}\,\bE_{x^{-i} \sim (q_t)^{-i}}\left[\KL(q_t(\cdot|x^{-i})\,\Vert\, ((N_{i,-c})_{\#} q_t(\cdot|x^{-i})) \right]\\
            &= e^{-u} \KL(q_t \,\Vert\, (N_{i,-c})_{\#} q_t).
        \end{align*}
        Then, by Eqn.~\eqref{eq:dpi-i-coord}, we have
        \begin{align*}
            \varphi_{i,c}(t+u) \leq \KL(q_t e^{uL_i} \,\Vert\, ((N_{i,-c})_{\#} q_t) e^{uL_i}) \leq e^{-u} \KL(q_t \,\Vert\, (N_{i,-c})_{\#} q_t) = e^{-u} \varphi_{i,c}(t),
        \end{align*}
        which holds for any $u \geq 0$.
        Therefore, the derivative can be bound as
        \begin{align*}
            \varphi_{i,c}'(t) = \lim_{u \to 0^{+}} \frac{\varphi_{i,c}(t+u) - \varphi_{i,c}(t)}{u} \le \lim_{u \to 0^{+}} \frac{e^{-u}}{u} \varphi_{i,c}(t) = -\varphi_{i,c}(t),
        \end{align*}
        which induces the target result
        \begin{align*}
            -\varphi_{i,c}'(t) \geq \varphi_{i,c}(t).
        \end{align*}
\qed

\subsection{Proof of~\Cref{lem:mask_mart}}
\label{sec:proof_mask_mark}

The proof follows from~\citep[Lemma 5.2.2]{conforti2025non}. We add the argument below for completeness.
 Let us define
    \begin{align*}
        f(t, x_t) \coloneqq  s_{T - t}(x_t \odot_i c, x_t) \bI\{i \in m(x_t)\},
    \end{align*}
    where the dependence on $i$ and $c$ is omitted for simplicity.
    In view of \Cref{lem:ito}, for $0 \leq \ell \leq t < T$, we can write 
    \begin{align*}
        f(t, x_t) = f(\ell, x_\ell) + \int_\ell^t\left[\partial_t f(s, x_s) + (\back L_s f)(s, x_s)\right]\d s + M_t,
    \end{align*}
    with generator \(
    \{\back L_s\}_{s \in [\ell, t]}\) as follows
    \begin{align*}
    \left(\back L_s f\right)(s, x) &= \sum_{y \neq x} Q_{T-s}(y, x) s_{T-s}(y, x)\big(f(s, y) - f(s, x)\big)\\
    &= \sum_{i' \in m(x)} \sum_{c' \in [S]} s_{T-s}(x \odot_{i'} c', x)\big(f(s, x \odot_{i'} c') - f(s, x)\big),
\end{align*}
    and the compensation process $\{M_u\}_{u\in[\ell, t]}$ defined as
    \begin{align*}
        M_u = \int_\ell^u \sum_{i' \in m(x_s)} \sum_{c' \in [S]} \big(f(s, x_s \odot_{i'} c') - f(s, x_s)\big)\big(\d N_s^{x_s, x_s \odot_{i'} c'} - \lambda_s^{x_s, x_s \odot_{i'} c'} \d s\big).
    \end{align*}
    With similar argument as in the proof of Lemma~\ref{lem:uniform_martingale}, one has 
    $\E_{\back q_{t\mid \ell}(\cdot \mid x_\ell)}[M_t] = 0,$ which leads to
    \begin{align*}
        \bE_{x_t \sim \back q_{t \mid \ell} (\cdot \mid x_{\ell})} [f(t, x_t) - f(\ell, x_\ell)] =  \bE_{x_t \sim \back q_{t \mid \ell} (\cdot \mid x_{\ell})} \left[\int_{\ell}^t \big(\partial_t f(s, x_s) + (\back L_s f)(s, x_s)\big) \d s\right].
    \end{align*}
    Taking derivative with respect to $t$ on both side, we arrive at
    \begin{align*}
        \frac{\d}{\d t} \bE_{x_t \sim \back q_{t \mid \ell} (\cdot \mid x_{\ell})} [f(t, x_t)] = \bE_{x_t \sim \back q_{t \mid \ell} (\cdot \mid x_{\ell})} \left[\partial_t f(t, x_t) + (\back L_t f)(t, x_t)\right].
    \end{align*}
    Now let us consider each term on the right hand side above separately. By~\Cref{prop:score_function}, it obeys that $$s_{T - t}(x_t \odot_i c, x_t) = \frac{1}{e^{T - t} - 1} \frac{q_0(x_t \odot_i c)}{q_0(x_t)},$$ and we have
    \begin{align*}
        \frac{\partial}{\partial t} f(t, x_t) = \frac{e^{T - t}}{e^{T - t} - 1} f(t, x_t).
    \end{align*}
    Next, direct calculations yield 
\begin{align*}
 (\back L_t f)(t, x_t) 
& = \sum_{i' \in m(x_t)} \sum_{c' \in [S]} s_{T -t}(x_t \odot_{i'} c', x_t)\Bigl(s_{T - t}(x_t \odot_{i} c \odot_{i'} c', x_t \odot_{i'} c') \bI\{i \in m(x_t \odot_{i'} c')\}\\
&\hspace{10em} 
- s_{T - t}(x_t \odot_i c, x_t) \bI\{i \in m(x_t)\}\Bigr) \\
&\quad = \frac{1}{e^{T - t} - 1} f(t, x_t) \left(\sum_{i' \in m(x_t) \setminus \{i\}} \sum_{c' \in [S]} \frac{q_0(x_t \odot_i c \odot_{i'} c')}{q_0(x_t \odot_i c)} - \sum_{i' \in m(x_t)} \sum_{c' \in [S]} \frac{q_0(x_t \odot_{i'} c')}{q_0(x_t)}\right) \\
&\quad = \frac{1}{e^{T - t} - 1} f(t, x_t)(|m(x_t) \setminus \{i\}| - |m(x_t)|)
 = -\frac{1}{e^{T - t} - 1} f(t, x_t).
    \end{align*}
Putting everything together leads to  
    \begin{align*}
         \frac{\d}{\d t} \bE_{x_t \sim \back q_{t \mid \ell} (\cdot \mid x_{\ell})} [f(t, x_t)] = \bE_{x_t \sim \back q_{t \mid \ell} (\cdot \mid x_{\ell})} [f(t, x_t)],
    \end{align*}
    and therefore,
    \begin{align*}
        \bE_{x_t \sim \back q_{t \mid \ell} (\cdot \mid x_{\ell})} [f(t, x_t)] = e^{t - \ell} \cdot f(\ell, x_\ell).
    \end{align*}
    
Finally, in view of the relation \(\Pr(x_t^i = \mask \mid x_{\ell}^i = \mask) = \frac{1 - e^{t - T}}{1 - e^{\ell - T}}\), we conclude the following 
\begin{align*}
\begin{aligned}
\bE_{x_t \sim \back q_{t \mid \ell}(\cdot \mid x_\ell)} \big[s_{T - t}(x_t \odot_i c, x_t) \,\bI\{i \in m(x_t)\}\big] &= e^{t-\ell} \cdot s_{T - \ell}(x_{\ell} \odot_i c, x_{\ell}) \bI\{i \in m(x_{\ell})\}\\
&= \bE_{x_t \sim \back q_{t \mid \ell}(\cdot \mid x_\ell)} \big[s_{T - t}(x_\ell \odot_i c, x_\ell)\,\bI\{i \in m(x_t)\}\big],
\end{aligned}
\end{align*}
which completes the proof of the desired result. 
\qed

%% file: appendix/app_aux_lemmas.tex
\section{Proofs of the auxiliary lemmas}

\subsection{Proof of \Cref{lem:I-variance}}
\label{sec:pf-i-variance}
For a continuous random variable $U$ in $\mathbb{R}^d$ with density function $p_U$ with respect to Lebesgue measure, define the differential entropy of $U$ as
\begin{align}\label{eq:defn-diff-ent}
    \cH^{\mathrm{diff}}(U) = -\int_{\mathbb{R}^d} p_U \log(p_U) \d x,
\end{align}
where we adopt the convention $0\log(0) = 0$ again.
By definition of mutual information, we have
\begin{align}
    I(W; W + \varepsilon_{\mathrm{noise}}) &= \cH^{\mathrm{diff}}(W + \varepsilon_{\mathrm{noise}}) - \cH^{\mathrm{diff}}(W + \varepsilon_{\mathrm{noise}} \mid W)\notag\\
    &\overset{\mathrm{(a)}}{=} \cH^{\mathrm{diff}}(W + \varepsilon_{\mathrm{noise}}) - \E_{w}\big[\cH^{\mathrm{diff}}(w + \varepsilon_{\mathrm{noise}} \mid W = w)\big]\notag\\
    &\overset{\mathrm{(b)}}{=} \cH^{\mathrm{diff}}(W + \varepsilon_{\mathrm{noise}}) - \E_{w}\big[\cH^{\mathrm{diff}}(\varepsilon_{\mathrm{noise}} \mid W = w)\big]\notag\\
    &\overset{\mathrm{(c)}}{=} \cH^{\mathrm{diff}}(W + \varepsilon_{\mathrm{noise}}) - \cH^{\mathrm{diff}}(\varepsilon_{\mathrm{noise}}),\label{eq:manifold-i-w-w}
\end{align}
where in (a), we use the chain rule of differential entropy; in (b), we apply the translation invariance property, i.e., $\cH^{\mathrm{diff}}(U) = \cH^{\mathrm{diff}}(c_0 + U)$ for any constant $c_0$; in (c), we use the condition that $\varepsilon_{\mathrm{noise}} \indep W$.

Denote the Gaussian density function with mean $0$ and variance $\sigma^2 I_d$ as $\phi_\sigma(\cdot)$. 
Since $\varepsilon_{\mathrm{noise}} \sim \cN(0, \sigma^2 I_d)$, we can compute with \Cref{eq:defn-diff-ent} that
\begin{align}
    \cH^{\mathrm{diff}}(\varepsilon_{\mathrm{noise}}) &= -\int_{\mathbb{R}^d} \phi_\sigma(x) \log\big(\phi_\sigma(x)\big) \d x\notag\\
    &= -\int_{\mathbb{R}^d} \phi_\sigma(x) \left(-\frac{d}{2} \log(2 \pi \sigma^2) - \frac{\|x\|_2^2}{2\sigma^2} \right) \d x\notag\\
    &= \frac{d}{2} \log(2 \pi \sigma^2) + \frac{\E[\|\varepsilon_{\mathrm{noise}}\|_2^2]}{2\sigma^2}\notag\\
    &= \frac{d}{2} \log(2 \pi e \sigma^2),\label{eq:gaussian-diff-entropy}
\end{align}
where $\|\cdot\|_2$ is the Euclidean norm in $\mathbb{R}^d$.
For $\cH^{\mathrm{diff}}(W + \varepsilon_{\mathrm{noise}})$, notice that
\begin{align*}
    \mathrm{Var}\big[W + \varepsilon_{\mathrm{noise}}\big] = \mathrm{Var}[W] + \mathrm{Var}[\varepsilon_{\mathrm{noise}}] + 2\,\mathrm{Cov}\big[W, \varepsilon_{\mathrm{noise}}\big]= \mathrm{Var}[W] + \sigma^2 I_d.
\end{align*}
By \citet[Page 255]{cover1999elements}, for distributions with the same finite variance, $\cH^{\mathrm{diff}}$ is maximized at the centered Gaussian random variable.
Therefore, we have
\begin{align*}
    \cH^{\mathrm{diff}}(W + \varepsilon_{\mathrm{noise}}) &\leq \cH^{\mathrm{diff}}\Big(\cN\big(0, \mathrm{Var}[W] + \sigma^2 I_d\big)\Big) = \frac{d}{2} \log(2 \pi e) + \frac{1}{2} \log\left(\mathrm{det}\big(\mathrm{Var}[W] + \sigma^2 I_d\big)\right),
\end{align*}
where $\mathrm{det}(\cdot)$ is the determinant of matrices, and the calculation is the same as in \Cref{eq:gaussian-diff-entropy}. 
Since $\mathrm{Var}[W]$ is a positive semidefinite matrix, we can apply the matrix inequality that
\begin{align*}
    \log\left(\mathrm{det}\big(\mathrm{Var}[W] + \sigma^2 I_d\big)\right) &= d\log(\sigma^2) + \log\left(\mathrm{det}\big(I_d + \mathrm{Var}[W/\sigma^2]\big)\right)\\
    &\leq d \log(\sigma^2) + \tr\left(\mathrm{Var}[W/\sigma^2]\right)\\
    &= d \log(\sigma^2) + \frac{\tr\left(\mathrm{Var}[W]\right)}{\sigma^2},
\end{align*}
which leads to 
\begin{align}\label{eq:gaussian-channel-bound}
    \cH^{\mathrm{diff}}(W + \varepsilon_{\mathrm{noise}}) &\leq \frac{d}{2} \log(2 \pi e \sigma^2) + \frac{\tr\left(\mathrm{Var}[W]\right)}{2\sigma^2}.
\end{align}
Plugging Eqns.~\eqref{eq:gaussian-diff-entropy} and \eqref{eq:gaussian-channel-bound} into \Cref{eq:manifold-i-w-w}, we conclude that
\begin{align*}
    I(W; W + \varepsilon_{\mathrm{noise}}) \leq \frac{d}{2} \log(2 \pi e \sigma^2) + \frac{\tr\left(\mathrm{Var}[W]\right)}{2\sigma^2} - \frac{d}{2} \log(2 \pi e \sigma^2) = \frac{\tr\left(\mathrm{Var}[W]\right)}{2\sigma^2}.
\end{align*}
\qed
\subsection{Proof of~\Cref{lem:bin-even-odd}}
For $X \sim \mathrm{Bin}(n , 1/2)$, its pmf is given by 
\begin{align*}
    \P(X = x) = {n \choose x} \left(\frac{1}{2}\right)^n \propto {n \choose x}.
\end{align*}
Notice that our desired bound is equivalent to the following equation:
\begin{align*}
    \sum_{x: x \bmod 2 = 0} {n \choose x}  = \sum_{x: x \bmod 2 = 1} {n \choose x},
\end{align*}
which follows from the binomial theorem for \(0 = (1 - 1)^n\). \qed

\subsection{Proof of~\Cref{lem:alg_is_ctmc}}
\label{sec:proof_is_ctmc}

    The CTMC~\eqref{eq:modified_ttl} in the lemma statement can be decomposed into \(d\) independent CTMCs for each coordinate. For coordinates \(i\) such that \(x_{t_k}^i \neq \mathrm{MASK}\) clearly neither~\Cref{eq:modified_ttl} nor~\Cref{alg:modified_ttl} makes a change. Next, we fix \(i \in m(x_{t_k})\).
    First, we compute the probability that the \(i\)-th coordinate remains masked for \(y_{t_{k+1}}\):
    \begin{equation*}
    \begin{aligned}
        \Pr(y_{t_{k+1}}^i  = \mathrm{MASK} \mid y_{t_k}) &= \exp\left(\int_{t_k}^{t_{k+1}} \left(-\sum_{c \in [S]} \wh s_{T - t_k}(y_{t_k} \odot_{i} c, y_{t_k} ) \frac{e^{T - t_k} - 1}{e^{T - t} - 1}\right) \d t\right) \\
        & = \exp(\wh Q_k^i(\mathrm{MASK}) \Delta_k) \\
        & = \cP_k,
    \end{aligned}
    \end{equation*}
    where \(\wh Q_k^i(\mathrm{MASK})\), \(\Delta_k\), and \(\cP_k\) are defined in~\Cref{alg:modified_ttl}. Next, for \(c \in [S]\) we can write
    \begin{equation*}
        \Pr(y_{t_{k+1}}^i = c \mid x_{t_k}) = \Pr(x_{t_{k+1}}^i = c \mid x_{t_k} \text{ and } x_{t_{k+1}}^i \neq \mathrm{MASK}) (1 - \cP_k).
    \end{equation*}
    Since for any \(t \in [t_k, t_{k+1})\) the rates \(\wh Q_t(x, x \odot_i c)\) are proportional to \(\wh Q_k^i(c)\), we get that
    \begin{equation*}
         \Pr(y_{t_{k+1}}^i = c \mid x_{t_k} \text{ and } y_{t_{k+1}}^i \neq \mathrm{MASK}) = \frac{\wh Q_k^i(c)}{\sum_{b \in [S]} \wh Q_k^i(b)},
    \end{equation*}
    which matches the expression in~\Cref{alg:modified_ttl}. Therefore, the distribution of \(y_{t_{k+1}}\) defined by the CTMC matches the distribution of \(x_{t_{k+1}}\) from the algorithm. 
\qed

\subsection{Proof of~\Cref{lem:control_at_t}}
\label{sec:proof_control_at_t}
  In view of the definition of $D(\cdot, \cdot)$, one can write  
    \begin{equation*}
    \begin{aligned}
        &s_{T - t}(x_t \odot_i c, x_t)\breg \left( s_{T - t}(x_{\ell} \odot_i c, x_{\ell}), s_{T - t}(x_t \odot_i c, x_t)\right) \\
        &= s_{T - t}(x_{\ell} \odot_i c, x_{\ell}) -  s_{T - t}(x_t \odot_i c, x_t) +  s_{T - t}(x_t \odot_i c, x_t) \log \frac{ s_{T - t}(x_t \odot_i c, x_t) }{ s_{T - t}(x_{\ell} \odot_i c, x_{\ell})}.
    \end{aligned}
    \end{equation*}
    The first two terms cancel out in expectation by~\Cref{lem:mask_mart}; i.e., for any \(c \in [S]\), one has
    \begin{align*}
         \bE_{x_t \sim \back q_{t \mid \ell}(\cdot \mid x_\ell)} \Big[\sum_{i \in m(x_t)} \left(s_{T - t}(x_\ell \odot_i c, x_\ell) - s_{T - t}(x_t \odot_i c, x_t)\right)\Big] = 0.
    \end{align*}
    Next, using~\Cref{eq:mask_score}, we obtain 
    \begin{equation*}
        \frac{ s_{T - t}(x_t \odot_i c, x_t) }{s_{T - t}(x_{\ell} \odot_i c, x_{\ell})} = \frac{q_0(x_t \odot_i c) q_0(x_{\ell})}{q_0(x_t) q_0(x_{\ell} \odot_i c)}.
    \end{equation*}
    Using this relation, we continue 
    \begin{align}
        &\bE_{x_{\ell}, x_t \sim \back q_{\ell, t}} \sum_{i \in m(x_t)} \sum_{c \in [S]} s_{T - t}(x_t \odot_i c, x_t) \log \frac{q_0(x_t \odot_i c) q_0(x_{\ell})}{q_0(x_t) q_0(x_{\ell} \odot_i c)} \notag\\
        &\quad = \bE_{y_\ell, y_t \sim \back q_{\ell, t}} \sum_{i \notin m(y_t)} 
        \log \frac{q_0(y_t) q_0(y_{\ell} \odot_i \mask)}{q_0(y_t \odot_i \mask) q_0(y_{\ell} \odot_i y_t^i)} \notag\\
        &\quad = \sum_{i \in [d]} \bE_{y_\ell, y_t \sim \back q_{\ell, t}} 
        \log \frac{q_0(y_t) q_0(y_{\ell} \odot_i \mask)}{q_0(y_t \odot_i \mask) q_0(y_{\ell} \odot_i y_t^i)}, \label{eq:change_measure}
    \end{align}
    where in the second line, we used the definition of the score function along with the natural bijection between the sets
    \(\{(x, i, c) \text{, for } x \in \cX \text{, } i \in m(x) \text{, and } c \in [S]\}\) and \(\{(y, i) \text{, for } y \in \cX \text{ and } i \notin m(y)\}\) 
    to change the measure under the expectation: 
    \begin{align*}
        x_t &\to y_t \odot_i \mask \\
        x_\ell &\to y_\ell \odot_i \mask \\
        x_t \odot_i c &\to y_t \\
        x_\ell \odot_i c &\to y_\ell \odot_i y_t^i.
    \end{align*}
    Note that since \(y_\ell\) appears earlier in the backward process, \(y_\ell^i\) can be masked or unmasked. Since the \(i\)-th element of \(x_{
    \ell
    } \odot_i c\) is unmasked by construction, we explicitly set the \(i\)-th element of \(y_\ell\) to \(y_t^i\). The third line follows from the fact that, for \(i \in m(y_t)\), the term is equal to zero.

    Next, we define, for fixed 
    \(t, y_t\), and \(i \in [d]\),
    \begin{equation*}
        f_i(y) \defn \log \frac{q_0(y \odot_i y_t^i)}{q_0(y \odot_i \mask)},
    \end{equation*}
    and rewrite~\Cref{eq:change_measure} as follows:
    \begin{equation*}
    \sum_{i \in [d]} \bE_{y_\ell, y_t \sim \back q_{\ell, t}} 
        \log \frac{q_0(y_t) q_0(y_{\ell} \odot_i \mask)}{q_0(y_i \odot_i \mask) q_0(y_{\ell} \odot_i y_t^i)} = \sum_{i \in [d]} \bE_{y_\ell, y_t \sim \back q_{\ell, t}} \left[f_i(y_t) - f_i(y_\ell)\right].
    \end{equation*}
    We observe that as the value \(f_i(y)\) does not depend on the \(i\)-th coordinate of \(y\), we can apply Dynkin's formula,~\Cref{lem:ito} to the remaining \(d - 1\) coordinates for the forward process: \[\bE_{y_\ell, y_t \sim \back q_{\ell, t}}\left[f_i(y_t) - f_i(y_{\ell})\right] =  \int_t^{\ell}\bE_{y_v \sim \back q_{v}}\sum_{j \neq i}  \left[f_i(y_v) -  f_i(y_v \odot_j \mask)\right] \d v.\]
    With this, we continue:
    \begin{align}
\label{eq:total_cor_lemma_dynkin}
        &\sum_{i \in [d]} \bE_{y_\ell, y_t \sim \back q_{\ell, t}} \left[f_i(y_t) - f_i(y_\ell)\right] \notag\\
        &\quad = 
        \sum_{i \in [d]} \int_{\ell}^t \bE_{y_v, y_t \sim \back q_{v, t}} \sum_{j \notin m(y_v) \cup \{i\}} \log \frac{q_0(y_v 
    \odot_i y_t^i) q_0(y_v \odot_i \mathrm{MASK} \odot_j \mathrm{MASK})}{q_0(y_v \odot_i \mathrm{MASK}) q_0(y_v \odot_i y_t^i \odot_j \mathrm{MASK})} \d v \notag\\
        &\quad = \sum_{i \neq j \in [d]} \int_{\ell}^t \bE_{y_v, y_t \sim \back q_{v, t}} \log \frac{q_0(y_v \odot_i y_t^i) q_0(y_v \odot_i \mathrm{MASK} \odot_j \mathrm{MASK})}{q_0(y_v \odot_i \mathrm{MASK}) q_0(y_v\odot_i y_t^i \odot_j \mathrm{MASK})} \d v \notag\\
        &\quad = \sum_{i \neq j \in [d]} \int_{\ell}^t e^{t - v} \bE_{y_v \sim \back q_v} \log \frac{q_0(y_v) q_0(y_v \odot_i \mathrm{MASK} \odot_j \mathrm{MASK})}{q_0(y_v \odot_i \mathrm{MASK}) q_0(y_v \odot_j \mathrm{MASK})} \d v,
    \end{align}
    where in the third line, as before, we extended the sum to \(j \in m(y_v) \setminus \{i\}\) since additional terms equal zero. 
    The last line follows from
    \begin{equation*}
        \Pr(y_v^i \neq \mask \mid y_t^i \neq \mask) = e^{v - t}.
    \end{equation*}
    Next, let \(y_v^{-(i,j)}\) denote all unmasked elements of \(y_v\), except \(i\)-th and \(j\)-th.
    We can write
    \begin{equation*}
        \frac{q_0(y_v) q_0(y_v \odot_i \mathrm{MASK} \odot_j \mathrm{MASK})}{q_0(y_v \odot_i \mathrm{MASK}) q_0(y_v \odot_j \mathrm{MASK})} = \frac{q_0(y_v^i, y_v^j \mid y_v^{-(i,j)})}{q_0(y_v^i \mid y_v^{-(i,j)}) q_0(y_v^j \mid y_v^{-(i,j)})},
    \end{equation*}
    and thus,
    \begin{align}
        &\sum_{i \neq j \in [d]} \int_{\ell}^t e^{t - v} \bE_{y_v \sim \back q_v} \log \frac{q_0(y_v) q_0(y_v \odot_i \mathrm{MASK} \odot_j \mathrm{MASK})}{q_0(y_v \odot_i \mathrm{MASK}) q_0(y_v \odot_j \mathrm{MASK})} \d v \notag\\
        &\quad = \sum_{i \neq j} \int_{\ell}^t e^{t - v} \mi(y_v^i; y_v^j \mid y_v^{-(i, j)}) \d v \notag\\
        &\quad = \int_{\ell}^t e^{t - v} \cI(T - v) \d v,
\label{eq:total_cor_lemma_final}
    \end{align}
    as \(y_v \sim q_{T - v}\). Combining~Eqns.~(\ref{eq:change_measure}), (\ref{eq:total_cor_lemma_dynkin}), and~(\ref{eq:total_cor_lemma_final}) concludes the proof.
\qed

%% file: reference-diffusion.bib
@article{shi2024simplified,
  title={Simplified and generalized masked diffusion for discrete data},
  author={Shi, Jiaxin and Han, Kehang and Wang, Zhe and Doucet, Arnaud and Titsias, Michalis},
  journal={Advances in neural information processing systems},
  volume={37},
  pages={103131--103167},
  year={2024}
}

@article{liang2025low,
  title={Low-dimensional adaptation of diffusion models: Convergence in total variation},
  author={Liang, Jiadong and Huang, Zhihan and Chen, Yuxin},
  journal={arXiv preprint arXiv:2501.12982},
  year={2025}
}

@article{von2025scaling,
  title={Scaling Behavior of Discrete Diffusion Language Models},
  author={von R{\"u}tte, Dimitri and Fluri, Janis and Pooladzandi, Omead and Sch{\"o}lkopf, Bernhard and Hofmann, Thomas and Orvieto, Antonio},
  journal={arXiv preprint arXiv:2512.10858},
  year={2025}
}

@article{bach2025sampling,
  title={Sampling Binary Data by Denoising through Score Functions},
  author={Bach, Francis and Saremi, Saeed},
  journal={arXiv preprint arXiv:2502.00557},
  year={2025}
}

@inproceedings{
pham2025discrete,
title={Discrete Markov Probabilistic Models: An Improved Discrete Score-Based  Framework with sharp convergence bounds under minimal assumptions},
author={Le-Tuyet-Nhi Pham and Dario Shariatian and Antonio Ocello and Giovanni Conforti and Alain Oliviero Durmus},
booktitle={ International Conference on Machine Learning},
year={2025},
}

@article{austin2018multi,
  title = {Multi-Variate Correlation and Mixtures of Product Measures},
  author = {Austin, Tim},
  year = 2020,
  month = jul,
  journal = {Kybernetika},
  pages = {459--499},
  issn = {0023-5954, 1805-949X},
  doi = {10.14736/kyb-2020-3-0459},
  urldate = {2026-02-04},
  langid = {english}
}

@article{feller1940integro,
  title={On the integro-differential equations of purely discontinuous Markoff processes},
  author={Feller, Willy},
  journal={Transactions of the American Mathematical Society},
  volume={48},
  number={3},
  pages={488--515},
  year={1940}
}

@article{feinberg2014solutions,
  title={On solutions of {kolmogorov's} equations for nonhomogeneous jump Markov processes},
  author={Feinberg, Eugene A and Mandava, Manasa and Shiryaev, Albert N},
  journal={Journal of Mathematical Analysis and Applications},
  volume={411},
  number={1},
  pages={261--270},
  year={2014},
  publisher={Elsevier}
}

@inproceedings{park2024jump,
  title={Jump your steps: Optimizing sampling schedule of discrete diffusion models},
  author={Park, Yong-Hyun and Lai, Chieh-Hsin and Hayakawa, Satoshi and Takida, Yuhta and Mitsufuji, Yuki},
  booktitle={The Thirteenth International Conference on Learning Representations},
  year={2025}
}

@inproceedings{
chen2022sampling,
title={Sampling is as easy as learning the score: theory for diffusion models with minimal data assumptions},
author={Sitan Chen and Sinho Chewi and Jerry Li and Yuanzhi Li and Adil Salim and Anru Zhang},
booktitle={The Eleventh International Conference on Learning Representations },
year={2023},
}

@article{benton2024denoising,
  title={From denoising diffusions to denoising markov models},
  author={Benton, Joe and Shi, Yuyang and De Bortoli, Valentin and Deligiannidis, George and Doucet, Arnaud},
  journal={Journal of the Royal Statistical Society Series B: Statistical Methodology},
  volume={86},
  number={2},
  pages={286--301},
  year={2024},
  publisher={Oxford University Press US}
}

@article{li2024generative,
  title={Generative emulation of weather forecast ensembles with diffusion models},
  author={Li, Lizao and Carver, Robert and Lopez-Gomez, Ignacio and Sha, Fei and Anderson, John},
  journal={Science Advances},
  volume={10},
  number={13},
  year={2024},
  publisher={American Association for the Advancement of Science}
}

@article{zeni2025generative,
  title={A generative model for inorganic materials design},
  author={Zeni, Claudio and Pinsler, Robert and Z{\"u}gner, Daniel and Fowler, Andrew and Horton, Matthew and Fu, Xiang and Wang, Zilong and Shysheya, Aliaksandra and Crabb{\'e}, Jonathan and Ueda, Shoko and others},
  journal={Nature},
  volume={639},
  number={8055},
  pages={624--632},
  year={2025},
  publisher={Nature Publishing Group UK London}
}

@article{ho2022video,
  title={Video diffusion models},
  author={Ho, Jonathan and Salimans, Tim and Gritsenko, Alexey and Chan, William and Norouzi, Mohammad and Fleet, David J},
  journal={Advances in Neural Information Processing Systems},
  volume={35},
  pages={8633--8646},
  year={2022}
}

@article{gillespie2001approximate,
  title={Approximate accelerated stochastic simulation of chemically reacting systems},
  author={Gillespie, Daniel T},
  journal={The Journal of chemical physics},
  volume={115},
  number={4},
  pages={1716--1733},
  year={2001},
  publisher={AIP Publishing}
}

@article{chen2025optimal,
  title={Optimal inference schedules for masked diffusion models},
  author={Chen, Sitan and Cong, Kevin and Li, Jerry},
  journal={arXiv preprint arXiv:2511.04647},
  year={2025}
}

@article{libreaking,
  title={Breaking {AR'}s Sampling Bottleneck: Provable Acceleration via Diffusion Language Models},
  author={Li, Gen and Cai, Changxiao},
  journal={Advances in Neural Information Processing Systems},
  volume={38},
    year={2025}
}

@article{liang2025discrete,
  title={Discrete diffusion models: Novel analysis and new sampler guarantees},
  author={Liang, Yuchen and Liang, Yingbin and Lai, Lifeng and Shroff, Ness},
  journal={Advances in Neural Information Processing Systems},
  volume={39},
  year={2025}
}

@article{liang2025absorb,
  title={Absorb and Converge: Provable Convergence Guarantee for Absorbing Discrete Diffusion Models},
  author={Liang, Yuchen and Huang, Renxiang and Lai, Lifeng and Shroff, Ness and Liang, Yingbin},
  journal={Advances in Neural Information Processing Systems},
  volume={39},
  year={2025}
}

@inproceedings{ren2024discrete,
  title={How discrete and continuous diffusion meet: Comprehensive analysis of discrete diffusion models via a stochastic integral framework},
  author={Ren, Yinuo and Chen, Haoxuan and Rotskoff, Grant M and Ying, Lexing},
  booktitle={International Conference on Learning Representations},
year={2025},
}

@inproceedings{lou2023discrete,
  title={Discrete diffusion modeling by estimating the ratios of the data distribution},
  author={Lou, Aaron and Meng, Chenlin and Ermon, Stefano},
  booktitle={International Conference on Machine Learning},
    pages={4735--4763},
  year={2024},
  organization={PMLR}
}

@article{campbell2022continuous,
  title={A continuous time framework for discrete denoising models},
  author={Campbell, Andrew and Benton, Joe and De Bortoli, Valentin and Rainforth, Thomas and Deligiannidis, George and Doucet, Arnaud},
  journal={Advances in Neural Information Processing Systems},
  volume={35},
  pages={28266--28279},
  year={2022}
}

@article{conforti2025non,
  title={Non-Asymptotic Convergence of Discrete Diffusion Models: Masked and Random Walk dynamics},
  author={Conforti, Giovanni and Durmus, Alain and Pham, Le-Tuyet-Nhi},
  journal={arXiv preprint arXiv:2512.00580},
  year={2025}
}

@article{li2023towards,
  title={Towards Faster Non-Asymptotic Convergence for Diffusion-Based Generative Models},
  author={Li, Gen and Wei, Yuting and Chen, Yuxin and Chi, Yuejie},
  journal={arXiv preprint arXiv:2306.09251},
  year={2023}
}

@article{watson2023novo,
  title={De novo design of protein structure and function with RFdiffusion},
  author={Watson, Joseph L and Juergens, David and Bennett, Nathaniel R and Trippe, Brian L and Yim, Jason and Eisenach, Helen E and Ahern, Woody and Borst, Andrew J and Ragotte, Robert J and Milles, Lukas F and others},
  journal={Nature},
  volume={620},
  number={7976},
  pages={1089--1100},
  year={2023},
  publisher={Nature Publishing Group UK London}
}

@article{ho2020denoising,
  title={Denoising diffusion probabilistic models},
  author={Ho, Jonathan and Jain, Ajay and Abbeel, Pieter},
  journal={Advances in Neural Information Processing Systems},
  volume={33},
  pages={6840--6851},
  year={2020}
}

@article{meng2022concrete,
  title={Concrete score matching: Generalized score matching for discrete data},
  author={Meng, Chenlin and Choi, Kristy and Song, Jiaming and Ermon, Stefano},
  journal={Advances in Neural Information Processing Systems},
  volume={35},
  pages={34532--34545},
  year={2022}
}

@inproceedings{ou2024your,
  title={Your absorbing discrete diffusion secretly models the conditional distributions of clean data},
  author={Ou, Jingyang and Nie, Shen and Xue, Kaiwen and Zhu, Fengqi and Sun, Jiacheng and Li, Zhenguo and Li, Chongxuan},
  booktitle={The Thirteenth International Conference on Learning Representations },
  year={2025}
}

@article{sahoo2024simple,
  title={Simple and effective masked diffusion language models},
  author={Sahoo, Subham and Arriola, Marianne and Schiff, Yair and Gokaslan, Aaron and Marroquin, Edgar and Chiu, Justin and Rush, Alexander and Kuleshov, Volodymyr},
  journal={Advances in Neural Information Processing Systems},
  volume={37},
  pages={130136--130184},
  year={2024}
}

@article{chen2024convergence,
  title = {Convergence {{analysis}} of {{discrete diffusion model}}: {{exact implementation through uniformization}}},
  shorttitle = {Convergence {{Analysis}} of {{Discrete Diffusion Model}}},
  author = {Chen, Hongrui and Ying, Lexing},
  year = 2025,
  month = jun,
  journal = {Journal of Machine Learning},
  volume = {4},
  number = {2},
  pages = {108--127},
  issn = {2790-2048, 2790-203X},
  doi = {10.4208/jml.240812},
  urldate = {2026-02-03},
  copyright = {https://creativecommons.org/licenses/by/4.0/}
}

@article{li2025dimension,
  title={Dimension-free convergence of diffusion models for approximate Gaussian mixtures},
  author={Li, Gen and Cai, Changxiao and Wei, Yuting},
  journal={arXiv preprint arXiv:2504.05300},
  year={2025}
}

@inproceedings{chen2023improved,
  title={Improved analysis of score-based generative modeling: User-friendly bounds under minimal smoothness assumptions},
  author={Chen, Hongrui and Lee, Holden and Lu, Jianfeng},
  booktitle={International Conference on Machine Learning},
  pages={4735--4763},
  year={2023},
  organization={PMLR}
}

@article{van1992uniformization,
  title={Uniformization for nonhomogeneous Markov chains},
  author={Van Dijk, Nico M},
  journal={Operations research letters},
  volume={12},
  number={5},
  pages={283--291},
  year={1992},
  publisher={Elsevier}
}

@article{gillespie1976general,
  title={A general method for numerically simulating the stochastic time evolution of coupled chemical reactions},
  author={Gillespie, Daniel T},
  journal={Journal of computational physics},
  volume={22},
  number={4},
  pages={403--434},
  year={1976},
  publisher={Elsevier}
}

@article{li2024adapting,
  title={Adapting to unknown low-dimensional structures in score-based diffusion models},
  author={Li, Gen and Yan, Yuling},
  journal={Advances in Neural Information Processing Systems},
  volume={37},
  pages={126297--126331},
  year={2024}
}

@article{huang2024denoising,
  title={Denoising diffusion probabilistic models are optimally adaptive to unknown low dimensionality},
  author={Huang, Zhihan and Wei, Yuting and Chen, Yuxin},
  journal={arXiv preprint arXiv:2410.18784},
  year={2024}
}

@article{austin2021structured,
  title={Structured denoising diffusion models in discrete state-spaces},
  author={Austin, Jacob and Johnson, Daniel D and Ho, Jonathan and Tarlow, Daniel and van den Berg, Rianne},
  journal={Advances in Neural Information Processing Systems},
  volume={34},
  pages={17981--17993},
  year={2021}
}

@article{dhariwal2021diffusion,
  title={Diffusion models beat {GANs} on image synthesis},
  author={Dhariwal, Prafulla and Nichol, Alexander},
  journal={Advances in Neural Information Processing Systems},
  volume={34},
  pages={8780--8794},
  year={2021}
}

@inproceedings{sohl2015deep,
  title={Deep unsupervised learning using nonequilibrium thermodynamics},
  author={Sohl-Dickstein, Jascha and Weiss, Eric and Maheswaranathan, Niru and Ganguli, Surya},
  booktitle={International Conference on Machine Learning},
  pages={2256--2265},
  year={2015},
}

@article{song2019generative,
  title={Generative modeling by estimating gradients of the data distribution},
  author={Song, Yang and Ermon, Stefano},
  journal={Advances in Neural Information Processing Systems},
  volume={32},
  year={2019}
}

@article{makur2018comparison,
  title={Comparison of channels: Criteria for domination by a symmetric channel},
  author={Makur, Anuran and Polyanskiy, Yury},
  journal={IEEE Transactions on Information Theory},
  volume={64},
  number={8},
  pages={5704--5725},
  year={2018},
  publisher={IEEE}
}

@inproceedings{zhang2024convergence,
  title={Convergence of score-based discrete diffusion models: A discrete-time analysis},
  author={Zhang, Zikun and Chen, Zixiang and Gu, Quanquan},
booktitle={International Conference on Learning Representations},
year={2025},
}

@article{mor2021systematic,
  title={A Systematic Review of Hidden Markov Models and Their Applications.},
  author={Mor, Bhavya and Garhwal, Sunita and Kumar, Ajay},
  journal={Archives of computational methods in engineering},
  volume={28},
  number={3},
  year={2021}
}

@article{mark2024application,
  title={The application of hidden Markov models in speech recognition},
  author={Gales, Mark and Young, Steve},
  journal={Foundations and Trends{\textregistered} in Signal Processing},
  volume={1},
  number={3},
  pages={195--304},
  year={2024},
  publisher={Emerald Publishing Limited}
}

@inproceedings{
pope2021intrinsic,
title={The Intrinsic Dimension of Images and Its Impact on Learning},
author={Phil Pope and Chen Zhu and Ahmed Abdelkader and Micah Goldblum and Tom Goldstein},
booktitle={International Conference on Learning Representations},
year={2021},
}

@article{gorban2018blessing,
  title={Blessing of dimensionality: mathematical foundations of the statistical physics of data},
  author={Gorban, Alexander N and Tyukin, Ivan Yu},
  journal={Philosophical Transactions of the Royal Society A: Mathematical, Physical and Engineering Sciences},
  volume={376},
  number={2118},
  pages={20170237},
  year={2018},
  publisher={The Royal Society Publishing}
}

@article{ingraham2019generative,
  title={Generative models for graph-based protein design},
  author={Ingraham, John and Garg, Vikas and Barzilay, Regina and Jaakkola, Tommi},
  journal={Advances in Neural Information Processing Systems},
  volume={32},
  year={2019}
}

@inproceedings{xu2022geodiff,
  title={Geodiff: A geometric diffusion model for molecular conformation generation},
  author={Xu, Minkai and Yu, Lantao and Song, Yang and Shi, Chence and Ermon, Stefano and Tang, Jian},
  booktitle={International Conference on Learning Representations},
year={2022},
}

@article{liebenau2017asymptotic,
  title={Asymptotic enumeration of graphs by degree sequence, and the degree sequence of a random graph},
  author={Liebenau, Anita and Wormald, Nick},
  journal={Journal of the European Mathematical Society},
  volume={26},
  pages={1--40},
  year={2024}
}

@book{cover1999elements,
  title={Elements of information theory},
  author={Cover, Thomas M},
  year={1999},
  publisher={John Wiley \& Sons}
}
